\documentclass[a4paper,11pt]{article}
\usepackage[top=4cm, bottom=4cm, left=3cm, right=3cm]{geometry}

\usepackage[utf8]{inputenc} 
\usepackage{hyperref}       
\usepackage{url}            
\usepackage{booktabs}       
\usepackage{amsfonts}       
\usepackage{nicefrac}       
\usepackage{microtype}      

\usepackage{color,xcolor}

\usepackage{amsmath,amsthm,amssymb,graphicx}
\usepackage{algpseudocode} 
\usepackage{algorithm}

\newcommand{\Nystrom}[1]{{Nystr\"om}}

\providecommand{\scal}[2]{\left\langle{#1},{#2}\right\rangle}
\providecommand{\nor}[1]{\lVert{#1}\rVert}

\providecommand{\tr}{\operatorname{Tr}}
\DeclareMathOperator*{\esssup}{ess\,sup}

\newcommand{\R}{\mathbb R}
\newcommand{\CC}{\R}

\newcommand{\N}{\mathbb N}

\newcommand{\hh}{\mathcal H}
\newcommand{\g}{\mathcal G}
\newcommand{\kk}{\mathcal K}

\newcommand{\la}{\lambda}

\newcommand{\lspanc}[2]{\overline{\operatorname{span}\{#1~|~#2\}}}

\newcommand{\argmin}[1]{\mathop{\operatorname{argmin}}_{#1}}

\newcommand{\expect}[1]{{\mathbb E}[#1]}

\newcommand{\X}{{X}}

\newcommand{\EE}{{\mathcal E}}

\newcommand{\rhox}{{\rho_{\X}}}
\newcommand{\Ltwo}{{L^2(\X,\rhox)}}

\renewcommand{\L}{{L}}

\newcommand{\K}{K}
\newcommand{\gK}{{\bf \K}}

\newcommand{\yn}{\widehat{y}}

\newcommand{\tK}{K_M}
\newcommand{\tL}{L_M}
\newcommand{\tC}{C_M}
\newcommand{\tCn}{\widehat{C}_M}
\newcommand{\tCl}{C_{M,\la}}
\newcommand{\tCnl}{\widehat{C}_{M,\la}}
\newcommand{\tS}{S_M}
\newcommand{\tSn}{\widehat{S}_M}
\newcommand{\tLl}{L_{M,\la}}
\newcommand{\tkappa}{\kappa}
\newcommand{\Ll}{L_\la}
\newcommand{\fh}{f_\hh}
\newcommand{\frho}{f_\rho}

\newcommand{\bzeta}{{\boldsymbol{z}}}
\newcommand{\bomega}{{\boldsymbol{\omega}}}
\newcommand{\Pzeta}{\mathbb{P}_Z}
\newcommand{\Pomega}{\mathbb{P}_\Omega}
\newcommand{\PP}{\mathbb{P}}

\newcommand{\eqals}[1]{\begin{align*}#1\end{align*}}
\newcommand{\eqal}[1]{\begin{align}#1\end{align}}
\newcommand{\bpr}{\begin{proof}}
\newcommand{\epr}{\end{proof}}
\newcommand{\be}{\begin{equation}}
\newcommand{\ee}{\end{equation}}

\newtheorem{definition}{Definition}
\newcommand{\bd}{\begin{definition}}
\newcommand{\ed}{\end{definition}}

\newcommand{\bi}{\begin{itemize}}
\newcommand{\ei}{\end{itemize}}

\newtheorem{ass}{Assumption}
\newcommand{\ba}{\begin{ass}}
\newcommand{\ea}{\end{ass}}

\newtheorem{remark}{Remark}
\newcommand{\br}{\begin{remark}}
\newcommand{\er}{\end{remark}}

\newtheorem{example}{Example}
\newcommand{\bex}{\begin{example}}
\newcommand{\eex}{\end{example}}

\newtheorem{proposition}{Proposition}
\newcommand{\bp}{\begin{proposition}}
\newcommand{\ep}{\end{proposition}}

\newtheorem{lemma}{Lemma}
\newcommand{\blm}{\begin{lemma}}
\newcommand{\elm}{\end{lemma}}

\newtheorem{theorem}{Theorem}
\newcommand{\bt}{\begin{theorem}}
\newcommand{\et}{\end{theorem}}

\newtheorem{corollary}{Corollary}
\newcommand{\bcor}{\begin{corollary}}
\newcommand{\ecor}{\end{corollary}}

\newcommand{\loz}[1]{{#1}}

\title{Generalization Properties of  Learning with Random Features}

\author{
	Alessandro Rudi  \thanks{This work was done when A.R. was working at Laboratory of Computational and Statistical Learning (Istituto Italiano di Tecnologia).} \\
	{\small INRIA - Sierra Project-team,} \\
	{\small École Normale Supérieure, Paris,} \\
	{\small 75012 Paris, France} \\
	\texttt{\small alessandro.rudi@inria.fr}\\
	\and
	Lorenzo Rosasco \\
	{\small University of Genova,} \\
	{\small Istituto Italiano di Tecnologia,} \\
	{\small Massachusetts Institute of Technology.}\\
	\texttt{\small lrosasco@mit.edu}\\
}

\begin{document}
	
	\maketitle

	\begin{abstract}
		We study the generalization properties of ridge regression with random features in the statistical learning  framework. We show for the first time that $O(1/\sqrt{n})$ learning bounds can be achieved with only  $O(\sqrt{n}\log n)$  random features rather than $O({n})$  as suggested by previous results. Further,  we prove  faster learning rates and show that they might require more random features, unless they are sampled according to  a possibly problem dependent distribution. Our results shed light on the statistical computational trade-offs in large scale kernelized learning, showing the potential  effectiveness of random features in reducing the computational complexity while keeping optimal generalization properties.
	\end{abstract}
	\section{Introduction}
	
	Supervised learning is a basic  machine learning problem where the goal is estimating a function from random noisy samples \cite{vapnik1998statistical,cucker02onthe}.
	The function to be learned is fixed, but unknown, and flexible non-parametric models
	are needed for good results. A general class of models  is based on  functions of the form,
	\eqal{\label{eq:non-par-model}
		{f}(x) \; = \; \sum_{i=1}^M \, \alpha_i \, q(x, \omega_i),
	}
	where $q$ is a non-linear function, $\omega_1, \dots, \omega_M\in \R^d$ are often called centers,   $\alpha_1, \dots, \alpha_M \in \R$ are  coefficients, and  
	$M = M_{n}$ could/should {\em grow} with the number of data points $n$. Algorithmically, the problem  reduces to   computing from  data the parameters $\omega_1, \dots, \omega_M$, $\alpha_1, \dots, \alpha_M$ and $M$. Among others,   one-hidden layer networks \cite{bishop2006}, or RBF networks \cite{Poggio/Girosi/90},  are  examples of classical approaches considering these models. Here, parameters are computed by considering a non-convex optimization problem, typically hard to solve and analyze \cite{pinkus1999approximation}. Kernel methods are another notable example of an approach \cite{schlkopf2002learning} using functions of the form~\eqref{eq:non-par-model}. In this case,  $q$ is assumed  to be a positive definite function \cite{aronszajn1950theory} and it is shown that choosing the centers to be the input points, hence $M=n$, suffices for optimal statistical results \cite{kimeldorf1970correspondence,scholkopf2001generalized,caponnetto2007optimal}.  As a by product,  kernel methods require  only finding   the coefficients $(\alpha_i)_i$, typically  by convex optimization. While theoretically sound and remarkably effective in small and medium size problems,  memory requirements make kernel methods unfeasible for  large scale problems.
	
	%
	Most popular approaches to tackle these limitations are  randomized and include sampling the centers at random, either in a data-dependent or in a data-independent way.  Notable examples include \Nystrom{} \cite{conf/icml/SmolaS00,conf/nips/WilliamsS00} and random features  \cite{conf/nips/RahimiR07} approaches.
	Given random centers, computations still reduce to convex optimization with potential big memory gains, 
	provided that the centers are fewer than the data-points. 
	In practice, the  choice of the number of centers  is   based on heuristics or memory  constraints, and the question arises of characterizing theoretically which choices  provide optimal learning bounds. Answering this question  allows to understand the statistical and computational trade-offs in using these randomized approximations. 
	For  \Nystrom{} methods, partial results in this direction were derived for example in \cite{conf/colt/Bach13}  and improved in  \cite{alaoui2014fast}, but only for a simplified setting where the input points are fixed.
	Results in the statistical learning setting were given  in \cite{rudi2015less} for ridge regression,
	showing in particular that   $O(\sqrt{n} \log n)$ random centers uniformly sampled from $n$ training points suffices to yield $O(1/\sqrt{n})$ learning bounds, {\em the same as full kernel ridge regression}.
	
	A question motivating our study is  whether similar results hold for random features approaches. 
	While several papers consider the properties of random features for approximating the kernel function, see  \cite{sriperumbudur2015} and references therein, fewer results consider their generalization properties.
	
	Several papers considered the properties of random features for approximating the kernel function, see \cite{sriperumbudur2015} and references therein, an interesting line of research with connections to sketching \cite{halko2011finding} and non-linear (one-bit) compressed sensing \cite{plan2014dimension}. However, only a  few results consider the generalization properties of learning with random features.
	
	An exception is one of the  original random features papers, which provides learning  bounds for a general class of  loss functions \cite{rahimi2009weighted}.  These results show that $O({n})$ random features are needed  for $O(1/\sqrt{n})$ learning bounds and choosing  less random features leads to worse bounds. 
	In other words, these results suggest that that computational gains come at the expense of learning accuracy. Later  results, see e.g. \cite{journals/jmlr/CortesMT10,conf/nips/YangLMJZ12,bach2015},  
	essentially  confirm these considerations, albeit the analysis in \cite{bach2015}  suggests that fewer random features could suffice if sampled in a problem dependent way.
	
	In this paper, we focus on the least squares loss, considering random features within a ridge regression approach.
	Our main  result  shows, under standard assumptions, that  the estimator obtained with a number of random features proportional to $O(\sqrt{n}\log n)$  achieves  $O(1/\sqrt{n})$ learning error, that is  the {\em same} prediction  accuracy of  the {\em exact} kernel ridge regression estimator. In other words, there are problems for which random features can allow to drastically reduce computational costs {\em without} any loss of prediction accuracy. To the best of our knowledge this is the first result showing that such an effect is possible. Our study improves on previous results by taking advantage of analytic and probabilistic results  developed to provide sharp analyses of kernel ridge regression. We further present  a second set of more refined results deriving fast convergence rates.  We   show that indeed fast rates are possible, but, depending on the problem at hand, a larger  number of features might be needed. 
	We then discuss how the requirement on the number of random features can be weakened at the expense of typically more complex sampling schemes. Indeed, in this latter case either some knowledge of the data-generating distribution or some 
	potentially data-driven sampling scheme is needed. For this  latter case, we borrow and extend ideas from \cite{bach2015,rudi2015less}  and inspired from the theory of statical leverage scores \cite{journals/jmlr/DrineasMMW12}. Theoretical findings are complemented by numerical simulation validating the bounds.
	
	The rest of the paper is organized as follows. In Section \ref{sect:background}, we review relevant results on learning with kernels, least squares and learning with random features. In Section \ref{sect:main-res}, we present and discuss our main results, while proofs are deferred to the appendix.  Finally, numerical experiments are presented in Section \ref{sec:exp}.

	\section{Learning with random features and ridge regression}\label{sect:background}
	We begin recalling basics ideas in  kernel methods and their approximation via random features.
	%
	\paragraph{Kernel ridge regression} 
	Consider the supervised  problem of learning a function given  a training set of $n$ examples $(x_i, y_i)_{i=1}^n$, where $x_i \in \X$, $\X = \R^D$ and $y_i \in \R$. Kernel methods are  nonparametric approaches defined by  a 
	{\em kernel } $K: \X \times \X \to \R$, that is a symmetric and positive definite (PD) function\footnote{A kernel $K$ is PD  if for all $x_1, \dots, x_N$ the $N$ by $N$ matrix with entries $K(x_i,x_j)$ is positive semidefinite.}.
	A particular instance is  kernel ridge regression  given by
	\eqal{\label{eq:base-KRR}
		\widehat{f}_\la(x) = \sum_{i=1}^n \alpha_i K(x_i, x), \quad \alpha = (\gK + \la n I)^{-1} y.
	}
	Here $\la > 0$, $y = (y_1, \dots, y_n)$, $\alpha \in \R^n$, and  $\gK$ is the $n$ by $n$ matrix  with  entries $\gK_{ij} = K(x_i, x_j)$. The above method is standard and can be derived from an empirical risk minimization perspective \cite{schlkopf2002learning}, and is related to Gaussian processes \cite{bishop2006}. While  KRR has optimal statistical properties-- see later--  its applicability to large scale datasets is limited since it requires $O(n^2)$ in space, to store $\gK$, and  roughly $O(n^3)$ in time,  to solve the linear system in~\eqref{eq:base-KRR}. Similar requirements are shared by other kernel methods \cite{schlkopf2002learning}.\\
	To explain the basic ideas behind using random features with ridge regression, it is useful to
	recall the computations needed to solve KRR when the kernel is linear $K(x,x')=x^\top x'$.
	In this case, Eq.~\eqref{eq:base-KRR} reduces to standard ridge regression and can be equivalenty computed considering,
	\eqal{\label{eq:base-KRRLin}
		\widehat{f}_\la(x) = x^\top\widehat{w}_\la \quad\quad \widehat{w}_\la = (\widehat{X}^\top\widehat{X} + \la n I)^{-1}\widehat{X}^\top y.
	}
	where $\widehat{X}$ is the  $n$ by $D$ data matrix. In this case, the complexity becomes  $O(nD)$ in space,  and  $O(nD^2+D^3)$ in time.
	Beyond  the linear case, the above reasoning  extends to  inner product kernels
	\eqal{\label{eq:rf-kerM}
		K(x,x') = \phi_M(x)^\top \phi_M(x')
	}
	where $\phi_M:\X\to \R^M$ is a finite dimensional (feature) map. In this case,  KRR can be computed considering~\eqref{eq:base-KRRLin} with  the data matrix 
	$\widehat{X}$ replaced by  the   $n$ by $M$  matrix $\tSn^\top = (\phi(x_1), \dots, \phi(x_n))$. The complexity is then  $O(nM)$ in space,  and  $O(nM^2+M^3)$ in time, hence much better than  $O(n^2)$ and  $O(n^3)$, as soon as $M\ll n$.
	Considering only kernels of the form~\eqref{eq:rf-kerM} can be restrictive. Indeed, classic examples of kernels, e.g.  the Gaussian kernel $e^{-\nor{x-x'}^2}$,  do not satisfy~\eqref{eq:rf-kerM} with finite $M$.  It is then natural to ask if the above reasoning can still be useful to reduce the computational burden for more complex kernels such as the Gaussian kernel.  Random features, that we recall next, show that this is indeed the case.
	\paragraph{Random features with ridge regression}
	The basic idea of random features \cite{conf/nips/RahimiR07} is to relax  Eq.~\eqref{eq:rf-kerM} assuming it holds only approximately,
	\eqal{\label{eq:rf-apprkerM}
		K(x,x') \approx \phi_M(x)^\top \phi_M(x'). 
	}
	Clearly,  if one such approximation exists the approach described in the previous section can still be used. A first question is then for which kernels an approximation of the form~\eqref{eq:rf-apprkerM} can be derived. A simple manipulation of the  Gaussian kernel provides one basic example.
	\bex[Random Fourier features \cite{conf/nips/RahimiR07}]\label{ex:rff}
	If we write the Gaussian kernel as  $K(x,x') =  G(x-x')$, with $G(z) = e^{-\frac{1}{2\sigma^2}\nor{z}^2}$, for a $\sigma > 0$, 
	then since the inverse Fourier transform of $G$ is a Gaussian, and using a basic symmetry argument,
	it is easy to show that 
	$$
	G(x-x') ~~~=~~~   \frac{1}{2\pi Z} \int\int_0^{2\pi} ~~\sqrt{2}\cos(w^\top x+ b)~\sqrt{2}\cos(w^\top x'+ b)~~ e^{-\frac{\sigma^2}{2}\nor{w}^2}dw~~db    
	$$
	where $Z$ is a normalizing factor.
	Then, the Gaussian kernel has an approximation of the form~\eqref{eq:rf-apprkerM}  with
	$
	\phi_M(x)= M^{-1/2}~(\sqrt{2}\cos(w_1^\top x+ b_1), \dots,   \sqrt{2}\cos(w_M^\top x+ b_M)),
	$
	and  $w_1, \dots, w_M$ and $b_1, \dots, b_M$ sampled independently from $ \frac{1}{Z}e^{-\sigma^2\nor{w}^2/2}$ and uniformly in $[0,2\pi]$, respectively. 
	\eex 
	The above example can be abstracted to a   general strategy. 
	Assume  the kernel $K$ to have an integral representation,
	\eqal{\label{eq:def-RF-integral}
		K(x,x') = \int_{\Omega} \psi(x,\omega) \psi(x', \omega)d \pi(\omega), \quad \forall x,x' \in \X,
	}
	where $(\Omega,\pi)$ is probability space and $\psi: \X \times \Omega \to \R$. The random features approach provides an approximation of the form~\eqref{eq:rf-apprkerM} where 
	$
	\phi_M(x)=M^{-1/2} ~(\psi(x,\omega_1),\dots, \psi(x,\omega_M)),
	$
	and with $\omega_1, \dots, \omega_M$   sampled independently with respect to $\pi$.
	%
	Key to the success of random features is that kernels, to which the above idea apply, abound-- see Appendix~\ref{sect:rf-examples} for a survey with some details.\\
	\begin{remark}[Random features, sketching and one-bit compressed sensing]
		We note that specific examples of random features can be seen as form of sketching \cite{halko2011finding}. This latter term typically refers to reducing data dimensionality by random projection, e.g. considering 
		$$
		\psi(x,\omega)=x^\top \omega,
		$$
		where $\omega \sim N(0, I)$ (or suitable bounded measures).  From a random feature perspective, we are defining an approximation of the linear kernel since
		$$
		\expect{\psi(x,\omega)\psi(x',\omega)}  = \expect{x^\top \omega \omega^\top x'} =  x^\top \expect{\omega \omega^\top} x' = x^\top x'.
		$$
		More general non-linear sketching can also be considered. For example in one-bit compressed sensing \cite{plan2014dimension} the following random features are relevant, 
		$$
		\psi(x,\omega)=\textrm{sign}(x^\top \omega)
		$$
		with $w\sim N(0, I)$ and $\textrm{sign}(a)=1$ if $a>0$ and $-1$ otherwise. 
		Deriving the corresponding kernel is more involved and we refer to \cite{cho2009} (see Section~\ref{sect:rf-examples} in the appendixes). 
	\end{remark}
	
	Back to supervised learning, combining random features with ridge regression  leads to,
	\eqal{\label{eq:algo-rf}
		\widehat{f}_{\la, M}(x) := \phi_M(x)^\top \widehat{w}_{\la, M}, \quad \textrm{with} \quad \widehat{w}_{\la, M} := (\widehat{S}_M^\top \widehat{S}_M + \la I)^{-1} \widehat{S}_M^\top \yn,
	}
	for $\la >0$,   $\widehat{S}_M^\top := n^{-1/2}~(\phi_M(x_1),\dots, \phi_M(x_n))$ and $\yn := n^{-1/2}~(y_1, \dots, y_n)$.
	
	Then,  random features can be used to reduce the  computational costs  of
	full kernel ridge regression as soon as  $M\ll n$ (see Sec.~\ref{sect:background}). However, since  random features rely on an approximation~\eqref{eq:rf-apprkerM}, the question is whether there is   a loss of prediction accuracy. This is the question we analyze in the rest of the paper.

	\section{Main Results}\label{sect:main-res}
	
	In this section, we present our main results  characterizing the  generalization properties of random features with ridge regression.
	We begin considering a basic setting and then discuss fast learning rates and the possible benefits of problem dependent sampling schemes. 
	%
	%
	%
	
	\subsection{$O(\sqrt{n} \log n)$ Random features lead to $O(1/\sqrt{n})$ learning error}
	
	We consider a standard statistical learning setting. The data $(x_i, y_i)_{i=1}^n$ are sampled identically and independently with respect to  a probability  $\rho$ on $\X \times \R$\loz{, with $\X$ a separable space (e.g. $\X = \R^D$, $D \in \N$)}. The goal is to   minimize the expected risk 
	$$ {\cal E}(f) = \int (f(x) - y)^2 d\rho(x,y),$$
	since this implies that $f$ will generalize/predict well new data. 
	Since we  consider estimators of the form~\eqref{eq:base-KRR},~\eqref{eq:algo-rf}  we are potentially  restricting the space of possible solutions. 
	\loz{Indeed, estimators of this form can be naturally related to the so called reproducing kernel Hilbert space}  (RKHS) corresponding to the PD kernel $K$.  Recall that,  the latter is the function space $\hh$ defined as  as the completion of the linear span of  $\{K(x,\cdot)~:~ x\in X\}$ with respect to the inner product $\scal{K(x,\cdot)}{K(x',\cdot)}:=K(x,x')$ \cite{aronszajn1950theory}. 
	In this view, the best possible solution  is  $f_\hh$ solving 
	\eqal{\label{eq:exists-fh}
		\min_{f\in \hh} {\cal E}(f).
	}
	We will assume throughout that  $f_\hh$ exists. We add one technical remark useful in the following. 
	\br
	Existence of $f_\hh$ is not ensured, since we  consider a potentially  infinite dimensional RKHS $\hh$, possibly universal \cite{steinwart2008support}.
	The situation is different if $\hh$ is replaced by 
	$
	\hh_R=\{f\in \hh~:~ \nor{f}\le R\},
	$ 
	with $R$ fixed a priori.   In this case a minimizer of  risk $\cal E$ always exists, but  $R$ needs to be fixed a priori and $\hh_R$ can't be universal. Clearly, assuming $f_\hh$ to exist, implies it belongs to a ball of radius $R_{\rho,\hh}$. However, our results do not require prior knowledge of $R_{\rho,\hh}$ and hold uniformly over all finite radii.
	\er
	
	The following is our first result on the learning properties of random features with   ridge regression.

	%

	%
	\bt\label{thm:rf-simple-universal}
	Assume that  $K$ is a kernel with an integral representation~\eqref{eq:def-RF-integral}. Assume $\psi$ continuous, such that $|\psi(x,\omega)| \leq \kappa$ almost surely, with $\kappa \in [1, \infty)$ and $|y|\le b$ almost surely, with $b > 0$. Let $\delta \in (0,1]$. If $n \geq n_0$ and $\la_n =  n^{-1/2}$, then a number of random features $M_n$ equal to
	$$M_n = c_0 ~ \sqrt{n} ~ \log \frac{108\kappa^2 \sqrt{n}}{\delta},$$ 
	is enough to guarantee, with probability at least $1 - \delta$, that
	$${\cal E}(\widehat{f}_{\la_n,M_n}) - {\cal E}(f_\hh) \leq \frac{c_1 \log^2 \frac{18}{\delta}}{\sqrt{n}}.$$
	In particular the constants $c_0, c_1$  do not depend on $n, \la, \delta$, and  $n_0$ does not depends on $n, \la, f_\hh, \rho$.
	\et 
	The above result is presented with some simplifications (e.g. the assumption of bounded output) for sake of presentation, while it is proved and presented in full generality in the Appendix. In particular, the values of all the constants are given explicitly. Here, we make a  few comments.
	The learning bound is the same achieved by the {\em exact} kernel ridge regression estimator~\eqref{eq:base-KRR} choosing $\la=n^{-1/2}$, see e.g. \cite{caponnetto2007optimal}. 
	The theorem derives a bound in a  worst case situation, where no assumption is made besides existence of $f_\hh$, and 
	is  optimal in a minmax sense \cite{caponnetto2007optimal}. 
	This means that, in this setting,  as soon as the number of features is order $\sqrt{n}\log n$,  the corresponding ridge regression estimator has  optimal generalization properties. This is remarkable considering  the corresponding gain from a computational perspective:  from  roughly $O(n^3)$ and $O(n^2)$ in time and space  for  kernel ridge regression to $O(n^2)$ and $O(n\sqrt{n})$ for ridge regression with random features (see Section~\ref{sect:background}).
	Consider that taking $\delta\propto1/n^2$ changes only the constants and allows to derive  bounds in expectation and  almost sure convergence (see Cor.~\ref{cor:simple-rates-expectation} in the appendix, for the result in expectation).
	\\
	%
	The above result shows that there is  a whole set of problems where computational gains
	are achieved without having to trade-off statistical accuracy. In the next sections we consider what happens under more benign assumptions, which are standard, but also somewhat more technical. 
	We first compare with previous works since the above setting is the one more closely related.
	\paragraph{Comparison with \cite{rahimi2009weighted}.} This is one of the original random features paper and considers the question of generalization properties. In particular they study the estimator
	$$\widehat{f}_{R}(x) = \phi_M(x)^\top \widehat{\beta}_{R,\infty}, \quad \widehat{\beta}_{R,\infty} = \argmin{\nor{\beta}_\infty \leq R} \frac{1}{n} \sum_{i=1}^n \ell(\phi_M(x_i)^\top \beta, y_i),$$
	for a fixed $R$,  a Lipshitz loss function $\ell$, and where $\nor{w}_\infty = \max\{|\beta_1|,\cdots,|\beta_M|\}$. 
	The largest space considered in \cite{rahimi2009weighted} is 
	\be\label{eq:hpi8R}
	\g_{R} = \left\{ \int \psi(\cdot, \omega) \beta(\omega) d\pi(\omega) ~\middle|~ |\beta(\omega)| < R \;\; \textrm{a.e.}\right\},
	\ee
	rather than a RKHS, where $R$ is fixed a priori. The best possible solution is $f^*_{\g_{R} }$ solving 
	$\min_{f \in \g_{R}} {\cal E}(f),$
	and the main result in \cite{rahimi2009weighted} provides the  bound
	\eqal{\label{eq:rahimi}
		{\cal E}(\widehat{f}_{R}) - {\cal E}(f^*_{\g_R}) \lesssim \frac{R}{\sqrt{n}} + \frac{R}{\sqrt{M}},
	} 
	This is the first and still one the main results  providing a statistical analysis for an estimator based on random features for a wide class of loss functions. There are a few elements of comparison with the result in this paper, but the main one is that to get $O(1/\sqrt{n})$ learning bounds,   the above result requires  $O(n)$ random features,  while a smaller number leads to worse bounds. This shows the main novelty of our analysis. Indeed we prove that, considering the square loss, fewer random features are sufficient, hence allowing computational gains without loss of accuracy. We add a few more tehcnical comments explaining : 1) how the setting we consider covers a wider range of problems, and 2) why the bounds we obtain are sharper.  First, note that the functional setting in our paper is more general in the following sense.
	It is easy to see that considering the RKHS $\hh$ is equivalent to consider
	$
	\hh_{2} = \left\{ \int \psi(\cdot, \omega) \beta(\omega) d\pi(\omega) ~\middle|~ \int |\beta(\omega)|^2d\pi(\omega) <\infty \right\}
	$
	and the following inclusions hold $\g_R \subset \g_\infty \subset \hh_2$.
	%
	Clearly, assuming  a minimizer of the expected risk to exists in $\hh_2$ {\em does not} imply it belongs to $\g_{\infty}$ or $\g_R$, while the converse is true. In this view, our results cover a wider range of problems.  Second, note that, this gap is not easy to bridge.  Indeed, even if we were to consider $\g_{\infty}$ in place of $\g_R$, the results in \cite{rahimi2009weighted} could be used to derive the bound
	\eqal{\label{eq-rahimi-adj}
		{\mathbb E} ~~ {\cal E}(\widehat{f}_{R}) - {\cal E}(f^*_{\g_\infty}) \lesssim \frac{R}{\sqrt{n}} + \frac{R}{\sqrt{M}} + A(R),
	}
	where   $A(R) := {\cal E}(f^*_{\g_R}) - {\cal E}(f^*_{\g_\infty})$ and $f^*_{\g_\infty}$ is a minimizer of the expected risk on $\g_\infty$.
	In this case  we would have to balance the various terms in~\eqref{eq-rahimi-adj}, which would lead to a worse bound.
	For example, we could consider  $R := \log n$,  obtaining a bound $n^{-1/2} \log n$ with an extra  logarithmic term, but 
	the result would hold only for $n$ larger than  a number of examples $n_0$ at least {\em exponential} with respect to the norm of $f_\infty$. Moreover, to derive results uniform with respect to $f_\infty$, we would have to keep into account the decay rate of $A(R)$ and this would get bounds slower than $n^{-1/2}$. 
	%
	%
	\paragraph{Comparison with other results.} Several other papers study the generalization properties of random features, see \cite{bach2015} and references therein. 
	For example, generalization bounds  are derived  in  \cite{journals/jmlr/CortesMT10}  from very general arguments.
	However,  the corresponding generalization bound requires a number of random features much larger than the number of training examples to give $O(1/\sqrt{n})$ bounds. The basic results in \cite{bach2015} are analogous to those in \cite{rahimi2009weighted} with the set $\g_R$ replaced by  $\hh_R$.  These results are closer, albeit more restrictive then ours (see Remark~\ref{eq:exists-fh}) and especially like the bounds in 
	\cite{rahimi2009weighted}  suggest $O(n)$ random features are needed for $O(1/\sqrt{n})$ learning bounds. 
	A novelty in \cite{bach2015}  is  the introduction of  more complex problem dependent sampling that can reduce the number  of random features. 
	%
	In Section~\ref{sect:non-uniform-sampling}, we show that using possibly-data dependent random features  can lead to rates  much  faster  than $n^{-1/2}$,  and using much less than $\sqrt{n}$ features.
	%
	%
	\br[Sketching and randomized numerical linear algebra (RandLA)] 
	Standard sketching techniques from RandLA \cite{halko2011finding} can be recovered, when $X$ is a bounded subset of  $\R^D$, by selecting $\psi(x,\omega) = x^\top \omega$ and $\omega$ sampled from suitable bounded distribution (e.g. $\omega = (\zeta_1, \dots, \zeta_d)$ independent Rademacher random variables). Note however that the final goal of the analysis in the randomized numerical linear algebra community is to minimize the empirical error instead of ${\cal E}$.
	\er
	\begin{figure}
		\centering
		\includegraphics[width=0.8\linewidth]{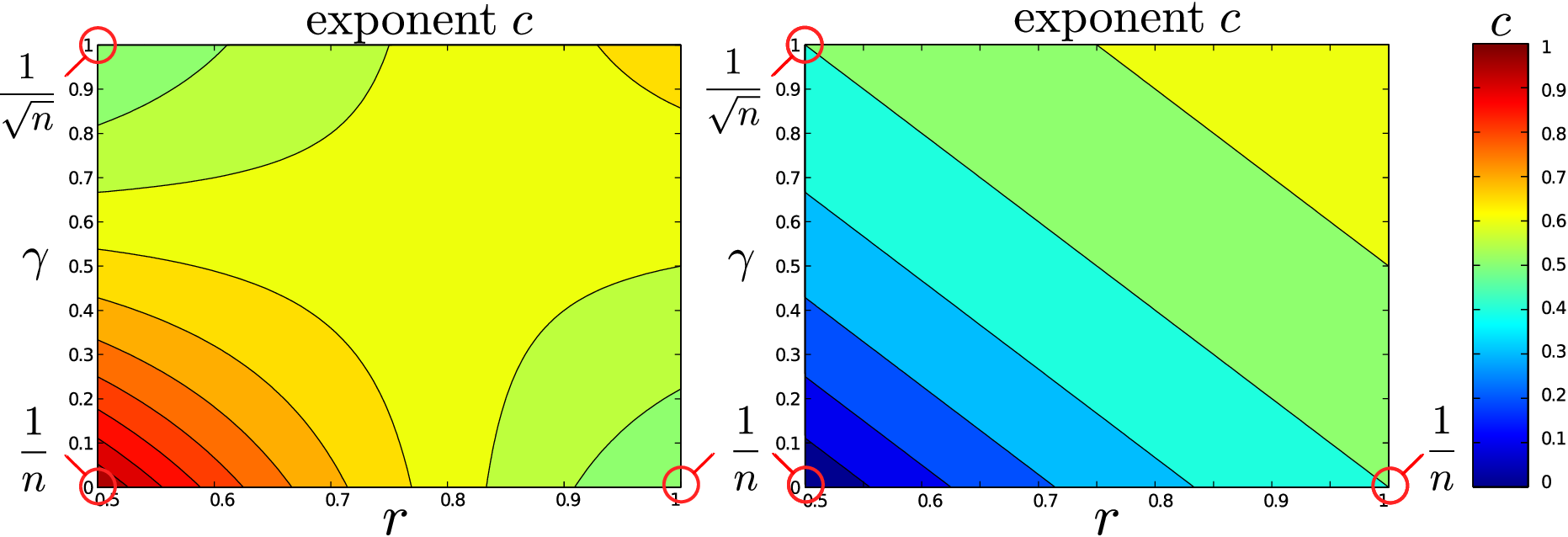}
		\caption{Random feat. $M = O(n^c)$ required for optimal generalization. Left: $\alpha = 1$. Right: $\alpha = \gamma$.}
		\label{fig:rf-fast-sanity}
	\end{figure}
	\subsection{Refined Results: Fast Learning Rates}
	Faster rates can be achieved under favorable conditions. Such conditions for kernel ridge regression are standard, but somewhat technical. Roughly speaking they characterize the ``size'' of the considered RKHS and the regularity of $f_\hh$.  The key quantity needed to make this precise is the integral operator defined by the kernel  
	$K$ and   the marginal distribution $\rhox$  of $\rho$ on $\X$, that is 
	$$ (Lg)(x) = \int_{\X} K(x, z) g(z) d\rhox(z), \quad \forall g \in \Ltwo,$$
	seen as a map from 
	$\Ltwo = \{f:\X \to \R ~|~ \nor{f}_\rho^2=\int |f(x)|^2d\rhox< \infty \}$ to itself.
	Under the assumptions of Thm.~\ref{thm:rf-simple-universal}, the integral operator is positive, self-adjoint and trace-class (hence compact) \cite{smale2007learning}. We next define the conditions that will lead to fast rates, and then comment on their interpretation. 
	\ba[Prior assumptions]\label{ass:prior}
	For  $\la > 0$, let  the effective dimension be defined as ${\cal N}(\la) := \tr \left((L+\la I)^{-1} L\right),$
	and assume,   there exists $Q > 0$ and $\gamma \in [0,1]$ such that, 
	\be\label{eq:degree} {\cal N}(\la) \leq Q^2 \la^{-\gamma}.\ee
	Moreover, assume there exists $  r\ge 1/2$ and $g \in \Ltwo$ such that
	\be\label{eq:sou} \fh(x) = (L^r g)(x) \quad \textrm{a.s.}\ee 
	\ea
	We provide some intuition on the meaning of the above assumptions, and defer the interested reader  to \cite{caponnetto2007optimal} for  more details.  The effective dimension can be seen as a ``measure of the size" of the RKHS $\hh$. Condition~\eqref{eq:degree} allows to  control the variance of the estimator and  is equivalent to conditions on covering numbers and related capacity measures \cite{steinwart2008support}. In particular, it holds if the eigenvalues $\sigma_i$'s of $L$  decay as $i^{-1/\gamma}$. Intuitively, a fast decay corresponds to a smaller RKHS, whereas  a slow decay corresponds to a larger RKHS. The case $\gamma=0$ is the more benign situation, whereas $\gamma=1$ is the worst case, corresponding to the basic setting. A classic example\loz{, when $X = \R^D$,} corresponds to considering kernels of smoothness $s$, in which case $\gamma=D/(2s)$  and condition~\eqref{eq:degree} is equivalent to assuming $\hh$ to be a  Sobolev space \cite{steinwart2008support}. Condition~\eqref{eq:sou} allows to control the bias of  the estimator and is common in approximation theory \cite{smale2003estimating}. It is a regularity condition that  can be  seen as form of weak sparsity of $f_\hh$.  Roughly speaking, it requires the expansion of $f_\hh$, on the the basis given by the the eigenfunctions $L$, to have coefficients that decay \loz{faster than} $\sigma_i^{r}$. A large value of $r$ means that the coefficients decay fast and hence many are close to zero. The case  $r=1/2$ is the worst case, and can be shown to be equivalent to assuming $f_\hh$ exists. This latter situation corresponds to  setting considered in the previous section. 
	We next show how these assumptions allow to derive  fast  rates.
	\bt\label{thm:rf-fast-universal}
	Let $\delta \in (0,1]$. Under Asm.~\ref{ass:prior} and the same assumptions of Thm.~\ref{thm:rf-simple-universal}, if $n \geq n_0$, and $\la_n = n^{-\frac{1}{2r+\gamma}}$, then a number of random features $M$ equal to
	$$M_n ~~=~~ c_0 ~n^{\frac{1 + \gamma(2r-1)}{2r + \gamma}}~\log\frac{108\kappa^2 n}{\delta},$$
	is enough to guarantee, with probability at least $1 - \delta$, that 
	$${\cal E}(\widehat{f}_{\la_n,M_n}) - {\cal E}(f_\hh) \leq c_1\log^2\frac{18}{\delta} ~ n^{-\frac{2r}{2r+\gamma}},$$
	for  $r\le 1$,   and  where $c_0, c_1$ do not depend on $n, \tau$, while $n_0$ does not depends on $n, f_\hh, \rho$.
	\et
	The above bound is the same as the one obtained by the full kernel ridge regression estimator and is optimal in a minimax sense \cite{caponnetto2007optimal}. For large $r$ and small $\gamma$ it approaches a $O(1/n)$ bound.  When $\gamma = 1$ and $r = 1/2$ the worst case bound of the previous section is recovered.
	Interestingly, the number of random features in different regimes  is typically smaller than $n$ but can be larger than  $O(\sqrt{n})$.   Figure.~\ref{fig:rf-fast-sanity} provides a pictorial representation of the number of random features needed  for optimal rates in different regimes. In particular $M \ll n$ random features are enough when $\gamma > 0$ and $r > 1/2$. For example for $r=1, \gamma=0$ (higher regularity/sparsity and a small RKHS) $O(\sqrt{n})$ are  sufficient to get a rate $O(1/n)$.  But, for example, if $r=1/2, \gamma=0$ (not too much regularity/sparsity but a small RKHS) $O(n)$ are needed for $O(1/n)$ error. The proof suggests that this effect can be a byproduct of sampling  features  in a data-independent way. Indeed, in the next section we show how 
	much fewer  features can be used considering problem dependent sampling schemes.
	
	\subsection{Refined Results: Beyond uniform sampling} \label{sect:non-uniform-sampling}
	
	We show next that  fast learning rates can be achieved with fewer random features  if they are somewhat {\em compatible} with the data distribution.
	This is made precise by the following condition.
	\ba[Compatibility condition]\label{ass:compatibility}
	Define the  {\em maximum random features dimension} as 
	\eqal{\label{eq:F-infty}
		{\cal F}_\infty(\la) ~= ~\sup_{\omega \in \Omega} ~\nor{(L+\la I)^{-1/2} \psi(\cdot, \omega)}^2_\rhox, \quad \la>0.
	}
	Assume there exists $\alpha \in [0, 1]$, and $F > 0$ such that
	${\cal F}_\infty(\la) \leq F \la^{-\alpha}, \quad  \forall \la > 0.
	$
	\ea
	The above assumption is abstract and we comment on it before showing how it affects the results. 
	The maximum random features dimension~\eqref{eq:F-infty} relates the random features to the data-generating distribution through the operator $L$. 
	It is always satisfied for $\alpha = 1$ ands $F = \kappa^2$. e.g.  considering any random feature satisfying~\eqref{eq:def-RF-integral}. The favorable situation corresponds to 
	random features such that  case $\alpha = \gamma$. The following  theoretical  construction borrowed from \cite{bach2015} gives an example.
	\bex[Problem dependent RF]\label{ex:lsrf}
	Assume $K$ is a kernel with an integral representation~\eqref{eq:def-RF-integral}. 
	For $s(\omega) = \nor{(L+\la I)^{-1/2}\psi(\cdot, \omega)}_\rhox^{-2}$  
	and $C_s := \int \frac{1}{s(\omega)} d\pi(\omega)$,
	consider the random features $\psi_s(x,\omega) = \psi(x,\omega) \sqrt{C_s s(\omega)},$ with  distribution $\pi_s(\omega) := \frac{\pi(\omega)}{C_s s(\omega)}$. We show in the Appendix that these random features  provide an integral representation of $K$  
	and  satisfy Asm.~\ref{ass:compatibility} with $\alpha = \gamma$.
	\eex
	We next show how random features satisfying Asm.~\ref{ass:compatibility} can lead to  better resuts.
	\bt\label{thm:rf-fast-compatibility}
	Let $\delta \in (0, 1]$. Under Asm.~\ref{ass:compatibility} and the same assumptions of Thm. \ref{thm:rf-simple-universal}, \ref{thm:rf-fast-universal}, if $n \geq n_0$,  and $\la_n = n^{-\frac{1}{2r+\gamma}}$, then a number of random features $M_n$ equal to
	$$M_n ~~=~~ c_0 ~n^{\frac{\alpha + (1 +\gamma - \alpha)(2r-1)}{2r + \gamma}}~ \log\frac{108\kappa^2n}{\delta},$$
	is enough to guarantee, with probability at least $1 - \delta$, that 
	$${\cal E}(\widehat{f}_{\la_n,M_n}) - {\cal E}(f_\hh) \leq c_1\log^2\frac{18}{\delta} ~ n^{-\frac{2r}{2r+\gamma}},$$
	where $c_0, c_1$ do not depend on $n, \tau$, while $n_0$ does not depends on $n, f_\hh, \rho$.
	\et
	The above learning bound is the same as Thm.~\ref{thm:rf-fast-universal}, but the number of random features is given by a more complex expression  depending on $\alpha$.  In particular, in the slow $O(1/\sqrt{n})$ rates scenario, that is $r=1/2$, $\gamma=1$, we see that $O( n^{\alpha/2})$ are needed, recovering $O(\sqrt{n})$, since $\gamma \leq \alpha \leq 1$. On the contrary, for a small RKHS, that is $\gamma=0$ and random features with $\alpha=\gamma$,  a constant (!) number of feature is sufficient. A similar trend is seen considering fast rates. 
	For $\gamma>0$ and $r>1/2$, if  $\alpha< 1$ then the number of random features is always smaller, and potentially much smaller,  then the number of random features sampled in a problem independent way, that is $\alpha=1$. For $\gamma=0$ and $r=1/2$, the number of number of features is $O(n^\alpha)$ and can be again just constant if $\alpha=\gamma$.  Figure~\ref{fig:rf-fast-sanity} depicts the  number of random features required if $\alpha=\gamma$. The above result shows the potentially dramatic effect of problem dependent random features. However the construction in Ex.~\ref{ex:lsrf}
	is theoretical. We comment on this in the next  remark.
	\br[Random features leverage scores]
	The  construction in Ex.~\ref{ex:lsrf} is  theoretical, however {\em empirical random features leverage scores}
	$\widehat s(\omega)=  \widehat{v}(\omega)^\top(\gK + \lambda n I)^{-1}\widehat{v}(\omega)$, with $\widehat{v}(\omega) \in \R^n$, $(\widehat{v}(\omega))_i = \psi(x_i,\omega)$, can be considered. Statistically, this requires considering an extra estimation step. It seems our proof can be extended to account for this, and we will pursue this in a future work. Computationally, it requires devising approximate numerical strategies, like standard leverage scores {\em \small \cite{journals/jmlr/DrineasMMW12}}.
	\er
	%
	\begin{figure}[t]
		\begin{center}
			\includegraphics[trim={14cm 0 0 0},clip,width=0.8\linewidth]{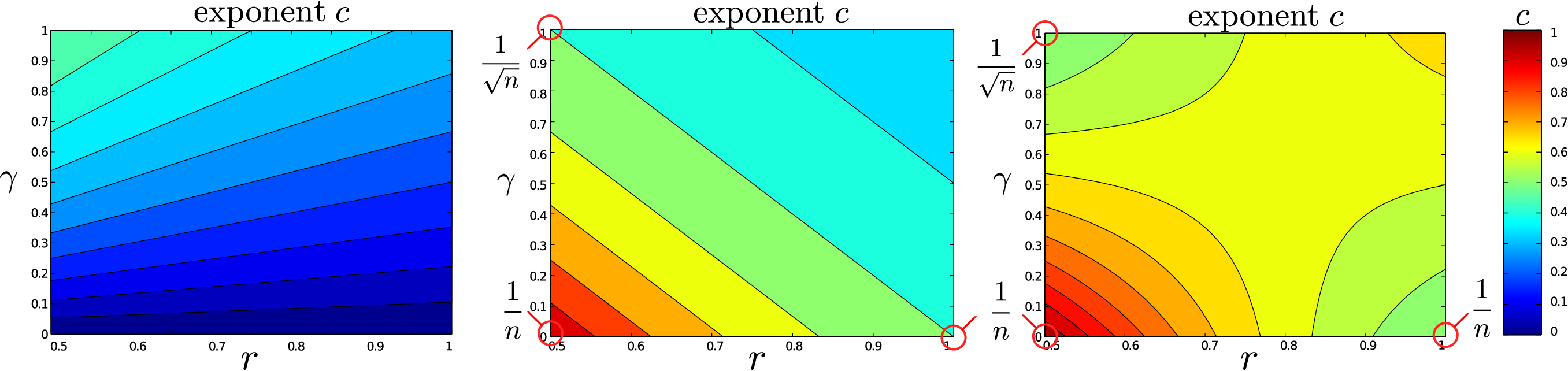}
		\end{center}
		\caption{Comparison between the number of features $M = O(n^c)$ required by \Nystrom{} (uniform sampling, left) \cite{rudi2015less} and Random Features ($\alpha = 1$, right), for optimal generalization. \label{fig:nystrom-comparison}}
	\end{figure}
	\paragraph{Comparison with \Nystrom{}.} This question was recently considered in \cite{conf/nips/YangLMJZ12} and our results offer new insights. In particular, recalling the results in \cite{rudi2015less}, we see that in  the slow rate setting there is essentially no difference between random features and \Nystrom{} approaches, neither from a  statistical nor from a computational point of view. In the case of fast rates,  \Nystrom{} methods with uniform sampling requires $O(n^{-\frac{1}{2r+\gamma}})$ random centers, which compared to Thm.~\ref{thm:rf-fast-universal},  suggests  \Nystrom{} methods can  be advantageous in this regime. 
	While problem dependent random features provide a further improvement, it should be compared with the number of centers needed for \Nystrom{} with leverage scores, which is $O(n^{-\frac{\gamma}{2r+\gamma}})$ and hence again better, see Thm.~\ref{thm:rf-fast-compatibility}. In summary, both random features and \Nystrom{} methods achieve optimal statistical guarantees while reducing computations. They are essentially the same in the worst case, while \Nystrom{} can be better for benign problems.
	\\
	Finally we add a few words about the main steps in the proof.
	\paragraph{Steps of the proof.} The proofs are quite technical and long and are collected in the appendices. They use a battery of tools developed to analyze KRR and related methods. The key challenges in the analysis include analyzing the bias of the estimator, the effect of noise in the outputs, the effect of random sampling in the data, the approximation due to random features and a notion of orthogonality between the function space corresponding to random features and the full RKHS. The last two points are the main elements on novelty in the proof. In particular, compared to other studies,  we identify and study the  quantity needed to assess the effect of the random feature approximation if the goal is prediction rather than the kernel approximation itself.

	
	\section{Numerical results}\label{sec:exp}
	While the learning bounds we present are optimal, there are no lower bounds on the  number of random features, hence we present numerical experiments validating our bounds. 
	%
	Consider a spline kernel of order $q$ (see \cite{Wahba/90} Eq.~2.1.7 when $q$ integer), defined as
	$$ \Lambda_q(x,x') = \sum_{k=-\infty}^\infty {e^{2\pi i k x} e^{-2\pi i k z}}{|k|^{-q}}, $$
	almost everywhere on $[0,1]$, with $q \in \R$, for which we have
	$$\int_0^1 \Lambda_{q}(x,z) \Lambda_{q'}(x',z) dz = \Lambda_{q+q'}(x,x'),$$
	for any $q,q' \in \R$. Let $\X = [0,1]$, and $\rhox$ be the uniform distribution. 
	For $\gamma \in (0,1)$ and $r \in [1/2,1]$ let,
	$K(x,x') = \Lambda_\frac{1}{\gamma}(x,x')$, $\psi(\omega,x) = \Lambda_{\frac{1}{2\gamma}}(\omega,x)$, $f_*(x) = \Lambda_{\frac{r}{\gamma} + \frac{1}{2} +\epsilon}(x,x_0)$
	with $\epsilon > 0, x_0 \in \X$. Let  $\rho(y|x)$ be a Gaussian density with variance $\sigma^2$ and mean $f^*(x)$.
	Then  Asm~\ref{ass:prior}, \ref{ass:compatibility} are  satisfied and  $\alpha = \gamma$. We compute the KRR estimator for $n \in \{10^3,\dots,10^4\}$ and select  $\la$ minimizing the excess risk computed analytically. 
	Then we compute the RF-KRR estimator and select the number of features $M$ needed to obtain an excess risk within $5\%$ of the one by KRR.
	In Figure~\ref{fig:sim-1},  the theoretical and estimated behavior of
	the excess risk, $\la$ and $M$ with respect to $n$ are reported together with their standard deviation over 100 repetitions. 
	The experiment shows that the predictions by Thm.~\ref{thm:rf-fast-compatibility} are  accurate, since the theoretical predictions estimations are within one standard deviation from the values measured in the simulation.
	\begin{figure}[t]
		\includegraphics[width=0.32\textwidth,height=0.23\textwidth]{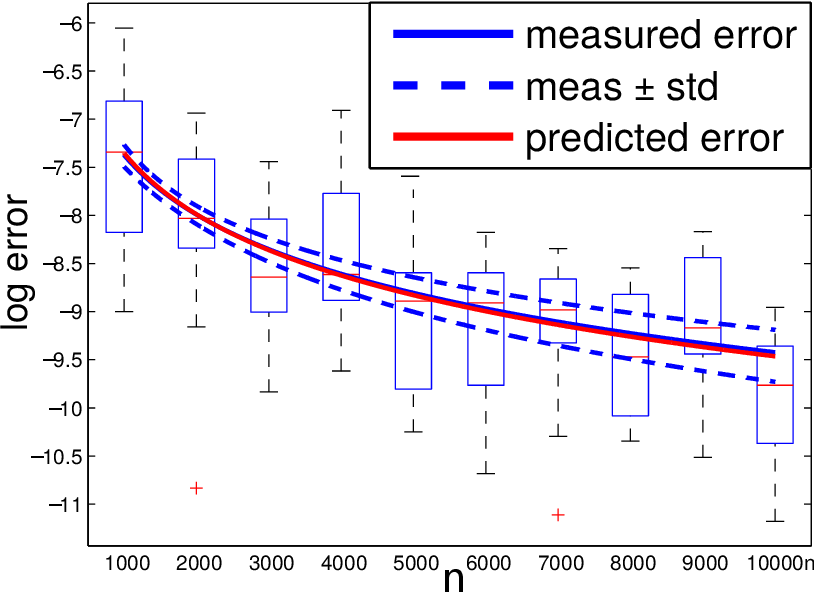}
		\hfill
		\includegraphics[width=0.32\textwidth,height=0.23\textwidth]{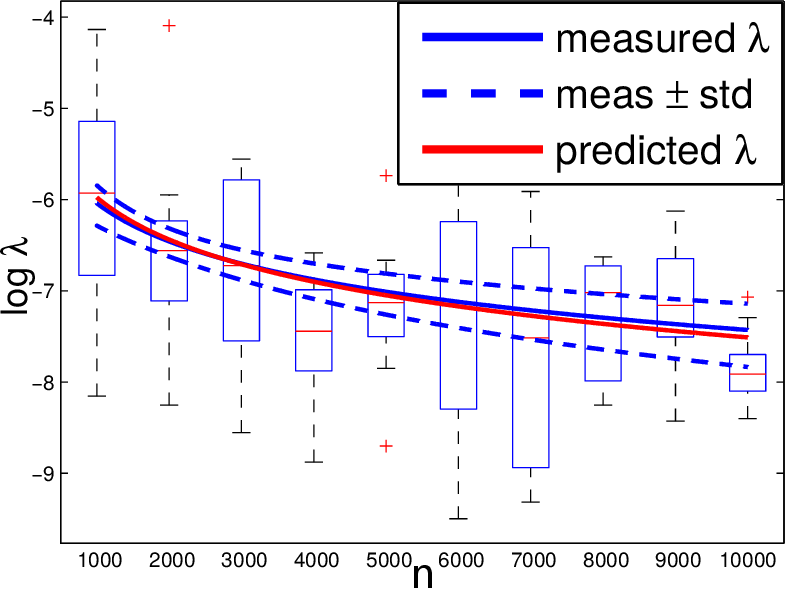}
		\hfill
		\includegraphics[width=0.32\textwidth,height=0.23\textwidth]{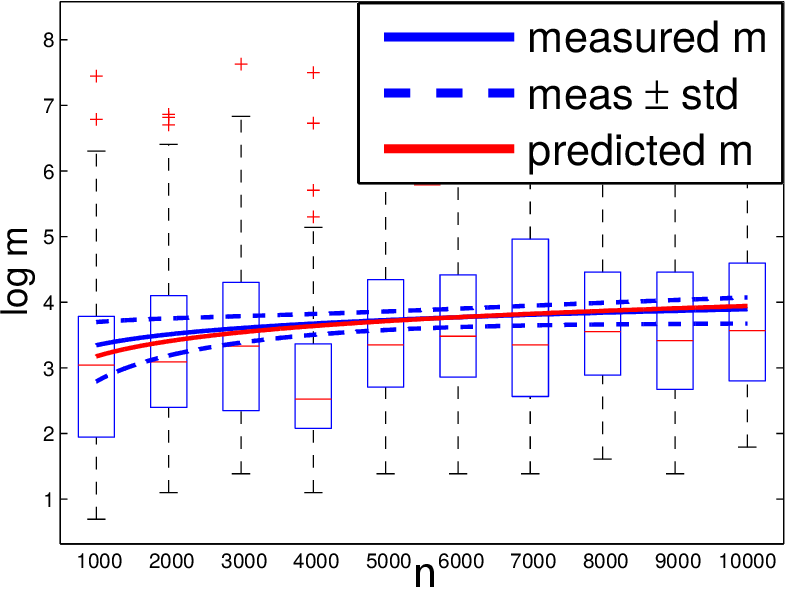}\\
		\includegraphics[width=0.32\textwidth,height=0.23\textwidth]{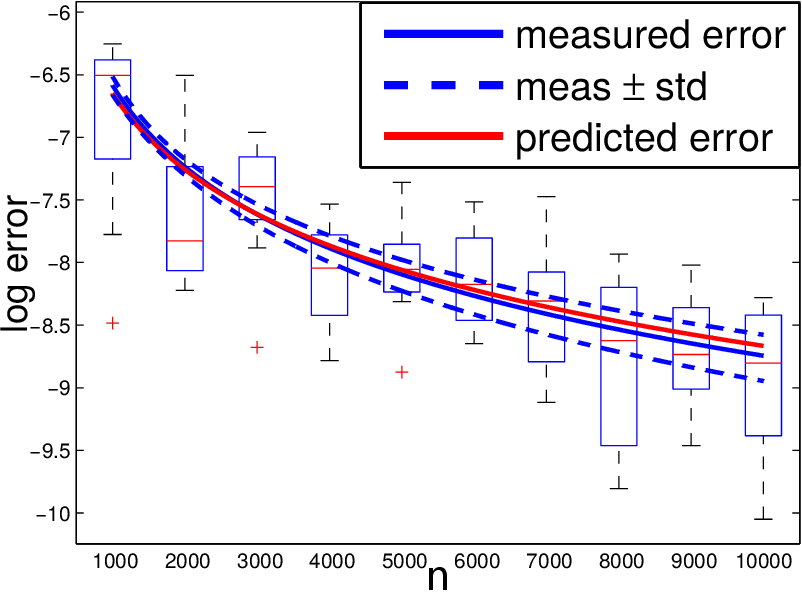}
		\hfill
		\includegraphics[width=0.32\textwidth,height=0.23\textwidth]{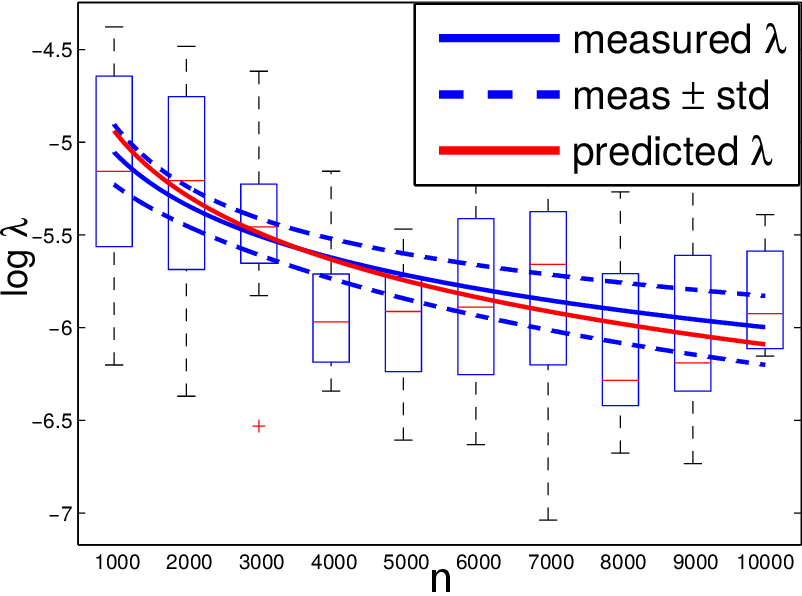}
		\hfill
		\includegraphics[width=0.32\textwidth,height=0.23\textwidth]{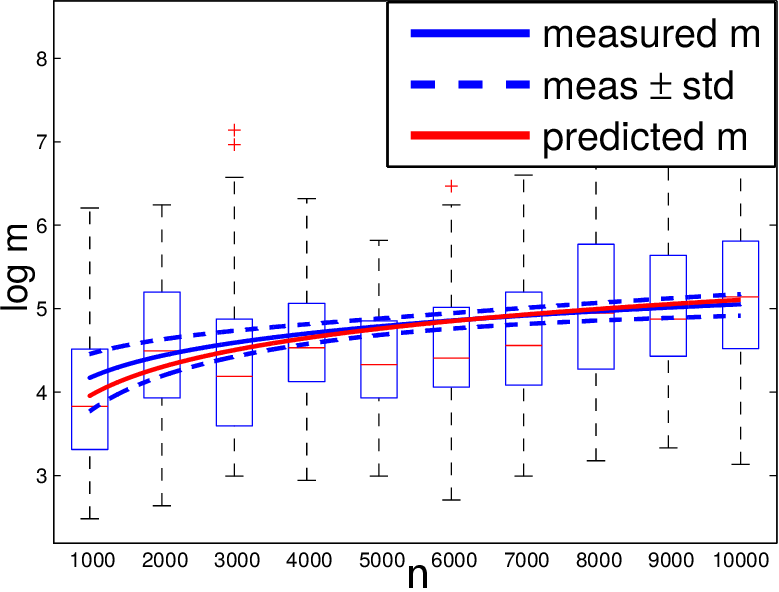}
		\caption{Comparison of theoretical and simulated rates for: excess risk ${\cal E}(\widehat{f}_{\la,M}) - \inf_{f \in \hh}{\cal E}(f)$, $\la$, $M$, w.r.t. $n$ (100 repetitions). Parameters $r = 11/16, \gamma = 1/8$ (top), and $r = 7/8, \gamma = 1/4$ (bottom).\label{fig:sim-1}}
		\vspace{-0.3cm}
	\end{figure}
	\section{Conclusion}
	In this paper, we provide a thorough analyses of the generalization properties of random features with ridge regression.
	We consider a statistical learning theory setting where data are noisy and sampled at random. Our main results show that there are
	large classes of learning problems where random features allow to reduce computations while preserving optimal  statistical accuracy of exact kernel ridge regression.
	This in contrast with previous state of the art results suggesting computational gains needs to be traded-off with statistical accuracy. 
	Our results open several venues for both theoretical and empirical work. As mentioned in the paper, it would be interesting to analyze 
	random features with empirical leverage scores. This is immediate if input points are fixed, but our approach should allow to also consider the statistical learning setting. Beyond KRR, it would be interesting to analyze random features together with other approaches, in particular accelerated and stochastic gradient methods, or distributed techniques. It should be possible to extend the results in the paper to consider these cases. A more substantial generalization would be to consider loss functions other than quadratic loss, since this require different techniques from empirical process theory.

		\paragraph{Acknowledgments}
		The authors gratefully acknowledge the contribution of Raffaello Camoriano  who was involved in the initial phase of this project. These preliminary result appeared in the 2016 NIPS workshop ``Adaptive and Scalable Nonparametric Methods in ML''. This work is funded by the Air Force project FA9550-17-1-0390 (European Office of Aerospace Research and Development) and by the FIRB project RBFR12M3AC (Italian Ministry of Education, University and Research).

	{\small
		\bibliographystyle{unsrt}
		\bibliography{biblio}

\begin{thebibliography}{10}

\bibitem{vapnik1998statistical}
V.~Vapnik.
\newblock {\em Statistical learning theory}, volume~1.
\newblock Wiley New York, 1998.

\bibitem{cucker02onthe}
F.~Cucker and S.~Smale.
\newblock On the mathematical foundations of learning.
\newblock {\em Bulletin of the AMS}, 39:1--49, 2002.

\bibitem{bishop2006}
C.~Bishop.
\newblock {\em Pattern Recognition and Machine Learning}.
\newblock Springer, 2006.

\bibitem{Poggio/Girosi/90}
T.~Poggio and F.~Girosi.
\newblock Networks for approximation and learning.
\newblock {\em Proceedings of the IEEE}, 1990.

\bibitem{pinkus1999approximation}
A.~Pinkus.
\newblock Approximation theory of the mlp model in neural networks.
\newblock {\em Acta Numerica}, 8:143--195, 1999.

\bibitem{schlkopf2002learning}
B.~Sch{\"o}lkopf and A.~J. Smola.
\newblock {\em Learning with Kernels: Support Vector Machines, Regularization,
  Optimization, and Beyond (Adaptive Computation and Machine Learning)}.
\newblock MIT Press, 2002.

\bibitem{aronszajn1950theory}
N.~Aronszajn.
\newblock Theory of reproducing kernels.
\newblock {\em Transactions of the AMS}, 68(3):337--404, 1950.

\bibitem{kimeldorf1970correspondence}
G.~S. Kimeldorf and G.~Wahba.
\newblock A correspondence between bayesian estimation on stochastic processes
  and smoothing by splines.
\newblock {\em The Annals of Mathematical Statistics}, 41(2):495--502, 1970.

\bibitem{scholkopf2001generalized}
B.~Sch{\"o}lkopf, R.~Herbrich, and A.~J. Smola.
\newblock A generalized representer theorem.
\newblock In {\em Computational learning theory}, pages 416--426. Springer,
  2001.

\bibitem{caponnetto2007optimal}
A.~Caponnetto and E.~{De Vito}.
\newblock Optimal rates for the regularized least-squares algorithm.
\newblock {\em FoCM}, 2007.

\bibitem{conf/icml/SmolaS00}
A.~J. Smola and B.~Sch{\"o}lkopf.
\newblock Sparse greedy matrix approximation for machine learning.
\newblock In {\em ICML}, 2000.

\bibitem{conf/nips/WilliamsS00}
C.~Williams and M.~Seeger.
\newblock Using the nystr{\"o}m method to speed up kernel machines.
\newblock In {\em NIPS}, 2000.

\bibitem{conf/nips/RahimiR07}
A.~Rahimi and B.~Recht.
\newblock Random features for large-scale kernel machines.
\newblock In {\em NIPS}, 2007.

\bibitem{conf/colt/Bach13}
F.~Bach.
\newblock Sharp analysis of low-rank kernel matrix approximations.
\newblock In {\em COLT}, 2013.

\bibitem{alaoui2014fast}
A.~Alaoui and M.~Mahoney.
\newblock Fast randomized kernel ridge regression with statistical guarantees.
\newblock In {\em NIPS}. 2015.

\bibitem{rudi2015less}
A.~Rudi, R.~Camoriano, and L.~Rosasco.
\newblock Less is more: Nystr\"{o}m computational regularization.
\newblock In {\em NIPS}. 2015.

\bibitem{sriperumbudur2015}
B.~K. Sriperumbudur and Z.~Szabo.
\newblock Optimal rates for random fourier features.
\newblock {\em ArXiv e-prints}, June 2015.

\bibitem{halko2011finding}
N.~Halko, P.~Martinsson, and J.~A. Tropp.
\newblock Finding structure with randomness: Probabilistic algorithms for
  constructing approximate matrix decompositions.
\newblock {\em SIAM review}, 53(2):217--288, 2011.

\bibitem{plan2014dimension}
Yaniv Plan and Roman Vershynin.
\newblock Dimension reduction by random hyperplane tessellations.
\newblock {\em Discrete \& Computational Geometry}, 51(2):438--461, 2014.

\bibitem{rahimi2009weighted}
Ali Rahimi and Benjamin Recht.
\newblock Weighted sums of random kitchen sinks: Replacing minimization with
  randomization in learning.
\newblock In {\em NIPS}, 2009.

\bibitem{journals/jmlr/CortesMT10}
C.~Cortes, M.~Mohri, and A.~Talwalkar.
\newblock On the impact of kernel approximation on learning accuracy.
\newblock In {\em AISTATS}, 2010.

\bibitem{conf/nips/YangLMJZ12}
T.~Yang, Y.~Li, M.~Mahdavi, R.~Jin, and Z.~Zhou.
\newblock Nystr{\"o}m method vs random fourier features: A theoretical and
  empirical comparison.
\newblock In {\em NIPS}, pages 485--493, 2012.

\bibitem{bach2015}
F.~Bach.
\newblock On the equivalence between quadrature rules and random features.
\newblock {\em ArXiv e-prints}, February 2015.

\bibitem{journals/jmlr/DrineasMMW12}
P.~Drineas, M.~Magdon-Ismail, M.~W. Mahoney, and D.~P. Woodruff.
\newblock Fast approximation of matrix coherence and statistical leverage.
\newblock {\em JMLR}, 13:3475--3506, 2012.

\bibitem{cho2009}
Y.~Cho and L.~K. Saul.
\newblock Kernel methods for deep learning.
\newblock In Y.~Bengio, D.~Schuurmans, J.D. Lafferty, C.K.I. Williams, and
  A.~Culotta, editors, {\em NIPS}, pages 342--350. 2009.

\bibitem{steinwart2008support}
I.~Steinwart and A.~Christmann.
\newblock {\em Support Vector Machines}.
\newblock Springer New York, 2008.

\bibitem{smale2007learning}
S.~Smale and D.~Zhou.
\newblock Learning theory estimates via integral operators and their
  approximations.
\newblock {\em Constructive approximation}, 26(2):153--172, 2007.

\bibitem{smale2003estimating}
S.~Smale and D.~Zhou.
\newblock Estimating the approximation error in learning theory.
\newblock {\em Analysis and Applications}, 1(01):17--41, 2003.

\bibitem{Wahba/90}
G.~Wahba.
\newblock {\em Spline Models for Observational Data}, volume~59 of {\em
  CBMS-NSF Regional Conference Series in Applied Mathematics}.
\newblock SIAM, Philadelphia, 1990.

\bibitem{vito2005learning}
E.~De~Vito, L.~Rosasco, A.~Caponnetto, U.~D. Giovannini, and F.~Odone.
\newblock Learning from examples as an inverse problem.
\newblock In {\em JMLR}, pages 883--904, 2005.

\bibitem{boucheron2004concentration}
S.~Boucheron, G.~Lugosi, and O.~Bousquet.
\newblock Concentration inequalities.
\newblock In {\em Advanced Lectures on Machine Learning}. 2004.

\bibitem{yurinsky1995sums}
V.~V. Yurinsky.
\newblock {\em Sums and Gaussian vectors}.
\newblock 1995.

\bibitem{tropp2012user}
J.~A. Tropp.
\newblock User-friendly tools for random matrices: An introduction.
\newblock 2012.

\bibitem{minsker2011some}
S.~Minsker.
\newblock On some extensions of bernstein's inequality for self-adjoint
  operators.
\newblock {\em arXiv}, 2011.

\bibitem{fujii1993norm}
J.~Fujii, M.~Fujii, T.~Furuta, and R.~Nakamoto.
\newblock Norm inequalities equivalent to heinz inequality.
\newblock {\em Proceedings of the American Mathematical Society}, 118(3), 1993.

\bibitem{caponnetto2006adaptation}
Andrea Caponnetto and Yuan Yao.
\newblock Adaptation for regularization operators in learning theory.
\newblock Technical report, DTIC Document, 2006.

\bibitem{bhatia2013matrix}
Rajendra Bhatia.
\newblock {\em Matrix analysis}, volume 169.
\newblock Springer Science \& Business Media, 2013.

\bibitem{raginsky2009locality}
M.~Raginsky and S.~Lazebnik.
\newblock Locality-sensitive binary codes from shift-invariant kernels.
\newblock In {\em NIPS}, 2009.

\bibitem{kar2012random}
P.~Kar and H.~Karnick.
\newblock Random feature maps for dot product kernels.
\newblock In {\em AISTATS}, 2012.

\bibitem{pham2013fast}
N.~Pham and R.~Pagh.
\newblock Fast and scalable polynomial kernels via explicit feature maps.
\newblock In {\em Proceedings of the 19th ACM SIGKDD international conference
  on Knowledge discovery and data mining}, pages 239--247. ACM, 2013.

\bibitem{conf/icml/LeSS13}
Q.~Le, T.~Sarl{\'o}s, and A.~Smola.
\newblock Fastfood - computing hilbert space expansions in loglinear time.
\newblock In {\em ICML}, 2013.

\bibitem{yang2014random}
J.~Yang, V.~Sindhwani, Q.~Fan, H.~Avron, and M.~Mahoney.
\newblock Random laplace feature maps for semigroup kernels on histograms.
\newblock In {\em Computer Vision and Pattern Recognition (CVPR), 2014 IEEE
  Conference on}, pages 971--978. IEEE, 2014.

\bibitem{hamid2014compact}
R.~Hamid, Y.~Xiao, A.~Gittens, and D.~Decoste.
\newblock Compact random feature maps.
\newblock In {\em ICML}, pages 19--27, 2014.

\bibitem{conf/icml/YangSAM14}
J.~Yang, V.~Sindhwani, H.~Avron, and M.~W. Mahoney.
\newblock Quasi-monte carlo feature maps for shift-invariant kernels.
\newblock In {\em ICML}, volume~32 of {\em JMLR Proceedings}, pages 485--493.
  JMLR.org, 2014.

\bibitem{cotter2011explicit}
Andrew Cotter, Joseph Keshet, and Nathan Srebro.
\newblock Explicit approximations of the gaussian kernel.
\newblock {\em arXiv preprint arXiv:1109.4603}, 2011.

\bibitem{steinwart2006explicit}
Ingo Steinwart, Don Hush, and Clint Scovel.
\newblock An explicit description of the reproducing kernel hilbert spaces of
  gaussian rbf kernels.
\newblock {\em IEEE Transactions on Information Theory}, 52(10):4635--4643,
  2006.

\bibitem{vedaldi2012efficient}
A.~Vedaldi and A.~Zisserman.
\newblock Efficient additive kernels via explicit feature maps.
\newblock {\em Pattern Analysis and Machine Intelligence, IEEE Transactions
  on}, 34(3):480--492, 2012.

\end{thebibliography}
	}
	
	\newpage 
	
	\appendix

	{\Large \bf Generalization Properties of Learning with Random Features}
	
	{\Large \bf Supplementary Materials}
	
	$ $\\
	The supplementary materials are divided in the following four section
	
	{\em A. Proofs} - where the proofs for Section 3 are provided
	
	{\em B. Concentration Inequalities} - where probabilistic tools necessary for the proofs are recalled
	
	{\em C. Operator Inequalities} - where some analytic inequalities used in the proofs are recalled
	
	{\em D. Auxiliary Results} - where some technical lemmas necessary to the proof are derived
	
	{\em E. Examples of Random Features} - where examples of random features expansion are recalled

	\section{Proofs}\label{sect:proofs}
	In Sect.~\ref{sect:notation}, the notation is introduced and some standard identities are recalled. In Sect.~\ref{sect:analytic-results}, the excess risk is decomposed in five terms (Eq.~\eqref{eq:excess-risk-expansion}-\eqref{eq:excess-risk-expansion-term5}) that are further simplified in Lemma~\ref{lm:dec-term2},~\ref{lm:dec-term3},~\ref{lm:dec-term4},~\ref{lm:dec-term5}. The complete decomposition is presented in Thm.~\ref{thm:geo-dec}. In Sect.~\ref{sect:prob-estimates}, the terms in decomposition are bounded in probability, in particular Lemma~\ref{lm:conc-S} bounds the variance term, Lemma~\ref{lm:conc-C} the computational error term, while Lemma~\ref{lm:conc-beta} controls the constants. Finally the proofs of the main results are presented in Section~\ref{sect:proof-main-res} together with the more general results of Thm.~\ref{thm:main-bound}. 
	
	First we recall the assumptions needed to derive the results. They are already presented or implied in the main text, here we collect and number them.
	
	{\bf Assumption~\ref{ass:compatibility}} (Compatibility condition) {\em
		There exists $\alpha \in [0, 1]$ and $F > 0$ such that
		$$ {\cal F}_\infty(\la) \leq F \la^{-\alpha}, \quad  \forall \la > 0.$$
	}
	
	\ba[Random Features are bounded and continuous]\label{ass:kernel-bounded}
	The kernel $K$ has an integral representation as in Eq.~\ref{eq:def-RF-integral}, with $\psi$ continuous in both variables and bounded, that is, there exists $\kappa \geq 1$ such that $|\psi(x,\omega)| \leq \kappa$ for any $x, \in \X$ and $\omega \in \Omega$. The associated RKHS $\hh$ is separable.
	\ea
	Note that the assumption above is satisfied when the random feature is continuous and bounded and the space $X$ is separable (e.g. $\R^d$, $d \in \N$ or any Polish space). Indeed the continuity of $\psi$ implies the continuity of $K$, which, together with the separability of $X$ implies the separability of $\hh$.
	
	\ba[Noise on the $y$ is sub-exponential, and there exists $\fh$]\label{ass:noise}
	For any $x \in \X$
	$$\mathbb{E}[|y|^p~|~ x] \leq \frac{1}{2} p! \sigma^2 B^{p-2}, \quad \forall p \geq 2.$$  
	Moreover there exists $\fh \in \hh$ such that ${\cal E}(\fh) = \inf_{f \in \hh} {\cal E}(f)$.
	\ea
	Note that the above assumption on $y$ is satisfied when $y$ is bounded, sub-gaussian or sub-exponential. In particular, if $|y| \in [-\frac{b}{2}, \frac{b}{2}]$ almost surely, with $b \in (0, \infty)$ then the assumption above is satisfied with $\sigma = B = b$. 
	
	\ba[Effective dimension]\label{ass:intrinsic} Let $\la > 0$.
	There exists $Q > 0$ and $\gamma \in [0,1]$ such that, for any $\la > 0$
	$$ {\cal N}(\la) \leq Q^2 \la^{-\gamma}.$$
	\ea
	It is the first part of Asm.~\ref{ass:prior}, for the sake of clarity we need to split it in two, since many results depend either on the first or on the second part.
	
	\ba[Source condition]\label{ass:source}
	There exists $1/2 \leq r \leq 1$ and $g \in \Ltwo$ such that
	$$ \fh(x) = (L^r g)(x) \quad \textrm{a.s.}$$
	We denote with $R$ the quantity $1 \vee \nor{g}_\rhox$.
	\ea

	\subsection{Kernel and Random Features Operators}\label{sect:notation}
	In this section, we provide the notation, recall some useful facts and define some operators used in the rest of the appendix.
	In the rest of the paper we denote with $\nor{\cdot}$ the operatorial norm and with $\nor{\cdot}_{HS}$ the Hilbert-Schmidt norm. Let ${\cal L}$ be a Hilbert space, we denote with $\scal{\cdot}{\cdot}_{\cal L}$ the associated inner product, with $\nor{\cdot}_{\cal L}$ the norm and with $\tr(\cdot)$ the trace. Let $Q$ be a bounded self-adjoint linear operator on a separable Hilbert space ${\cal L}$, we denote with $\la_{\max}(Q)$ the biggest eigenvalue of $Q$, that is
	$
	\la_{\max}(Q) = \sup_{\nor{f}_{\cal L} \leq 1} \scal{f}{Q f}_{\cal L}.
	$
	Moreover, we denote with $Q_\la$ the operator $Q + \la I$, where $Q$ is a linear operator, $\la \in \R$ and $I$ the identity operator, so for example $\tCnl := \tCn + \la I$.
	Moreover we recall some basic properties of norms in Hilbert spaces. 
	\br\label{rem:dec-norms}
	Let $V_0, \dots, V_t$ with $t \in \N$ be Hilbert spaces. Let $q \in V_0$ and $A_i: V_i \to V_{i-1}$ bounded linear operators and $f \in V_t$. We recall that the identity $q = (A_1)\cdots(A_t)(f)$, implies $\nor{q}_{V_0} \leq \nor{A_1}\dots \nor{A_t} \nor{f}_{V_t}$.
	\er 
	
	Let $\X$ be a probability space and $\rho$ be a probability distribution on $\X\times\R$ satisfying Assumption~\ref{ass:kernel-bounded}. We denote $\rhox$ its marginal on $\X$ and $\rho(y|x)$ the conditional distribution on $\R$. Let $\Ltwo$ be the Lebesgue space of square $\rhox$-integrable functions, with the canonical inner product
	$$\scal{g}{h}_\rhox = \int_{\X} g(x)h(x) d\rhox(x),\quad \forall g,h \in \Ltwo,$$
	and the norm $\nor{g}_\rhox^2 = \scal{g}{g}_\rhox$, for all $g \in \Ltwo$.
	Let $(\Omega, \pi)$ be a probability space and $\psi:\Omega \times \X \to \CC$ be a continuous and bounded map as in Asm.~\ref{ass:kernel-bounded}. Moreover let the kernel $K$ be defined by Eq.~\eqref{eq:def-RF-integral}. We denote with $K_x$ the function $K(x,\cdot)$, for any $x \in \X$. Then the Reproducing Kernel Hilbert Space $\hh$ induced by $K$ is defined by 
	$$\hh = \lspanc{K_x}{x \in \X},\quad \textrm{completed with}\quad \scal{K_x}{K_{x'}}_\hh = K(x,x')~~\forall x,x' \in \X.$$

	We now define the operators needed in the rest of the proofs. Let $n \in \N$, and $(x_1,y_1),\dots,(x_n,y_n) \in X\times\R$ be sampled independently according to $\rho$. 
	
	\bd\label{def:all-ops}
	Let $P: \Ltwo \to \Ltwo$ be the projection operator with the same range of $L$.
	Let $\frho: \X \to \R$ be defined as 
	$$ \frho(x) = \int y d\rho(y|x) ~~ \textrm{a. e.}$$
	\ed
	We now recall a useful characterization of the excess risk, in term of the quantities defined above.
	\br[from \cite{cucker02onthe,vito2005learning}]\label{rem:excess-risk-to-Ltwo}
	When $\int y^2 d\rho$ is finite, then $\frho \in \Ltwo$ and $\frho$ is the minimizer of ${\cal E}$ over all the measurable functions. When $\int K(x,x) d\rhox$ is finite, the range of $P$ and of $L$ is the closure of $\hh$ in $\Ltwo$. When both conditions hold, for any $f \in \Ltwo$ the following hold
	$${\cal E}(f) - \inf_{g \in \hh} {\cal E}(g) = \nor{f - P \frho}_{\rhox}^2 + 2\scal{f}{(I-P)f_\rho}_\rho.$$
	The latter term is zero if $f\in \hh$, but we will also see that it is zero for all the functions defined by $M$ random features.
	Moreover if there exists $\fh \in \hh$ minimizing ${\cal E}$, then Asm.~\ref{ass:source} is equivalent to requiring the existence of $r \geq 1/2$, $g \in \Ltwo$ such that
	\eqal{\label{eq:Pfro-Lrg}
		P \frho = L^r g,
	}
	with $R := \nor{g}_\Ltwo$.
	\er 
	In the following we define analogous operators for the approximated kernel $K_M:= \phi_M(x)^\top \phi_M(x')$, with 
	$$\phi_M(x) := M^{-1/2}(\psi(x,\omega_1),\dots,\psi(x, \omega_M)),$$ for any $x,x' \in \X$,
	where $M \in \N$ and $\omega_1,\dots,\omega_M \in \Omega$ are sampled independently according to $\pi$. We denote with $\psi_{\omega}$ the function $\psi(\cdot, \omega)$ for any $\omega \in \Omega$. According to the following remark, we have that $\psi_{\omega_i} \in \Ltwo$ almost surely.
	\br\label{rm:psi-in-ltwo}
	Under Asm.~\ref{ass:kernel-bounded} and the fact that $\rho$ is a finite measure, $\psi_\omega \in \Ltwo$ almost surely.
	\er 
	Now we are ready for defining the following operators, depending on $\phi_M$ or $K_M$.
	\bd\label{def:all-ops-tilde}
	For all $g \in \Ltwo$, $\beta \in \CC^M$, $\alpha \in \CC^n$ and $i \in \{1,\dots,M\}$, we have
	\begin{itemize}
		\item $\tS: \CC^M \to \Ltwo, \quad (\tS \beta)(\cdot) = \phi_M(\cdot)^\top \beta$,
		\item $\tS^*: \Ltwo \to \CC^M, ~~ (\tS^*g)_i = \frac{1}{\sqrt{M}}\int_X \psi_{\omega_i}(x) g(x)d\rhox(x)$,
		\item $\tL: \Ltwo \to \Ltwo, \quad (\tL g)(\cdot) = \int_\X K_M(\cdot,z)g(z)d\rhox(z)$.
		\item $\tC: \CC^M \to \CC^M, \quad \tC = \int_X \phi_M(x) \phi_M(x)^\top d\rhox(x)$,
		\item $\tCn: \CC^M \to \CC^M, \quad \tCn = \frac{1}{n} \sum_{i=1}^n \phi_M(x_i)\phi_M(x_i)^\top$.
	\end{itemize}
	\ed
	Note that the operators above satisfy the properties in the following remark.
	\br[from \cite{caponnetto2007optimal}]\label{rem:S-C-L}
	Under Asm.~\ref{ass:kernel-bounded} the linear operators $L$ is trace class and the linear operators $\tL, \tC, \tS, \tCn, \tSn$ are finite dimensional. Moreover we have that $L = SS^*$, $\tL = \tS \tS^*$, $\tC = \tS^*\tS$ and $\tCn = \tSn^*\tSn$. Finally $L, \tL, \tC, \tCn$ are self-adjoint and positive operators, with spectrum is $[0, \kappa^2]$.
	\er  
	In the next remark we rewrite $\widehat{f}_{\la, M}$ in terms of the operators introduced above.
	\br\label{rem:rf-in-ltwo}
	Let $\widehat{f}_{\la, M}$ defined as in Eq.~\ref{eq:algo-rf}. Under Assumption~\ref{ass:kernel-bounded}, $\widehat{f}_{\la, M} \in \Ltwo$ almost surely, since $\psi_\omega$ is in $\Ltwo$ almost surely (Rem.~\ref{rm:psi-in-ltwo}) and $\widehat{f}_{\la, M}$ is a linear combination of $\psi_{\omega_1},\dots, \psi_{\omega_M}$. In particular, 
	$$ \widehat{f}_{\la, M} = \tS \tCnl^{-1} \tSn^* \yn.$$
	\er
	%
	%
	
	
	\subsection{Analytic Result}\label{sect:analytic-results}
	In this subsection we decompose analytically the excess risks in different terms, that will be bounded, via concentration inequalities, in the next section. 
	Under Asm.~\ref{ass:kernel-bounded}, since $\widehat{f}_{\la, M} \in \Ltwo$ almost surely, we have 
	\eqal{\label{eq:excess-risk-to-Ltwo}
		\EE(\widehat{f}_{\la,M}) - \inf_{f \in \hh}\EE(f) = \nor{\widehat{f}_{\la,M} - P\frho}^2_\rhox+ 2\scal{   \widehat{f}_{\la,M}}{(I-P)f_\rho}_\rho,
	}
	(for more details see Rem.~\ref{rem:excess-risk-to-Ltwo},~\ref{rem:rf-in-ltwo}).
	In our analysis we decompose the first term considering, 
	\eqal{\label{eq:excess-risk-expansion}
		\widehat{f}_{\la,M} - P \frho  &= \widehat{f}_{\la,M} -  \tS \tCnl^{-1}\tS^*\frho  \\
		& \quad + ~~ \label{eq:excess-risk-expansion-term2} \tS \tCnl^{-1}\tS^*(I-P)\frho \\
		& \quad + ~~ \label{eq:excess-risk-expansion-term3} \tS \tCnl^{-1}\tS^* P \frho - \tL\tLl^{-1} P \frho \\
		& \quad + ~~ \label{eq:excess-risk-expansion-term4} \tL\tLl^{-1} P \frho - \L \Ll^{-1} P \frho \\
		& \quad + ~~ \label{eq:excess-risk-expansion-term5} \L \Ll^{-1} P \frho - P \frho, 
	}
	and further show that almost surely
	$$
	\scal{   \widehat{f}_{\la,M}}{(I-P)f_\rho}_\rho=0.
	$$
	We comment on the role of the various terms in the decomposition.
	The first controls the variance of the outputs $y$, the second the interaction between the space of models spanned by $\psi_{\omega_1},\dots,\psi_{\omega_M}$ and $\hh$, the third the approximation of the inverse covariance operator $\tCnl^{-1}$, the fourth controls how close is the integral operator $L_M$ to $L$, while the last controls the approximation error of the models in $\hh$. The $\Ltwo$ norm of $\widehat{f}_{\la, M} - P\frho$ is bounded by the sum of the $\Ltwo$ norms of the terms, that are further bounded in Lemma.~\ref{lm:dec-term2},~\ref{lm:dec-term3},~\ref{lm:dec-term4},~\ref{lm:dec-term5}.
	The final analytical decomposition is given in Thm.~\ref{thm:geo-dec}. First we need a preliminary result.
	\blm\label{lm:L-characterization}
	Under Asm.~\ref{ass:kernel-bounded}, the operator $L$ is characterized by
	$$ L  = \int \psi_\omega \otimes \psi_\omega d\pi(\omega).$$
	\elm
	\bpr
	By Asm.~\ref{ass:kernel-bounded}, we have that $\psi_\omega \in \Ltwo$ almost surely and uniformly bounded. By using the kernel expansion of Eq.~\eqref{eq:def-RF-integral}, the linearity of the Bochner integral and of the dot product, we have that for any $f, g \in \Ltwo$ the following holds
	\eqals{
		\scal{f}{L g}_\rhox &= \int f(x) K(x,z) g(z) d\rhox(x)d\rhox(z) \\
		& = \int f(x) \psi(x,\omega) \psi(z,\omega) g(z) d\rhox(x)d\rhox(z) d\pi(\omega) \\
		& = \int \scal{f}{\psi_\omega}_\rhox \scal{g}{\psi_\omega}_\rhox ~d\pi(\omega) = \scal{f}{\int \psi_\omega  \scal{g}{\psi_\omega}_\rhox ~d\pi(\omega)}_\rhox \\
		& = \scal{f}{\left(\int \psi_\omega \otimes \psi_\omega d\pi(\omega)\right) g}_\rhox.
	}
	\epr
	Now we are ready to prove that the second term of the expansion in Eq.~\ref{eq:excess-risk-expansion} is zero. We obtain this result by proving that $\nor{(I-P)\psi_\omega} = 0$ almost everywhere.
	\blm\label{lm:dec-term2}
	Under Asm.~\ref{ass:kernel-bounded}, the following holds for any $\la > 0, M, n \in \N$,
	$$\nor{\tS \tCnl^{-1}\tS^*(I-P)\frho}_{\rhox} = 0~~\textrm{a. s.},$$
	and moreover
	$$
	\scal{   \widehat{f}_{\la,M}}{(I-P)f_\rho}_\rho=0~~\textrm{a. s.},
	$$
	where the latter result holds more generally for any $f\in \text{range}(\tS).$
	\elm
	\bpr 
	Note that, since $P$ is the projection operator on the range of $L$ and $L$ is trace class, then $(I - P)L = 0$, this implies that $\tr((I-P)L(I-P)) = 0$.
	By the characterization of $L$ given in Lemma~\ref{lm:L-characterization}, the linearity of the bounded operator $I-P$ and of the trace, we have that
	\eqals{
		0 & = \tr \left((I-P)L(I-P)\right)  = \tr\left((I-P)\left(\int \psi_\omega \otimes \psi_\omega d\pi(\omega) \right)(I-P)\right)\\
		& = \int \tr\left((I-P) (\psi_\omega \otimes \psi_\omega) (I-P) \right) d\pi(\omega) \\
		& = \int \nor{(I-P) \psi_\omega}_\rhox^2 d\pi(\omega),
	}
	where the last step is due to the fact that  $\tr(A(v \otimes v) A) = \tr(Av ~ \otimes Av) = \nor{Av}_\rhox^2$ for any bounded self adjoint operator $A$ and any function $v \in \Ltwo$. The equation above implies that $(I-P) \psi_\omega = 0$ almost surely on the support of $\pi$.
	Now we study $\tS^*(I-P)$, for any $\beta \in \R^M$ and any $f \in \Ltwo$ we have
	\eqals{
		\scal{\beta}{\tS^*(I-P)f}_{\R^M} &= \frac{1}{\sqrt{M}} \sum_{i=1}^M \beta_i \scal{(I-P)\psi_{\omega_i}}{f}_\rhox = 0 ~~\textrm{a. s.},
	}
	where the last step is due to the fact that $\scal{0}{v} = 0$ for any $v$ and  $(I-P)\psi_{\omega_i} = 0$ almost surely, since $\omega_i$ are distributed according to $\pi$ and $(I-P)\psi_\omega = 0$ almost surely on the support of $\pi$. Now
	$$\nor{\tS \tCnl^{-1} \tS^*(I-P) \frho}_\rhox \leq \nor{\tS \tCnl^{-1}} \nor{\tS^*(I-P)}\nor{\frho}_\rhox = 0 ~~\textrm{a. s.}$$
	Similarly for  $f\in \text{range}(\tS)$ it exists $\beta \in \R^M$ such that $f= \tS \beta$, so that  
	$$\scal{f}{(I-P)f_\rho}_\rho= \scal{\tS\beta}{S(I-P)f_\rho}_\rho=\scal{\beta}{\tS^*(I-P)f_\rho}_{\R^M}\le \nor{\beta}_{\R^M} \nor{\tS^*(I-P)} \nor{f_\rho}_\rho
	$$\epr
	\blm\label{lm:dec-term3}
	Under Asm.~\ref{ass:kernel-bounded}, and Eq.~\eqref{eq:Pfro-Lrg} the following holds for any $\la > 0, M,n \in \N$
	$$ \nor{\tS \tCnl^{-1}\tS^* P \frho - \tL\tLl^{-1} P \frho} \leq R \kappa^{2r-1} \nor{\tLl^{-1/2}\L^{1/2}}\nor{\tS \tCnl^{-1}\tCl^{-1/2}} \nor{\tCl^{-1/2} (\tC - \tCn)}.$$
	\elm
	\bpr
	First of all we recall that $Z^*f(ZZ^*) = f(Z^*Z)Z^*$ for any continuous spectral function and any compact operator $Z$. By the characterization of $\tL$ in Rem.~\ref{rem:S-C-L} under Asm.~\ref{ass:kernel-bounded}, we have
	$$\tL\tLl^{-1} = \tS\tS^*(\tS\tS^* + \la I)^{-1} = \tS (\tS^* \tS + \la I)^{-1}\tS^* = \tS \tCl^{-1}\tS^*,$$
	since $(\cdot + \la I)^{-1}$ is a continuos spectral function on $[0, \infty)$, which contains the spectrum of $L$ that is in $[0,\kappa^2]$. Equivalently, the equation above could be proven algebraically via the Woodbury identity. Now we have 
	$$ (\tS \tCnl^{-1}\tS^* - \tL\tLl^{-1}) P \frho = \tS(\tCnl^{-1} - \tCl^{-1})\tS^* P \frho = \tS\tCnl^{-1}(\tC - \tCn)\tCl^{-1}\tS^* P \frho,$$
	where the last step is due to the identity $A^{-1} - B^{-1} = A^{-1}(B-A)B^{-1}$ valid for any bounded invertible linear operator $A, B$. In particular by multiplying and dividing by $\tCl^{1/2}$ we have the following decomposition
	\eqals{
		\tS\tCnl^{-1}(\tC - \tCn)\tCl^{-1}\tS^* P \frho ~~=~~ (\tS\tCnl^{-1}\tCl^{1/2}) ~~ (\tCl^{-1/2}(\tC - \tCn)) ~~ (\tCl^{-1}\tS^* P \frho).}
	The result is given by bounding the norm of the lhs of the identity above, by the product of the norms of the parentheses on the rhs (see Rem.~\ref{rem:dec-norms}). 
	Note that by applying Eq.~\eqref{eq:Pfro-Lrg}, we have that there exists $g \in \Ltwo$, such that $P \frho = L^r g$ and by dividing and multiplying for $\tLl^{1/2}$, we have $$\tCl^{-1}\tS^* P \frho = (\tCl^{-1}\tS^* \tLl^{1/2}) ~ (\tLl^{-1/2}\L^{1/2}) ~ L^{r-1/2}~g.$$ 
	Now note that, by Asm.~\ref{ass:source}, we have $r \geq 1/2$, $\nor{g}_\rhox \leq R$ and $\nor{L^{r-1/2}} \leq \kappa^{2r-1}$ since $L$ is compact with the spectrum in $[0,\kappa^2]$ and $2r-1 \geq 0$. 
	By the fact that $( \cdot + \la I)^{-2}$ is a continuous spectral function on $[0,\infty)$ containing the spectrum of $\tC$, we have that $\tS \tCl^{-2} \tS^* = \tLl^{-2}\tL$ and so for any $\la > 0$ 
	$$\nor{\tCl^{-1}\tS^*\tL^{1/2}}^2 =  \nor{\tL^{1/2} \tS \tCl^{-2} \tS^* \tL^{1/2}} = \nor{\tLl^{-2}\tL^{2}} \leq 1.$$
	\epr
	\blm\label{lm:dec-term4}
	Under Asm.~\ref{ass:kernel-bounded}, and Eq.~\eqref{eq:Pfro-Lrg} the following holds for any $\la > 0, M \in \N$
	$$ \nor{(\L \Ll^{-1} - \tL\tLl^{-1}) P \frho} \leq  R\sqrt{\la} \nor{\tLl^{-1/2}\Ll^{1/2}} \nor{\Ll^{-1/2}(L - \tL)}^{2r-1}\nor{\Ll^{-1/2}(L - \tL)\Ll^{-1/2}}^{2-2r}$$
	\elm
	\bpr 
	By the algebraic identities $A(A + \la I) = I - \la (A + \la I)^{-1}$ valid for any bounded positive operator and $A^{-1} - B^{-1} = A^{-1}(B - A)B^{-1}$ valid for any invertible bounded operators, we have
	$$(\L \Ll^{-1} - \tL\tLl^{-1})  P \frho  = \la(\tLl^{-1} - \Ll^{-1})  P \frho = \la\tLl^{-1}(L  - \tL)\Ll^{-1}  P \frho.$$ 
	By applying Eq.~\eqref{eq:Pfro-Lrg}, we have that there exists $g \in \Ltwo$, such that $P \frho = L^r g$, so by multiplying and dividing by $\tL^{1/2}$, we preform the following decomposition
	\eqals{
		(\L \Ll^{-1} - \tL\tLl^{-1})P\frho ~=~ \sqrt{\la} ~ (\sqrt{\la}\tLl^{-1/2}) ~ (\tLl^{-1/2}\Ll^{1/2})~(\Ll^{-1/2}(L - \tL)\Ll^{-(1-r)})~(\Ll^{-r}\L^r)~g.
	}
	The result is given by bounding the norm of the lhs of the identity above, by the product of the norms of the parentheses on the rhs (see Rem.~\ref{rem:dec-norms}).
	Note that $\nor{\sqrt{\la}\tLl^{-1/2}} \leq 1$ and $\nor{\Ll^{-r}\L^r} \leq 1$ for any $\la >0$ and $\nor{g}_\rhox \leq R$.
	Now we apply Proposition~\ref{prop:interp} on $\nor{\Ll^{-1/2}(L - \tL)\Ll^{-(1-r)}}$, indeed note that $0 \leq 1 - r \leq 1/2$, so by setting $\sigma = 2-2r$, $X = \Ll^{-1/2}(L - \tL)$, $A = \Ll^{-1/2}$ and applying the proposition, we have
	$$\nor{\Ll^{-1/2}(L - \tL)\Ll^{-\sigma/2}} \leq \nor{\Ll^{-1/2}(L - \tL)}^{2r-1}\nor{\Ll^{-1/2}(L - \tL)\Ll^{-1/2}}^{2-2r}.$$
	\epr
	\blm\label{lm:dec-term5}
	Under Asm.~\ref{ass:kernel-bounded}, and Eq.~\eqref{eq:Pfro-Lrg} the following holds for any $\la > 0$,
	$$\nor{\L \Ll^{-1} P \frho - P \frho} \leq R \la^r.$$
	\elm
	\bpr
	By the identity $A(A + \la I)^{-1} = I - \la (A + \la)^{-1}$ valid for $\la > 0$ and any bounded self-adjoint positive operator and by Eq.~\eqref{eq:Pfro-Lrg} for which there exists $g \in \Ltwo$ such that $P\frho = \L^r g$, we have
	\eqal{\label{eq:further-dec-term5}
		(I - \L \Ll^{-1}) P \frho = \la \Ll^{-1} P \frho = \la \Ll^{-1} L^r g ~~=~~ \la^r ~~ (\la^{1-r}\Ll^{-(1-r)}) ~~ (\Ll^{-r}\L^r) ~~ g.}
	The result is given by bounding the norm of the lhs of Eq.~\ref{eq:further-dec-term5}, by the product of the norms of the parentheses on the rhs (see Rem.~\ref{rem:dec-norms}).
	Note that $\nor{\la^{1-r}\Ll^{-(1-r)}} \leq 1$ and $\nor{\Ll^{-r}\L^r} \leq 1$, while $R := \nor{g}_\rhox$ according to Eq.~\eqref{eq:Pfro-Lrg}.
	\epr 
	\bt[Analytic Decomposition]\label{thm:geo-dec}
	Under Assumptions~\ref{ass:kernel-bounded} and Eq.~\eqref{eq:Pfro-Lrg} let $\widehat{f}_{\la, M}$ as in Eq.~\ref{eq:algo-rf}. For any $\la > 0$ and $M \in \N$, the following holds
	\eqal{
		|\EE(\widehat{f}_{\la, M}) - \inf_{f \in \hh}\EE(f)|^{1/2}  \;\;\leq\;\; \beta\;(\;\;\underbrace{{\cal S}(\la, M, n)}_\textrm{Sample Error}\;\; + \underbrace{{\cal C}(\la,M)}_\textrm{Computational Error} + \underbrace{R\la^{v}}_\textrm{Approximation Error})
	}
	where $v = \min(r,1)$, 
	\begin{enumerate}
		\item ${\cal S}(\la,M,n) := \nor{\tCl^{-1/2}(\tSn^*\widehat{y} - \tS^*\frho)}_\rhox + R\kappa^{2r-1}\nor{\tCl^{-1/2}(\tC - \tCn)}$,
		\item ${\cal C}(\la, M) := R\sqrt{\la}\nor{\Ll^{-1/2}(L - \tL)}^{2v-1}\nor{\Ll^{-1/2}(L - \tL)\Ll^{-1/2}}^{2-2v}$,
		\item $\beta := \max(1,(1 - \beta_1)^{-1}) \max(1,(1-\beta_2)^{-1/2})$, with $\beta_1 := \la_{\max}(\tCl^{-1/2}(\tC - \tCn)\tCl^{-1/2})$ and $\beta_2 := \la_{\max}(\Ll^{-1/2}(L - \tL)\Ll^{-1/2})$.
	\end{enumerate}
	\et
	\bpr
	Under Asm.~\ref{ass:kernel-bounded}, the excess risk is characterized by Eq.~\ref{eq:excess-risk-to-Ltwo} as we recalled at the beginning of this subsection. We decomposed the quantity $\widehat{f}_{\la, M} - P\frho$ according to the terms in Eq.~\ref{eq:excess-risk-expansion}-\ref{eq:excess-risk-expansion-term5}. The $\Ltwo$ norm of $\widehat{f}_{\la, M} - P\frho$ is bounded by the sum of the $\Ltwo$ norms of the terms, that are further bounded in Lemma.~\ref{lm:dec-term2},~\ref{lm:dec-term3},~\ref{lm:dec-term4},~\ref{lm:dec-term5}.
	In particular, for the first term, by writing $\widehat{f}_{\la,M}$ in terms of the linear operators in Def.~\ref{def:all-ops-tilde} (see Rem.~\ref{rem:rf-in-ltwo}) and by multipling and dividing by $\tCl^{1/2}$ we have
	$$\widehat{f}_{\la,M} - \tS\tCnl^{-1}\tS^*\frho = \tS\tCnl^{-1}(\tSn^*\yn - \tS^*\frho) ~~ = ~~ (\tS\tCnl^{-1}\tCl^{1/2}) ~~ (\tCl^{-1/2}(\tSn^*\yn - \tS^*\frho)),$$ 
	then we bound the norm of the term with the norm of the parenthesis in the decomposition above (see Rem.~\ref{rem:dec-norms}).
	By collecting the result above with the bounds in Lemma.~\ref{lm:dec-term2},~\ref{lm:dec-term3},~\ref{lm:dec-term4},~\ref{lm:dec-term5}, we have
	$$|\EE(\widehat{f}_{\la, M}) - \inf_{f \in \hh}\EE(f)|^{1/2} \leq 
	b_1 A  + b_1 b_2 B + b_3 {\cal C}(\la, M) + D,$$
	where $b_1 := \nor{\tS\tCnl^{-1}\tCl^{1/2}}$, $b_2 := \nor{\tLl^{-1/2}L^{1/2}}$, $b_3 :=\nor{\tLl^{-1/2}L_\la^{1/2}}$, \\ $A := \nor{\tCl^{-1/2}(\tSn^*\widehat{y} - \tS^*\frho)}_\rhox$, $B:= R\kappa^{2r-1}\nor{\tCl^{-1/2}(\tC - \tCn)}$ and $D := R\la^r$.
	Now note that $b_2 \leq b_3$ for any $\la > 0$ since, for any $X,T$, bounded linear operators, with $T$ positive, by multiplying and dividing for $T_\la$ the following holds
	\eqal{\label{eq:ABla-geq-AB}
		\nor{XT} \leq \nor{XT_\la}\nor{T_\la^{-1}T},
	}
	and $\nor{T_\la^{-1}T} \leq 1$, for any $\la > 0$. Since $A, B, {\cal C}(\la, M), D$ will contribute to the rates of the bound, while $b_1, b_3$ are responsible for the numerical constants, we are going to bound the excess risk in order to collect $b_1, b_3$ in a multiplicative term as follows 
	$$|\EE(\widehat{f}_{\la, M}) - \inf_{f \in \hh}\EE(f)|^{1/2} \leq \max(1,b_1)\max(1,b_3)(A + B + {\cal C}(\la, M) + D).$$
	Finally, we further simplify $b_1, b_3$, in particular we apply Prop.~\ref{prop:invAtoinvB} in the appendix, obtaining $b_3 \leq (1-\beta_2)^{-1/2}$. For $b_1$ note that,
	$$\nor{\tS\tCnl^{-1}\tCl^{1/2}} \leq \nor{\tS\tCnl^{-1/2}}\nor{\tCnl^{-1/2}\tCl^{1/2}} \leq \nor{\tCnl^{-1/2}\tCl^{1/2}}^2,$$
	since, $\nor{\tS\tCnl^{-1/2}} \leq \nor{\tCl^{1/2}\tCnl^{-1/2}}$ for the same reasoning in Eq.~\eqref{eq:ABla-geq-AB}. Then we apply Prop.~\ref{prop:invAtoinvB} in the appendix, obtaining $b_1 \leq (1-\beta_1)^{-1}$.
	\epr
	In the next subsection, we are going to find probabilistic estimates for terms in the analytic decomposition of the excess risk in Thm.~\ref{thm:geo-dec}.
	\subsection{Probabilistic Estimates}\label{sect:prob-estimates}
	In this section we provide bounds in probability for the quantities $\beta, {\cal S}, {\cal C}$ of Thm.~\ref{thm:geo-dec} and for the empirical effective dimension. The notation is introduced in Sect.~\ref{sect:notation}. 
	First, we fix the notation on the random variables used in the rest of the subsection.
	Recall that $\beta, {\cal S}, {\cal C}$ are expressed with respect to the random variables $\bzeta:=((x_1, y_1), \dots, (x_n, y_n))$, and $\bomega := (\omega_1, \dots, \omega_M)$. The associated sample space is $W := Z \times \Omega^M$ and $Z := (\X \times \R)^n$, with probability measure $\PP := \rho^{\otimes n} \otimes \pi^{\otimes M}$. 
	In particular let $Q \subseteq W$ be an event, we denote with $Q|\bomega$ the subset of $Z$ associated to the event $Q$ given $\bomega$, that is $Q | \bomega := \{ \bzeta ~|~ (\bzeta,\bomega) \in Q\}$. By denoting $\rho^{\otimes n}$ with $\Pzeta$ and $\pi^{\otimes M}$ by $\Pomega$, we recall that 
	\eqal{\label{eq:prob-cond}
		\PP(Q) = \int_{\Omega^M} \Pzeta(Q|\bomega) d\Pomega(\bomega).
	}
	Moreover, we recall the following basic facts about ${\cal F}_\infty(\la)$ and ${\cal N}(\la)$. We can characterize the upper and lower bounds for ${\cal F}_\infty(\la)$, in particular we have that ${\cal F}_\infty(\la) \leq \kappa^2 \la^{-1}$ when $\psi$ is uniformly bounded by $\kappa$ (see Asm.~\ref{ass:kernel-bounded}), moreover ${\cal F}_\infty(\la) \geq {\cal N}(\la)$ indeed  ${\cal N}(\la)$ is characterized by ${\cal N}(\la) = \mathbb{E}_\omega \nor{(L+\la I)^{-1/2} \psi_\omega}^2_\rhox$ (see Eq.~\ref{eq:Nla-characterization}), so
	$$ {\cal N}(\la) = \mathbb{E}_\omega \nor{(L+\la I)^{-1/2} \psi_\omega}^2_\rhox \leq \sup_{\omega \in \Omega} \nor{(L+\la I)^{-1/2} \psi_\omega}^2_\rhox = {\cal F}_\infty(\la).$$

	\subsubsection{Estimates for ${\cal S}(\la, M, n)$}
	The next lemma bounds the first term of ${\cal S}(\la, M, n)$ and use a similar technique to the one in \cite{caponnetto2007optimal}, while Lemma~\ref{lm:conc-S} bounds the whole ${\cal S}(\la, M, n)$.
	First we need to introduce ${\cal N}_M(\la)$ that is the effective dimension induced by the kernel $K_M$. For any $\la > 0$ define ${\cal N}_M(\la)$ as follows,
	$${\cal N}_M(\la): = \tr((\tL + \la I)^{-1}\tL).$$
	In Prop.~\ref{prop:emp-eff-dim} in the appendix, we bound ${\cal N}_M(\la)$ in terms of the effective dimension ${\cal N}(\la)$ that is the one associated to the kernel $K$. Prop.~\ref{prop:emp-eff-dim} refines the result of Prop.~1 of \cite{rudi2015less}, with simpler proof and slightly improved constants. 
	
	\blm\label{lm:sampling-error}
	Let $\delta \in (0,1]$, $n, M \in \N$ and $\la > 0$. Given $\omega_1, \dots, \omega_M \in \Omega$, under Assumptions~\ref{ass:kernel-bounded}, \ref{ass:noise}, the following holds with probability at least $1- \delta$
	\eqals{
		\nor{\tCl^{-1/2}(\tSn^*\yn  - \tS^* \frho)} \leq 2 \left(\frac{B \tkappa}{\sqrt{\la}n} + \sqrt{\frac{\sigma^2{\cal N}_M(\la)}{n}} \right)\log \frac{2}{\delta}
		.}
	\elm
	\bpr
	In this proof we bound the quantity under study, by using the Bernstein inequality for sum of zero-mean random vectors (see Prop.~\ref{prop:tail-vectors} in the appendix).  
	Since $\tSn^*\yn = n^{-1} \sum_{i=1}^n \phi_M(x_i) y_i$ (see Def.~\ref{def:all-ops-tilde}) we have
	$$\tCl^{-1/2}(\tSn^*\yn  - \tS^*\frho) = \frac{1}{n}\sum_{i=1}^n \zeta_i,$$
	where $\zeta_1,\dots,\zeta_n$ are defined as $\zeta_i = z_i - \mu$ with $z_i := \tCl^{-1/2} \phi_M(x_i) y_i$, and $\mu \in \R^M$ defined as $\mu := \tCl^{-1/2}\tS^*\frho$, for $1 \leq i \leq n$.
	Note that $\zeta_1,\dots, \zeta_n$ are independent and identically distributed random vectors given $\omega_1,\dots, \omega_M$, since $(x_1,y_1),\dots, (x_n, y_n)$ are assumed i.i.d. with respect to $\rho$. Moreover note that, by definition of  $\frho$,
	\eqals{
		\int \phi_M(x)y d\rho(x,y) = \int y d\rho(y|x)d\rhox(x) = \int \phi_M(x) \frho(x) d\rhox(x) =  \tS^*\frho.
	}
	So, by linearity of the expectation the $\zeta_i$,
	$$ \mathbb{E} z_i = \tCl^{-1/2} \int \phi_M(x)y d\rho(x,y) = \tCl^{-1/2} \tS^*\frho = \mu,$$
	which implies that $\zeta_i = z_i - \mu$ is a zero-mean random variable for $1 \leq i \leq n$. 
	Let $z$ be another random variable independent and identically distributed as the $z_i$'s. To apply the Bernstein inequality for random vectors, we need to bound their moments. First of all note that for any $p \geq 1$
	\eqals{
		\mathbb{E}\nor{\zeta_i}^p & ~=~ \mathbb{E} \nor{z_i - \mu}^p ~=~ \mathbb{E} \nor{z_i - \mathbb{E} z}^p \\
		& \leq ~ \mathbb{E}_{z_i}\mathbb{E}_{z} \nor{z_i - z}^p ~\leq~ 2^{p-1} \mathbb{E}_{z_i}\mathbb{E}_{z} (\nor{z_i}^p + \nor{z}^p) ~=~ 2^p \mathbb{E}_{z} \nor{z}^p.
	}
	In particular, by applying Asm.~\ref{ass:kernel-bounded} and Asm.~\ref{ass:noise}, we have 
	\eqals{
		\mathbb{E}_{z} \nor{z}^p & = \int_{\X \times \R} \nor{\tCl^{-1/2} \phi_M(x) y}^p d\rho(x,y) = \int_{\X} \nor{\tCl^{-1/2} \phi_M(x)}^p ~\int |y|^p d\rho(y|x) ~d\rhox(x) \\
		& \leq  \frac{1}{2} p! \sigma^2 B^{p-2} \int_{\X} \nor{\tCl^{-1/2} \phi_M(x)}^p d\rhox(x) \\
		& \leq \frac{1}{2} p! \sigma^2 B^{p-2} ~ \left(\sup_{x \in \X} \nor{\tCl^{-1/2} \phi_M(x)}^{p-2}\right) \int_{\X}\nor{\tCl^{-1/2} \phi_M(x)}^2 d\rhox \\
		&= \frac{1}{2} p! \sqrt{J(\la) \sigma^2}^2 \left(\frac{B \kappa}{\sqrt{\la}}\right)^{p-2}.
	}
	where $J(\la) = \int_X \nor{\tCl^{-1/2} \phi_M(x)}^2 d\rhox(x)$, while 
	$\nor{\tCl^{-1/2} \phi_M(x)} \leq \tkappa/\sqrt{\la}$ a.s. is given by
	\eqals{
		\|\tCl^{-1/2} & \phi_M(x)\|^2 \leq \frac{1}{\la}\sup_{x \in X} \nor{\phi_M(x)}^2 
		= \frac{1}{\la M} \sup_{x \in X} \sum_{i=1}^M |\psi_{\omega_i}(x)|^2 \\
		& \leq \frac{1}{\la M} \sum_{i=1}^M \sup_{x \in X}  |\psi_{\omega_i}(x)|^2 \leq \frac{1}{\la M} \sum_{i=1}^M \sup_{\omega \in \Omega, x \in X}  |\psi_{\omega}(x)|^2 \leq \left(\frac{\tkappa}{\sqrt{\la}}\right)^2,
	} 
	where the last step is due to Asm.~\ref{ass:kernel-bounded}. Finally, to concentrate the sum of random vectors, we apply Prop.~\ref{prop:tail-vectors}. To conclude the proof we need to prove that $J(\la) = {\cal N}_M(\la)$. Note that, by Rem.~\ref{rem:S-C-L}, we have that $\tL = \tS\tS^*$ and $\tC = \tS^*\tS$, so
	$${\cal N}_M(\la) = \tr \tL \tLl^{-1}  = \tr \tS^* \tLl^{-1}\tS =  \tr \tC \tCl^{-1},$$
	since $\tL = \tS\tS^*$ and $\tS^* \tLl^{-1}\tS = \tC \tCl^{-1}$. By the the ciclicity of the trace and the definition of $\tC$ in Def.~\ref{def:all-ops-tilde}, we have
	$$\tr \tC \tCl^{-1} = \int_X \tr (\phi_M(x) \phi_M(x)^\top \tCl^{-1})  d\rhox(x) = \int_X \nor{\tCl^{-1/2} \phi_M(x)}^2 d\rhox(x) = J(\la).$$
	\epr 
	\blm[Bounding ${\cal S}(\la,M,n)$]\label{lm:conc-S}
	Let $\delta \in (0,1/3]$, $n \in \N$ and let ${\cal S}(\la, M, n)$ be as in Thm.~\ref{thm:geo-dec}, point 1. Let $\bar{B} = B + 2 R\tkappa^{2r}$, $\bar{\sigma} = \sigma + \sqrt{R} \tkappa^r$. Under Asm.~\ref{ass:kernel-bounded},~\ref{ass:noise}
	the following holds with probability at least $1 - 3\delta$
	\eqal{\label{eq:event-S}
		{\cal S}(\la,m,n) \leq 4 \left(\frac{\bar{B} \tkappa}{\sqrt{\la}n} + \sqrt{\frac{\bar{\sigma}^2{\cal N}(\la)}{n}} \right)\log \frac{2}{\delta},
	}
	when $0 < \la < \nor{L}$ and $M \geq (4 + 18{\cal F}_\infty(\la))\log \frac{12\tkappa^2}{\la \delta}$.
	\elm
	\bpr
	Let $0 < \la < \nor{L}$ and $M \geq (4 + 18{\cal F}_\infty(\la))\log \frac{12\tkappa^2}{\la \delta}$ (the assumption on $\la, M$ are necessary for the application of Prop.~\ref{prop:emp-eff-dim}). Let $Q \subseteq W$ be the event satisfying Eq.~\ref{eq:event-S}. The goal is to prove that the probability associated to the event $Q$ is $\PP(Q) \geq 1 - 3\delta$.
	Since the quantity ${\cal S}(\la,M,n)$ is defined in Thm.~\ref{thm:geo-dec} as
	$${\cal S}(\la,M,n) = \nor{\tCl^{-1/2}(\tSn^*y - \tS^*\frho)} + R\kappa^{2r-1} \nor{\tCl^{-1/2}(\tC - \tCn)},$$
	we are first going to bound in probability the single terms on the rhs, given $\omega_1, \dots, \omega_M$, then we take the union event, and we use this to prove that $\PP(Q) \geq 1 - 3\delta$.  
	
	First of all we need to define four other events.
	Define the event $E^1_{\bomega} \subseteq Z$, as the event satisfying 
	\eqal{\label{eq:S-event-E1}
		\nor{\tCl^{-1/2}(\tSn^*\yn  - \tS^*\frho)} \leq 2 \left(\frac{B \tkappa}{\sqrt{\la}n} + \sqrt{\frac{\sigma^2{\cal N}_M(\la)}{n}} \right)\log \frac{2}{\delta}.
	}
	By Lemma~\ref{lm:sampling-error} we know that $E^1_{\bomega}$ holds with probability  $\Pzeta(E^1_{\bomega}) \geq 1 - \delta$ almost everywhere for $\bomega$.
	Then, define the event $E^2_{\bomega} \subset Z$, as the event satisfying
	\eqal{\label{eq:S-event-E2}
		\nor{\tCl^{-1/2}(\tC - \tCn)} \leq \frac{4\tkappa^2\log\frac{2}{\delta}}{\sqrt{\la}n} + \sqrt{\frac{4\tkappa^2 {\cal N}_M(\la)\log\frac{2}{\delta}}{n}}.
	}
	By applying Prop.~\ref{prop:Clahalf-C-control}, with $v_i = z_i = \phi_M(x_i)$ for $1 \leq i \leq n$, we have that $E^2_{\bomega}$ holds with probability $\Pzeta(E^2_{\bomega}) \geq 1- \delta$ almost everywhere for $\bomega$.
	Define $E_{\bomega} \subseteq Z$ as the event satisfying
	\eqal{\label{eq:S-event-E3}
		{\cal S}(\la, M, n) \leq 2\left(\frac{B\kappa + 2 R \kappa^{2r+1}}{\sqrt{\la}~ n} ~+~ (\sigma + R\kappa^{2r})\sqrt{\frac{{\cal N}_M(\la)}{n}} \right)\log\frac{2}{\delta}.
	} 
	Denote with $t$ the right hand side of the equation above and with  $s_1, t_1, s_2, t_2$  respectively the lhs and the rhs of Eq.~\eqref{eq:S-event-E1},~\eqref{eq:S-event-E2}. 
	We have that $s_1 \leq t_1$ and $s_2 \leq t_2$ implies ${\cal S}(\la, M, n) \leq t$, indeed ${\cal S}(\la, M, n) = s_1 + R\kappa^{2r-1}s_2$ and $ t_1 + R\kappa^{2r-1}t_2 \leq t$, since $\log(2/\delta) > 1$. In set terms $(E^1_{\bomega} \cap E^2_{\bomega}) ~ \subseteq ~ E_{\bomega}$, that implies $\Pzeta(E_{\bomega}) \geq \Pzeta(E^1_{\bomega} \cap E^2_{\bomega})$, in particular
	\eqals{
		\Pzeta(E_{\bomega}) &\geq \Pzeta(E^1_{\bomega} \cap E^2_{\bomega}) \geq \Pzeta(E^1_{\bomega}) + \Pzeta(E^2_{\bomega}) - 1 \geq 1 - 2\delta,
	}
	where we used the fact that for any probability measure $P$ and two events $A, B$, we have $P(A \cap B) = P(A) + P(B) - P(A \cup B) \geq P(A) + P(B) - 1$.
	The last event that we need to define is $A \subseteq \Omega^M$ satisfying ${\cal N}_M(\la) \leq 1.5 {\cal N}(\la)$. By Prop.~\ref{prop:emp-eff-dim}, we know that $A$ holds with probability $\Pomega(A) \geq 1 - \delta$. 
	
	Now we characterize the probability of $Q|\bomega$  when $\bomega \in A$. Denote with $t_E$ the rhs of Eq.\ref{eq:S-event-E3} defining $E_\bomega$ and $t_Q$ the rhs of Eq.~\ref{eq:event-S} defining $Q$ (and so $Q|\bomega$).
	When $\bomega \in A$, we have that ${\cal N}_M(\la) \leq 1.5 {\cal N}(\la)$ and so $t_E \leq t_Q$. Then, when $\bomega \in A$, we have that ${\cal S}(\la, M, n) \leq t_E$ implies ${\cal S}(\la, M, n) \leq t_Q$, that is $E_\bomega ~\subseteq~ Q|\bomega$, implying that $\Pzeta(E_\bomega) \leq \Pzeta(Q|\bomega)$. By using the expansion of $\PP(Q)$ in Eq.~\ref{eq:prob-cond}  we have
	\eqals{
		\PP(Q) &= \int_{A} \Pzeta(Q|\bomega) d\Pomega(\bomega) + \int_{\Omega^M \setminus A} \Pzeta(Q|\bomega) d\Pomega(\bomega) \geq  \int_{A} \Pzeta(Q|\bomega) d\Pomega(\bomega) \\
		& \geq \int_{A} \Pzeta(E_\omega) d\Pomega(\bomega) \geq (1-2\delta) \int_{A} d\Pomega(\bomega) \geq (1-2\delta)(1-\delta) \geq 1 - 3\delta.
	}
	\epr
	\subsubsection{Estimates for ${\cal C}(\la,M)$}
	
	\blm[Bounding ${\cal C}(\la, m)$]\label{lm:conc-C}
	Let ${\cal C}(\la, M)$ as in Thm.~\ref{thm:geo-dec}, point 2. Let $\delta \in (0, 1/2]$ and $\la > 0$. Under Asm.~\ref{ass:kernel-bounded}, following holds with probability at least $1 - 2\delta$
	\eqals{
		{\cal C}(\la, m) \leq 4R\tkappa^{2r-1}\left(\frac{\sqrt{\la{\cal F}_\infty(\la)} \log\frac{2}{\delta}}{M^{r}} + \sqrt{\frac{\la {\cal N}(\la)^{2r-1}{\cal F}_\infty(\la)^{2-2r}\log\frac{2}{\delta}}{M}}\right)t^{1-r},
	}
	when $M \geq (4 + 18 {\cal F}_\infty(\la))\log\frac{8 \tkappa^2}{\la \delta}$ and $t := \log \frac{11\tkappa^2}{\la}$.
	\elm
	\bpr
	We now study ${\cal C}(\la, M)$ that is
	\eqals{
		{\cal C}(\la, M) = R\la^{1/2}\nor{\Ll^{-1/2}(L - \tL)}^{2r-1}\nor{\Ll^{-1/2}(L - \tL)\Ll^{-1/2}}^{2-2r}.
	}
	We are going to bound the two terms in probability, via Prop.~\ref{prop:Clahalf-C-control} and Prop.~\ref{prop:concentration-ClaC}. First of all, we recall that ${\cal F}_\infty(\la) := \sup_{\omega \in \Omega} \nor{\Ll^{-1/2}\psi_\omega}^2_\rhox$ and that ${\cal N}(\la) = \mathbb{E} \nor{\Ll^{-1/2}\psi_\omega}^2_\rhox$, indeed by Lemma.~\ref{lm:L-characterization}, characterizing $L$ in terms of $\psi_\omega$, the linearity and the ciclicity of the trace, we have,
	\eqal{\label{eq:Nla-characterization}
		{\cal N}(\la) &:= \tr(\Ll^{-1} L) = \tr\left(\Ll^{-1} \int \psi_\omega \otimes \psi_\omega d\pi(\omega)\right) = \int \tr(\Ll^{-1} (\psi_\omega \otimes \psi_\omega)) d\pi(\omega)\\
		& = \int \scal{\psi_\omega}{\Ll^{-1}\psi_\omega}_\rhox d\pi(\omega) = \int \nor{\Ll^{-1/2}\psi_\omega}^2_\rhox d\pi(\omega), 
	}  
	where the last steps are due to the identity $\tr(A(v\otimes v)) = \scal{v}{Av} = \nor{A^{1/2}v}^2$ valid for any vector $v$ and any bounded self-adjoint positive operator $A$ on a Hilbert space.
	
	Define $A \subseteq \Omega^M$ the event satisfying
	$$ \nor{\Ll^{-1/2}(L-\tL)} \leq \frac{4\sqrt{{\cal F}_\infty(\la)\tkappa^2}\log\frac{2}{\delta}}{M} + \sqrt{\frac{4\tkappa^2{\cal N}(\la)\log\frac{2}{\delta}}{M}}.$$
	By the fact that $\nor{\cdot} \leq \nor{\cdot}_{HS}$ and by Prop.~\ref{prop:Clahalf-C-control}, with $v_i = \Ll^{-1/2}\psi_{\omega_i}$ and $z_i = \psi_{\omega_i}$ for $i \in \{1,\dots,M\}$ and $Q = T = L,~ T_n = \tL$,  we know that the event $A$ has probability $\Pomega(A) \geq 1-\delta$. 
	
	Define $B \subseteq \Omega^M$ the event satisfying 
	\eqal{\label{eq:C-event-B}
		\nor{\Ll^{-1/2}(L-\tL)\Ll^{-1/2}} \leq \frac{2\eta(1 + {\cal F}_\infty(\la))}{3 M} + \sqrt{\frac{2\eta {\cal F}_\infty(\la)}{M}},
	}
	with $\eta := \log \frac{8 \tkappa^2}{\lambda\delta}$.
	By Prop.~\ref{prop:concentration-ClaC}, with $Q = L$ and $v_i = \psi_{\omega_i}$ for $i \in \{1,\dots,M\}$, we know that $B$ has probability $\Pomega(B) \geq 1-\delta$.
	
	Now set $E = A \cap B$. When $E$ holds and under the assumption that $M \geq (4 + 18 {\cal F}_\infty(\la))\log\frac{8 \tkappa^2}{\la \delta}$, we have that the right hand side of Eq.~\eqref{eq:C-event-B} is smaller than $\sqrt{\frac{4\eta {\cal F}_\infty(\la)}{M}},$ and $\eta = \log \frac{2}{\delta} + \log \frac{4\kappa^2}{\la}$, so 
	\eqal{\label{eq:proof-bound-ClaM}
		{\cal C}&(\la, M) \leq R\la^{1/2}\left(\frac{4\sqrt{{\cal F}_\infty(\la)\tkappa^2}\log\frac{2}{\delta}}{M} + \sqrt{\frac{4\tkappa^2{\cal N}(\la)\log\frac{2}{\delta}}{M}}\right)^{2r-1}\sqrt{\frac{4\eta {\cal F}_\infty(\la)}{M}}^{2-2r} \\
		& \leq 4R\kappa^{2r-1}\left(\frac{\sqrt{\la{\cal F}_\infty(\la)} (\log\frac{2}{\delta})^r}{M^{r}} + \sqrt{\frac{\la {\cal N}(\la)^{2r-1}{\cal F}_\infty(\la)^{2-2r}\log\frac{2}{\delta}}{M}}\right)\left(1 + \frac{\log \frac{4\kappa^2}{\lambda}}{\log \frac{2}{\delta}}\right)^{1-r}.
	} 
	The event $E$ holds with probability $ \Pomega(E) \geq \Pomega(A) + \Pomega(B) - 1 \geq 1 - 2\delta$.
	We recall that since $E$ does not depend on $\bzeta$, the probability of $E$ in $W$ is given by the canonical extension $E' = Z \times E$ whose probability is again $\PP(E') = \Pomega(E) \geq 1 - 2\delta$. Finally,  we further upper bound in Eq.\eqref{eq:proof-bound-ClaM} the term $(\log\frac{2}{\delta})^r$ with $\log\frac{2}{\delta}$ since $\log\frac{2}{\delta} > 1, r \in [1/2,1]$, and $1 + (\log \frac{4\kappa^2}{\lambda})/(\log \frac{2}{\delta})$ with $\log \frac{11\kappa^2}{\lambda}$, for the same reasons and the fact that $1 + \log 4 = \log 4e \leq \log 11$.
	\epr

	\subsubsection{Estimates for $\beta$}
	
	\blm[Bounding the norm of $\tC$]\label{lm:norm-tC}
	Let $\delta \in (0,1]$. Under Asm.~\ref{ass:kernel-bounded}, the following holds with probability at least $1 - \delta$,
	$$\nor{\tC} \geq \frac{3}{4}\nor{L},$$
	when $M \geq 32\left(\frac{\kappa^2}{\nor{L}} + \tkappa^2\right)\log\frac{2}{\delta}$.
	\elm
	\bpr
	Define the event $A \in \Omega^M$ as the one satisfying 
	\eqals{
		\nor{L - \tL}_{HS} \leq  \frac{4\kappa^2}{M}\log\frac{2}{\delta} - \sqrt{\frac{2\kappa^2}{M}\log\frac{2}{\delta}}.
	}
	
	Note that when $\bomega \in A$ and $M \geq 32\left(\frac{\kappa^2}{\nor{L}} + \tkappa^2\right)\log\frac{2}{\delta}$, we have $\nor{L - \tL}_{HS} \leq \frac{1}{4} \nor{L}$ and, by using the characterization of $\tC, \tL$ in Rem.~\ref{rem:S-C-L} and the fact that $\nor{\cdot}_{HS} \geq \nor{\cdot}$, we have 
	\eqals{
		\nor{\tC} &= \nor{\tS^*\tS} = \nor{\tS\tS^*} = \nor{\tL}  \geq |\nor{L} - \nor{L - \tL}| \\
		& \geq \nor{L} - \nor{L - \tL} \geq \nor{L} - \nor{L - \tL}_{HS}  \geq \nor{L} - \frac{1}{4}\nor{L} \geq \frac{3}{4} \nor{L}.
	}
	Now we find a lower bound for the probability of $A$. Let $\zeta_i = L - \psi_{\omega_i} \otimes \psi_{\omega_i}$ be a random operator with $\omega_i$ independently and identically distributed w.r.t $\pi$ and $i \in \{1,\dots,M\}$. We have that $L - \tL = \frac{1}{M}\sum_{i=1}^M \zeta_i$ and ${\mathbb E} \zeta_i = 0$, by the characterization of $L$ in Lemma~\ref{lm:L-characterization}. Denote with $\cal L$ the Hilbert space of Hilbert-Schmidt operators on $\Ltwo$. Now note that, since it is trace class, $\zeta_i$ is a random vector belonging to ${\cal L}$, so we can apply Prop.~\ref{prop:tail-vectors}, with $T = \sup_{\omega \in \Omega} \nor{L - \psi_{\omega} \otimes \psi_{\omega}}_{HS} \leq 2\tkappa^2$ and $S = {\mathbb E} \nor{\zeta_1}_{HS}^2 \leq \kappa^2$, obtaining that $\Pomega(A) \geq 1-\delta$.
	\epr
	
	\blm[Bounding $\beta$]\label{lm:conc-beta} 
	Let $\delta \in (0, 1/3]$, and $\beta$ be as in Thm.~\ref{thm:geo-dec}, point 3. Under Asm.~\ref{ass:kernel-bounded}, the following holds with probability at least $1 - 3\delta$,
	$$ \beta < 2,$$ 
	when $0 < \la \leq \frac{3}{4}\nor{L}$, and 
	$$n \geq 18\left(2 + \frac{\kappa}{\la} \right)\log \frac{4\tkappa^2}{\la \delta}, \qquad M \geq 18\left(2 + {\cal F}_\infty(\la)\right)\log \frac{4\tkappa^2}{\la \delta} ~\vee 32\left(\frac{\kappa^2}{\nor{L}} + \tkappa^2\right)\log\frac{2}{\delta}.$$
	\elm
	\bpr 
	Let $\beta, \beta_1, \beta_2$ be defined as in Thm.~\ref{thm:geo-dec}, point 3.
	To bound $\beta$ in probability, we first bound $\beta_1$ and $\beta_2$ in probability, and then, under the intersection of the events, we control $\beta$.
	First of all, denote with $(a)$ the condition on $\la$, with $(b)$ the condition on $n$ and with $(c.1)$ the condition $M \geq 32\left(\frac{\kappa^2}{\nor{L}} + \tkappa^2\right)\log\frac{2}{\delta}$, while with $(c.2)$ the condition $M \geq 18\left(2 + {\cal F}_\infty(\la)\right)\log \frac{4\tkappa^2}{\la \delta}$.
	Define the event $E \subseteq W$ as the one satisfying $\beta_1 \leq \frac{1}{3}$. To bound the probability of $E$ we need an auxiliary event $A \in \Omega^M$ that is the one satisfying $\frac{3}{4} \nor{L} \leq \nor{\tC}$.
	The specific choice of $A$ will be made clear later. We have
	$$ \PP(E) = \int_A \Pzeta(E|\bomega) d\Pomega(\bomega) + \int_{\Omega^M \setminus A} \Pzeta(E|\bomega) d\Pomega(\bomega) \geq \int_A \Pzeta(E|\bomega) d\Pomega(\bomega).$$
	By  Prop.~\ref{prop:concentration-ClaC} and Rem.~\ref{rem:ClC} point 3, we know that  $\Pzeta(E|\bomega) \geq 1-\delta$, for any $\bomega$, when $\la \leq \nor{\tC}$ and condition $(b)$ hold. Note that, when $\bomega$ is in $A$ then $\frac{3}{4} \nor{L} \leq \nor{\tC}$ and so the condition $\la \leq \nor{\tC}$ is always satisfied by assuming $(a)$. Then under $(a), (b)$, we have $\Pzeta(E|\bomega) \geq 1-\delta$ when $\bomega \in A$, and so under the same conditions
	$$ \PP(E) \geq \int_A \Pzeta(E|\bomega) d\Pomega(\bomega) \geq (1-\delta) \int_A  d\Pomega(\bomega) = (1-\delta)\Pomega(A).$$
	By Lemma~\ref{lm:norm-tC}, $\Pomega(A) \geq 1-\delta$, when $(c.1)$ holds, so
	$\PP(E) \geq (1-\delta)\Pomega(A) \geq (1-\delta)^2 \geq 1 - 2\delta$ when $(a), (b), (c.1)$ hold.
	
	Define $D_0 \subseteq \Omega^M$ as the event satisfying $\beta_2 \leq \frac{1}{3}$. By  Prop.~\ref{prop:concentration-ClaC} and Rem.~\ref{rem:ClC} point 3, we know that  $\Pomega(D_0) \geq 1-\delta$,  when the condition $(c.2)$ and $(a)$ hold. So the event $D:= Z \times D_0$ has probability $\PP(D) = \Pomega(D_0) = 1 - \delta$, when $(c.2)$ holds.
	Finally, note that under the conditions $(a), (b), (c.1), (c.2)$, when the event $D \cap E$ hold, we have $\beta \leq (2/3)^{-3/2} < 2$. The probability of $D \cap E$ under the conditions $(a), (b), (c.1), (c.2)$ has probability
	$$\PP(D \cap E)  = \PP(D) + \PP(E) - \PP(D \cup E) \geq \PP(D) + \PP(E) - 1 \geq 1 - 3\delta.$$
	\epr
	%
	%
	%
	%
	%
	
	
	
	\subsection{Proof of the Main Result}\label{sect:proof-main-res}
	Here we prove Thm.~\ref{thm:rf-simple-universal},~\ref{thm:rf-fast-universal},~\ref{thm:rf-fast-compatibility} that are the main results of the paper.
	In particular the following Thm.~\ref{thm:main-bound} is a general version of the three theorem above, without the need of Assumptions~\ref{ass:intrinsic},~\ref{ass:compatibility} and valid for a wide range of $\la, M$.
	In Thm.~\ref{thm:main-bound-M-wrt-la}, we specialize the result of Thm.~\ref{thm:main-bound}, selecting $M$ in terms of $\la$ such that the upper bound of the excess risk depends only on $\la$ and is
	proportional to the same upper bound for kernel ridge regression that leads to optimal generalization bounds. Note that Thm.~\ref{thm:main-bound-M-wrt-la} is again independent of Asm.~\ref{ass:intrinsic},~\ref{ass:compatibility}.
	Finally, Thm.~\ref{thm:optimal-rates-RF-RKLS-with-constants} is obtained by Thm.~\ref{thm:main-bound-M-wrt-la}, by adding Asm.~\ref{ass:intrinsic},~\ref{ass:compatibility} and has all the constant explicit. Then Thm.~\ref{thm:rf-simple-universal} is a specification of Thm.~\ref{thm:optimal-rates-RF-RKLS-with-constants}, for the simple scenario where it is only required the existence of $\fh$ (that is Asm.~\ref{ass:intrinsic} satisfied with $\gamma = 1$, Asm.~\ref{ass:source} satisfied with $r=1/2$ and Asm.~\ref{ass:compatibility} with $\alpha=1$, see discussion after the introduction of the assumptions). Thm.~\ref{thm:rf-fast-universal} is a specification of Thm.~\ref{thm:optimal-rates-RF-RKLS-with-constants} for the fast rates (Asm.~\ref{ass:compatibility} is satisfied with $\alpha = 1$), while Thm.~\ref{thm:rf-fast-compatibility} is a simplified version of Thm.~\ref{thm:optimal-rates-RF-RKLS-with-constants} where the constants have been hidden. 
	
	\bt[Generalization Bound for RF-KRLS]
	\label{thm:main-bound}
	Let $\delta \in (0,1]$. Let  $\widehat{f}_{\la,M}$ be as in Eq.~\eqref{eq:algo-rf}.
	Under Asm.~\ref{ass:kernel-bounded},~\ref{ass:source},~\ref{ass:noise} when $0 < \la \leq \frac{3}{4}\nor{L}$ and
	\eqals{
		& \;\; n \geq 18\left(2 + \frac{\tkappa^2}{\la}\right)\log \frac{36\tkappa^2}{\la \delta}, \\
		& \;\; M \geq 18\left(q_0 + {\cal F}_\infty(\la)\right)\log \frac{108\tkappa^2}{\la \delta},
	}
	with $q_0 = 2(2 + \frac{\kappa}{\nor{L}} + \tkappa^2)$, then the following holds with probability at least $1-\delta$,
	\eqal{\label{eq:main-bound}
		\sqrt{\EE(\widehat{f}_{\la, M})-\inf_{f \in \hh}\EE(f)} \leq 2
		\left(\frac{4\bar{B} \tkappa}{\sqrt{\la}n} + \sqrt{\frac{16\bar\sigma^2{\cal N}(\la)}{n}} ~+~ {\frak C}(\la, M) ~+~ R \la^{r}\right)\log \frac{18}{\delta}.
	}
	where $\bar{B} = B + 2 R\tkappa$, $\bar{\sigma} = \sigma + 2\sqrt{R} \tkappa$,
	\eqal{\label{eq:copy-conc-C}
		{\frak C}(\la, M) := R\kappa^{2r-1}\left(\frac{\sqrt{\la{\cal F}_\infty(\la)}}{M^{r}} + \sqrt{\frac{\la {\cal N}(\la)^{2r-1}{\cal F}_\infty(\la)^{2-2r}}{M}}\right)t^{1-r}.
	}
	and $t := \log \frac{11\tkappa^2}{\la}$.
	\et
	\bpr
	Under Asm.~\ref{ass:kernel-bounded}, the existence of $\fh$ in Asm.~\ref{ass:noise} and Asm.~\ref{ass:source}, plus Rem.~\ref{rem:excess-risk-to-Ltwo}, we have the following analytical decomposition of the excess risk, by Thm,~\ref{thm:geo-dec}
	\eqal{\label{eq:proof-main-bound-excess-risk-expansion}
		|{\cal E}(\widehat{f}_{\la, M}) - \inf_{f\in \hh}{\cal E}(f)|^{1/2} \leq \beta\left({\cal S}(\la, M, n) \;\;+\;\; {\cal C}(\la, M) \;\;+\;\;  R\la^{r}\right),
	}
	where the quantities $\beta$, ${\cal C}(\la, M)$ and ${\cal S}(\la, M, n)$ are defined in the statement of Thm.~\ref{thm:geo-dec}.
	Under the same assumptions, Lemma~\ref{lm:conc-S},~\ref{lm:conc-C} and~\ref{lm:conc-beta} are devoted to bound in probability the three quantities, with the help of the concentration inequalities recalled in Section~\ref{sect:conc-ineq}, plus some auxiliary results in Section~\ref{sect:aux-results}, of the appendixes.
	
	Let $\tau := \delta/9$.
	Define the event $D \subseteq W$ as the one satisfying $\beta < 2$ (see Subsection~\ref{sect:prob-estimates} for the definition of the sample space $W$ for learning with random features, and the associated probability measure $\PP$).
	By Lemma~\ref{lm:conc-beta} we know that the event $D$ has probability $\PP(D) \geq 1 - 3\tau$, when the following conditions hold
	\eqals{(d_1) &~~0 \leq \la \leq \frac{3}{4}\nor{L},\qquad (d_2) ~~ n \geq 18\left(2 + \kappa/\lambda \right)\log \frac{4\tkappa^2}{\la \tau}, \\
		(d_3) &~~ M \geq 18\left(2 + {\cal F}_\infty(\la)\right)\log \frac{4\tkappa^2}{\la \tau} ~\vee~ 32\left(\frac{\kappa^2}{\nor{L}} + \tkappa^2\right)\log\frac{2}{\tau}.
	}
	Define the event $E \subseteq W$ as the one satisfying
	\eqal{\label{eq:copy-conc-S}
		{\cal S}(\la,M,n) \leq 4 \left(\frac{\bar{B} \tkappa}{\sqrt{\la}n} + \sqrt{\frac{\bar{\sigma}^2{\cal N}(\la)}{n}} \right)\log \frac{2}{\tau}.
	}
	By Lemma~\ref{lm:conc-S} we know that the probability of $E$ is $\PP(E) \geq 1 - 4\tau$, when the following conditions hold
	$$ (e_1)~~0 < \la < \nor{L}, \qquad (e_2)~~M \geq (4 + 18{\cal F}_\infty(\la))\log \frac{12\tkappa^2}{\la \tau}.$$
	Define the event $G \subseteq W$ as the one satisfying 
	\eqals{
		{\cal C}(\la, M) \leq {\frak C}(\la, M),
	}
	By Lemma~\ref{lm:conc-C} we know that $G$ holds with probability $\PP(G) \geq 1 - 2\tau$, when the $(d_1)$ and the following condition holds
	$$(g_1)~~M \geq (4 + 18 {\cal F}_\infty(\la))\log\frac{8 \tkappa^2}{\la \tau}.$$
	Finally Eq.~\eqref{eq:main-bound} is obtained from Eq.~\eqref{eq:proof-main-bound-excess-risk-expansion}, by bounding $\beta$ with $2$, ${\cal S}(\la, M, n)$ with Eq.~\eqref{eq:copy-conc-S} and ${\cal C}(\la, M)$ by ${\frak C}(\la, M)$. So by definition, Eq.~\eqref{eq:main-bound} holds under the event $D \cap E \cap G$ and the conditions $(d_1), (e_1)$ on $\la$, $(d_2)$ on $n$ and $(d_3), (e_2), (g_1)$ on $M$.
	The event $D\cap E \cap G$ has probability
	\eqals{
		\PP(D\cap E \cap G) ~~~&=~~~ \PP(W \setminus ( (W \setminus D)\cup(W \setminus E)\cup(W \setminus G))) \\
		& \geq ~~~ 1 ~~-~~ [~(1 - \PP(D)) ~+~ (1 - \PP(E)) ~+~ (1 - \PP(G))~] \\
		& = ~~~ \PP(D) + \PP(E)  + \PP(G) -  2 ~~~ \geq ~~~ 1 - 9 \tau.
	}
	Finally note that the conditions on $\la, n, M$ in the statement of this theorem imply, respectively, conditions $(d_1), (e_1)$ on $\la$, $(d_2)$ on $n$, and $(d_3), (e_2), (g_1)$ on $M$.
	\epr

	\bt[Generalization Bound for RF-KRLS]
	\label{thm:main-bound-M-wrt-la}
	Let $\delta \in (0,1]$. Let  $\widehat{f}_{\la,M}$ be as in Eq.~\eqref{eq:algo-rf}.
	Under Asm.~\ref{ass:kernel-bounded},~\ref{ass:source},~\ref{ass:noise}, when $0 < \la \leq \frac{3}{4}\nor{L}$ and
	\eqals{
		& \;\; n \geq 18\left(2 + \frac{\tkappa^2}{\la}\right)\log \frac{36\tkappa^2}{\la \delta}, \\
		& \;\; M \geq 4\tkappa^2\left(\frac{{\cal N}(\la)}{\la}\right)^{2r - 1}\left({\cal F}_\infty(\la) \log \frac{11 \tkappa^2}{\la}\right)^{2 - 2r}  \;\; \vee \;\;18\left(q_0 + {\cal F}_\infty(\la)\right)\log \frac{108\tkappa^2}{\la \delta},
	}
	with $q_0 = 2(2 + \frac{\kappa}{\nor{L}} + \tkappa^2)$, then the following holds with probability at least $1-\delta$,
	\eqal{\label{eq:main-bound-M-wrt-la}
		\sqrt{\EE(\widehat{f}_{\la, M})-\inf_{f \in \hh}\EE(f)} \leq 8
		\left(\frac{\bar{B} \tkappa}{\sqrt{\la}n} + \sqrt{\frac{\bar\sigma^2{\cal N}(\la)}{n}} ~+~ R \la^{r}\right)\log \frac{18}{\delta}.
	}
	Here $\bar{B} = B + 2 R\tkappa$, $\bar{\sigma} = \sigma + 2\sqrt{R} \tkappa$.
	\et
	\bpr
	First we apply Thm.~\ref{thm:main-bound}, then we add a condition on $M$ with respect to $\la$ such that we can bound $\mathfrak{C}(\la, M)$ with $R \la^r$.
	The condition we will consider is the following
	$$ (g_2) ~~ M \geq 4\tkappa^2 \la^{1-2v} {\cal N}(\la)^{2v-1} {\cal F}_\infty(\la)^{2-2v} t^{2-2r}.$$
	Indeed lower bounding with $(g_2)$ the occurrences of $M$ in $\mathfrak{C}(\la, M)$, we have
	\eqals{
		\mathfrak{C}(\la, M) &\leq R\kappa^{2r-1}\left(\sqrt{\frac{\la^{1 + 4r^2 - 2r}{\cal F}_\infty(\la)^{1+4r^2 - 4r} }{4^{2r}\tkappa^{4r-2} {\cal N}(\la)^{4r^2 - 2r}t^{6r - 4r^2 - 2}}} + \sqrt{\frac{\la^{2r}}{4\kappa^2}}\right)\\
		& \leq R\left(\sqrt{\frac{\la^{2r}}{4^{2r} \kappa^{12 r - 8 r^2 - 4} {\cal N}(\la)^{4r^2 - 2r}t^{6r - 4r^2 - 2}}} + \sqrt{\frac{\la^{2r} }{4\kappa^{4-4r}}}\right) \leq R \la^r,
	}
	where the second step is due to ${\cal F}_\infty(\la) \leq \kappa^2/\la$ and $1+4r^2 - 4r \geq 0$ for $r \geq 1/2$, while the last step is due to the following three facts. First, that $4^{2r}{\cal N}(\la)^{4r^2-2r} \geq 4$, since $4r^2 - 2r \geq 0$ on $r \in [1/2,1]$ and, by denoting with $(\la_i(L))_{i\geq 1}$ the eigenvalues of $L$, with $\nor{L} := \la_1(L) \geq \la_{2}(L) \geq \dots \geq 0$, and recalling that $0 \leq \la \leq \frac{3}{4}\nor{L}$, we have
	$${\cal N}(\la) := \tr(L  \Ll^{-1}) = \sum_{i \geq 1} \frac{\la_i(L)}{\la_i(L) + \la} \geq \frac{\la_1(L)}{\la_1(L) + \la} := \frac{\nor{L}}{\nor{L} + \la} > 1/2.$$
	Second, that $t^{6r - 4r^2 - 2} \geq 1$, since $6r - 4r^2 - 2 \geq 0$ on $r \in [1/2,1]$ and  $t \geq 1$, since $0 \leq \la \leq \frac{3}{4}\nor{L} \leq \frac{3}{4}\kappa^2$. 
	Third, that $\kappa^{12 r - 8 r^2 - 4} \geq 1$ and $\kappa^{4-4r} \geq 1$, since $12 r - 8 r^2 - 4 \geq 0$ and $4-4r \geq 0$ on $r \in [1/2,1], \kappa \geq 1$.
	\epr
	
	The following theorem is a specialization of the previous one, under ~\ref{ass:intrinsic},~\ref{ass:compatibility} and an explicit relation of $\la$ with respect to $n$.
	\bt\label{thm:optimal-rates-RF-RKLS-with-constants}
	Let $\delta \in (0,1]$. 
	Under Asm.~\ref{ass:kernel-bounded} and \ref{ass:intrinsic},~\ref{ass:source},~\ref{ass:compatibility},~\ref{ass:noise}, let $p := (2r + \gamma -1)^{-1}$, and 
	\eqals{
		&\;\; n ~~ \geq ~~ \left(2/\nor{L}\right)^{\frac{p+1}{p}} ~ \vee ~ \left(264\kappa^2p~\log(556 \kappa^2\delta^{-1}\sqrt{p \kappa^2})\right)^{1+p}\\
		&\;\; \lambda_n ~~=~~ n^{-\frac{1}{2r + \gamma}},\\
		&\;\; M_n ~~\geq~~ c_0 ~ n^\frac{\alpha +(2r-1)(1+\gamma - \alpha)}{2r+\gamma} \log \frac{108\tkappa^2}{\la \delta},
	}
	with $c_0 = 9(3 +  4\tkappa^2 + \frac{4\kappa^2}{\nor{L}} + \frac{\tkappa^2}{4}Q^{2r-1}F^{2-2r})$, 
	then the following holds with probability at least $1-\delta$,
	\eqal{\label{eq:rf-learning-rate}
		\EE(\widehat{f}_{\la_n,M_n})- \inf_{f \in \hh} \EE(f)  \;\;\leq \;\; c_1 \log^2\frac{18}{\delta} \;\; n^{-\frac{2r}{2r + \gamma}},
	}
	and $c_1 = 64(\bar{B}\tkappa + \bar{\sigma} Q + R)^2$.
	\et
	\bpr
	Let $\la = n^{-\frac{1}{2r+\gamma}}$ in Thm.~\ref{thm:main-bound-M-wrt-la} and substitute ${\cal F}_\infty(\la)$ and ${\cal N}(\la)$ by their bounds given in Asm.~\ref{ass:intrinsic},~\ref{ass:compatibility}. Note that to guarantee that $n$ satisfies the associated constraint  with respect to $\la$, in Thm.~\ref{thm:main-bound-M-wrt-la}, and that $\la$ is in $(0, \frac{3}{4} \nor{L}]$ we need that $n \geq (\frac{4}{3\nor{L}})^{\frac{p+1}{p}}$ and 
	$$
	n \geq \left(264\kappa^2p\log\frac{556 \kappa^2\sqrt{p \kappa^2}}{\delta}\right)^{1+p},
	$$ 
	with $p = \frac{1}{2r + \gamma - 1}$.
	\epr

	Note that the theorems in Sect~\ref{sect:main-res} are corollaries of the theorem above. For the sake of readability, in contrast to Thm.~\ref{thm:optimal-rates-RF-RKLS-with-constants}, the results in Thm.~\ref{thm:rf-simple-universal},~\ref{thm:rf-fast-universal},~\ref{thm:rf-fast-compatibility} are expressed with respect to $\tau := \log \frac{1}{\delta}$. Moreover in the statement of Thm.~\ref{thm:rf-simple-universal},~\ref{thm:rf-fast-universal},~\ref{thm:rf-fast-compatibility} the constants and the logarithmic terms are omitted. They can be recovered by plugging the coefficients detailed in the following proofs in the statement of Thm.~\ref{thm:optimal-rates-RF-RKLS-with-constants}.  
	
	\bpr[\bf Proof of Theorem~\ref{thm:rf-simple-universal}]
	This is an application of Thm.~\ref{thm:optimal-rates-RF-RKLS-with-constants}, with minimum number of assumptions. 
	Indeed the existence of $\fh$ and the fact that $|y| \leq b$ a. s. satisfies Asm.~\ref{ass:noise} with $\sigma = B = 2b$. The fact that $X$ is a Polish space and that $\psi$ is bounded continuous satisfy Asm.~\ref{ass:kernel-bounded}, and so the kernel is bounded by $\kappa^2$. Since the kernel is bounded, we have that Asm.~\ref{ass:intrinsic} is always satisfied with $\gamma = 1, Q = \kappa$; Asm.~\ref{ass:source} is always satisfied with $r = 1/2, R = 1 \vee \nor{\fh}_\hh$; Asm.~\ref{ass:compatibility} is always satisfied with $\alpha = 1, F = \kappa^2$. In particular we have the following constants $n_0 := 4\nor{L}^{-2} ~ \vee ~ \left(264\kappa^2~\log \frac{556 \kappa^3}{\delta}\right)^2$,
	$$c_0 := 9\left(3 + 4\kappa^2 + \frac{4\kappa^2}{\nor{L}} + \kappa^4/4\right), \qquad c_1 := 8(\bar{B}\kappa + \bar{\sigma}\kappa + 1 \vee \nor{\fh}_\hh),$$
	with $\bar{B} := 2b + 2\kappa(1 \vee \nor{\fh}_\hh)$ and $\bar{\sigma} := 2b + 2 \kappa\sqrt{1 \vee \nor{\fh}_\hh}$.
	\epr

	\bcor\label{cor:simple-rates-expectation}
	Under the same assumptions of Thm.~\ref{thm:rf-simple-universal}, if $n \geq \nor{L}^{-2} \vee \left(1056 \log \frac{1056\sqrt{278\kappa^5 b}}{\sqrt{c_1}}\right)^2$ and $\la_n =  n^{-1/2}$, then a number of random features $M_n$ equal to
	$$M_n = 2c_0 ~ \sqrt{n} ~ \log \left(c_2 n\right),$$ 
	is enough to guarantee that
	$$\mathbb{E} ~~ {\cal E}(\widehat{f}_{\la_n,M_n}) - {\cal E}(f_\hh) \leq \frac{40c_1}{\sqrt{n}}.$$
	In particular the constants $c_0, c_1$ are as in Thm.~\ref{thm:rf-simple-universal} and $c_2 = \frac{8\kappa^2\sqrt{b}}{\sqrt{c_1}}$.
	\ecor
	\bpr
	In the rest we will denote ${\cal E}(\widehat{f}_{\la_n,M_n}) - {\cal E}(f_\hh)$, with ${\cal R}(\widehat{f}_{\la_n,M_n})$ and will use the notation of Sect.~\ref{sect:prob-estimates}.
	Fix $\delta_0 = \frac{2c_1}{\kappa^2 b} n^{-1}$. Denote with $E$, the event satisfying ${\cal R}(\widehat{f}_{\la_n,M_n}) > t_0$, with $t_0 := c_1 \log^2\frac{18}{\delta_0} ~ n^{-1/2}$.
	
	First, note that $M_n \geq 2c_0\sqrt{n}\log \left(\frac{8\kappa^2\sqrt{b}}{\sqrt{c_1}} n\right)$, satisfies $M_n \geq c_0 ~ \sqrt{n}\log \frac{108\kappa^2 \sqrt{n}}{\delta_0}$ and any $n \geq \nor{L}^{-2} \vee \left(1056 \log \frac{1056\sqrt{278\kappa^5 b}}{\sqrt{c_1}}\right)^2$ satisfies $n \geq n_0(\delta_0)$ with $n_0(\delta)$ as in Thm.~\ref{thm:rf-simple-universal}. So, we can apply Thm.~\ref{thm:rf-simple-universal}, from which we know that $E$ holds with probability smaller than $\delta_0$.
	
	Second, by Rem.~\ref{rem:excess-risk-to-Ltwo} and Rem.~\ref{rem:rf-in-ltwo}, we have that \eqal{
		{\cal R}(\widehat{f}_{\la_n,M_n}) &= \nor{\tS \tCnl^{-1} \tSn^* \yn - P\frho}_\rhox \\
		&\leq \nor{\tS} \nor{\tCnl^{-1}}\nor{\tSn^*}\nor{\yn}_{\R^n}  + \nor{P}\nor{\frho}_\rhox \\
		&\leq \frac{\kappa^2 b}{\la} + b \leq \frac{2\kappa^2 b}{\la} =: R_0,
	}
	where we used the fact that $\nor{\tS},\nor{\tSn^*} \leq \kappa$ (see Def.~\ref{def:all-ops-tilde}), that $\nor{\tCnl^{-1}} \leq \la^{-1}$, that $\nor{\yn}^2 = \frac{1}{n} \sum_{y_i}^2$, that $\frho(x) = \mathbb{E} [y|x]$, that the $y$'s are bounded in $[-b, b]$, and the fact that $\kappa^2/\la \geq 1$, by definition of $\kappa, \la$.
	
	Now, by denoting with ${\bf 1}_E$ the indicator function for $E$, we have
	$$ \mathbb{E}~{\cal R}(\widehat{f}_{\la_n,M_n}) =  \mathbb{E}~{\bf 1}_E {\cal R}(\widehat{f}_{\la_n,M_n}) ~~ + ~~ \mathbb{E}~ {\bf 1}_{Z \setminus E} {\cal R}(\widehat{f}_{\la_n,M_n}).$$
	In particular
	$$\mathbb{E}~{\bf 1}_E {\cal R}(\widehat{f}_{\la_n,M_n}) \leq R_0 ~\mathbb{E}~{\bf 1}_E = R_0 \PP(E) \leq R_0 \delta_0.$$
	For the second term we have
	\eqals{
		\mathbb{E}~~{\bf 1}_{Z \setminus E} {\cal R}(\widehat{f}_{\la_n,M_n})  ~~~&=~~~ \mathbb{E}~~{\bf 1}_{\{{\cal R}(\widehat{f}_{\la_n,M_n}) \leq t_0\}} ~ {\cal R}(\widehat{f}_{\la_n,M_n}) ~~~=~~~ \int_0^{t_0} \PP({\cal R}(\widehat{f}_{\la_n,M_n}) > t) dt.
	}
	By changing variable, in the integral above, via $t = \frac{c_1}{\sqrt{n}} \log^2\frac{18}{\delta}$, and using the fact that $\PP({\cal R}(\widehat{f}_{\la_n,M_n}) > \frac{c_1}{\sqrt{n}} \log^2\frac{18}{\delta}) \leq \delta$, we have
	\eqals{
		\int_0^{t_0} \PP({\cal R}(\widehat{f}_{\la_n,M_n}) > t) dt &= \frac{2c_1}{\sqrt{n}} \int_{\delta_0}^{18} \frac{\log \frac{18}{\delta}}{\delta} \PP\left({\cal R}(\widehat{f}_{\la_n,M_n}) > \frac{c_1}{\sqrt{n}} \log^2\frac{18}{\delta}\right) d\delta \\
		& \leq \frac{2c_1}{\sqrt{n}} \int_{\delta_0}^{18} \log \frac{18}{\delta} d\delta  \\
		& = ~~ \frac{2c_1}{\sqrt{n}}\left(18 - \delta_0\left(1 + \log\frac{18}{\delta_0}\right)\right) ~~~\leq ~~~ \frac{36c_1}{\sqrt{n}}. 
	}
	So finally we have
	$$ \mathbb{E}~{\cal R}(\widehat{f}_{\la_n,M_n}) \leq R_0 \delta_0 + \frac{36c_1}{\sqrt{n}} = \frac{40c_1}{\sqrt{n}}.$$
	\epr
	
	\bpr[\bf Proof of Theorem~\ref{thm:rf-fast-universal}]
	This is an application of Thm.~\ref{thm:optimal-rates-RF-RKLS-with-constants}, where assumption Asm.~\ref{ass:compatibility} is satisfied with $F = \kappa^2$ and $\alpha = 1$. 
	Indeed the existence of $\fh$ and the fact that $|y| \leq b$ a. s. satisfies Asm.~\ref{ass:noise} with $\sigma = B = 2b$. The fact that $X$ is a Polish space and that $\psi$ is bounded continuous satisfy, Asm.~\ref{ass:kernel-bounded}, and so the kernel is bounded by $\kappa^2$. Since the kernel is bounded, we have that Asm.~\ref{ass:compatibility} is always satisfied with $\alpha = 1, F = \kappa^2$. Asm.~\ref{ass:noise} and Asm.~\ref{ass:source} are directly satisfied by Asm.~\ref{ass:prior}.
	In particular we obtain $n_0 := \left(2/\nor{L}\right)^{\frac{p+1}{p}} ~ \vee ~ \left(264\kappa^2p~\log(556 \kappa^2\delta^{-1}\sqrt{p \kappa^2})\right)^{1+p}$,
	$$ c_0 := 9\left(3 +  4\tkappa^2 + \frac{4\kappa^2}{\nor{L}} + \frac{\tkappa^{4-2r}}{4}Q^{2r-1}\right), \quad 
	c_1 := 64\left(\bar{B}\tkappa + \bar{\sigma} Q + R\right)^2,$$
	with $\bar{B} := 2b + 2\kappa R$ and $\bar{\sigma} := 2b + 2 \kappa\sqrt{R}$.
	\epr
	
	\bpr[\bf Proof of Example~\ref{ex:lsrf}]
	By definition of $\psi_s, \pi_s$ we have
	\eqals{
		\int \psi_s(x, \omega)\psi_s(x',\omega) d\pi_s(\omega) &= \int \psi(x, \omega)\sqrt{C_s s(\omega)}\psi_s(x',\omega)\sqrt{C_s s(\omega)}  \frac{1}{C_s s(\omega)}d\pi(\omega)  \\
		& = \int \psi(x, \omega)\psi(x',\omega) d\pi(\omega) = K(x,x').
	}
	Now we show that $\psi_s, \pi_s$ achieves ${\cal F}_\infty(\la) = {\cal N}(\la)$. By recalling that $s(\omega) = \nor{(L+\la I)^{-1/2} \psi(\cdot, \omega)}_\rhox^{-2}$, we have
	\eqals{
		{\cal F}_\infty(\la) & = \sup_{\omega \in \Omega} \nor{(L+\la I)^{-1/2} \psi_s(\cdot, \omega)}_\rhox^2 = C_s\sup_{\omega \in \Omega} s(\omega) \nor{(L+\la I)^{-1/2} \psi(\cdot, \omega)}_\rhox^2 \\
		&= C_s\sup_{\omega \in \Omega} \nor{(L+\la I)^{-1/2} \psi(\cdot, \omega)}_\rhox^{-2}  \nor{(L+\la I)^{-1/2} \psi(\cdot, \omega)}_\rhox^2 
		= C_s.
	}
	We recall that $C_s = \int \frac{1}{s(\omega)} d\pi$. Denoting with $\psi_\omega$ the function $\psi(\cdot, \omega)$ and considering that $\nor{A x}_\rhox = \tr(A^{2} (x \otimes x))$ for any bounded symmetrix linear operator $A$ and vector $v$, and that the trace is linear,
	\eqals{
		{\cal F}_\infty(\la) &= C_s = \int \nor{(L+\la I)^{-1/2} \psi_\omega}_\rhox^2 d \pi(\omega) = \int \tr((L+\la I)^{-1} (\psi_\omega \otimes \psi_\omega)) d \pi(\omega) \\
		& = \tr\left((L+\la I)^{-1} \int(\psi_\omega \otimes \psi_\omega) d \pi(\omega) \right) = \tr((L+\la I)^{-1}L) = {\cal N}(\la).
	} 
	where the fact that $L = \int \psi_\omega \otimes \psi_\omega d\pi(\omega)$ is due to Lemma~\ref{lm:L-characterization}.
	\epr

	\bpr[\bf Proof of Theorem~\ref{thm:rf-fast-compatibility}]
	This is a version of Thm.~\ref{thm:optimal-rates-RF-RKLS-with-constants}, with simplified set of assumptions. 
	Indeed the existence of $\fh$ and the fact that $|y| \leq b$ a. s. satisfies Asm.~\ref{ass:noise} with $\sigma = B = 2b$. The fact that $X$ is a Polish space and that $\psi$ is bounded continuous, satisfy Asm.~\ref{ass:kernel-bounded}, and so the kernel is bounded by $\kappa^2$. Asm.~\ref{ass:noise} and Asm.~\ref{ass:source} are directly satisfied by Asm.~\ref{ass:prior}.In particular we obtain $n_0 := \left(2/\nor{L}\right)^{\frac{p+1}{p}} ~ \vee ~ \left(264\kappa^2p~\log(556 \kappa^2\delta^{-1}\sqrt{p \kappa^2})\right)^{1+p}$,
	$$ c_0 := 9\left(3 +  4\tkappa^2 + \frac{4\kappa^2}{\nor{L}} + \frac{\tkappa^2}{4}Q^{2r-1} F^{2-2r}\right), \quad 
	c_1 := 64\left(\bar{B}\tkappa + \bar{\sigma} Q + R\right)^2,$$
	with $\bar{B} := 2b + 2\kappa R$ and $\bar{\sigma} := 2b + 2 \kappa\sqrt{R}$.
	\epr
	
	
	\section{Concentration Inequalities}\label{sect:conc-ineq}
	Here we recall some standard concentration inequalities that will be used in Sect.~\ref{sect:prob-estimates}.
	The following inequality is from Thm.3 of \cite{boucheron2004concentration} and will be used in Lemma~\ref{prop:emp-eff-dim}, together with other inequalities, to concentrate the empirical effective dimension to the true effective dimension.
	\bp[Bernstein's inequality for sum of random variables]\label{prop:tail-bern}
	Let $x_1,\dots,x_n$ be a sequence of independent and identically distributed random variables on $\R$ with zero mean. If there exists an $T, S \in \R$ such that $x_i \leq T$ almost everywhere and $\mathbb{E} x_i^2 \leq S$, for $i \in \{1,\dots,n\}$. For any $\delta > 0$ the following holds with probability at least $1-\delta$:
	$$\frac{1}{n} \sum_{i=1}^n x_i \leq \frac{2T\log\frac{1}{\delta}}{3n} + \sqrt{\frac{2S\log\frac{1}{\delta}}{n}}.$$
	If there exists $T' \geq \max_i|x_i|$ almost everywhere, then the same bound, with $T'$ instead of $T$, holds for the for the absolute value of the left hand side, with probability at least  $1-2\delta$.
	\ep
	\bpr
	It is a restatement of Theorem 3 of \cite{boucheron2004concentration}.
	\epr
	The following inequality is and adaptation of Thm.~3.3.4 in \cite{yurinsky1995sums} and is a generalization of the previous one to random vectors. It is used primarily in Lemma~\ref{lm:sampling-error}, to control the sample error. Moreover it is used in Prop.~\ref{prop:emp-eff-dim}, Lemma~\ref{lm:conc-beta}, to control the empirical effective dimension and to bound the term $\beta$ of Thm.~\ref{thm:geo-dec}, in the main theorem. Finally it is used to prove the inequality in Prop.~\ref{prop:Clahalf-C-control}.
	\bp[Bernstein's inequality for sum of random vectors]\label{prop:tail-vectors}
	Let $z_1,\dots,z_n$ be a sequence of independent identically distributed random vectors on a separable Hilbert space $\hh$. Assume $\mu = \mathbb{E} z_i$ exists and let $\sigma, M \geq 0$ such that
	$$ \mathbb{E} \nor{z_i - \mu}_{\hh}^p \leq \frac{1}{2}p!\sigma^2 M^{p-2}, \quad \forall p \geq 2,$$
	for any $i \in \{1,\dots,n\}$. Then for any $\delta \in (0,1]$:
	$$\left\|\frac{1}{n}\sum_{i=1}^n z_i - \mu\right\|_{\hh} \leq \frac{2 M \log\frac{2}{\delta}}{n} + \sqrt{\frac{2 \sigma^2\log\frac{2}{\delta}}{n}}$$
	with probability at least $1 - \delta$.
	\ep
	\bpr
	restatement of Theorem 3.3.4 of \cite{yurinsky1995sums}.
	\epr
	The following inequality is essentially Thm.~7.3.1 in  \cite{tropp2012user} (generalized to separable Hilbert spaces by the technique in Section 4 of \cite{minsker2011some}). It is a generalization of the Bernstein inequality to random operators. It is mainly used to prove the inequality in Prop.~\ref{prop:concentration-ClaC}.
	\bp[Bernstein's inequality for sum of random operators]\label{prop:tail-operators}
	Let $\hh$ be a separable Hilbert space and let $X_1,\dots,X_n$ be a sequence of independent and identically distributed self-adjoint positive random operators on $\hh$. Assume that there exists $\mathbb{E} X_i = 0$ and $\la_{\max}(X_i) \leq T$ almost surely for some $T > 0$, for any $i \in \{1,\dots,n\}$. Let $S$ be a positive operator such that $\mathbb{E} (X_i)^2 \leq S$. Then for any $\delta \in (0,1]$ the following holds
	\begin{align*}
	\la_{\max}\left(\frac{1}{n}\sum_{i=1}^n X_i\right) \leq \frac{2T\beta}{3n} + \sqrt{\frac{2\nor{S}{}\beta}{n}}
	\end{align*}
	with probability at least $1-\delta$. Here $\beta = \log\frac{2\tr S}{\nor{S}{}\delta}$. 
	
	If there exists $L'$ such that $L' \geq \max_i \nor{X_i}{}$ almost everywhere, then the same bound holds with $L'$ instead of $L$ for the operator norm, with probability at least $1-2\delta$.
	\ep
	\bpr
	The theorem is a restatement of Theorem 7.3.1 of  \cite{tropp2012user} generalized to the separable Hilbert space case by means of the technique in Section 4 of \cite{minsker2011some}.
	\epr
	
	\section{Operator Inequalities}\label{sect:op-ineq}
	Let $\hh, \kk$ be separable Hilbert spaces and $A, B: \hh \to \hh$ bounded linear operators.
	
	The following inequality is needed to prove the interpolation inequality in Prop.~\ref{prop:interp}, that is needed to perform a fine split of the computational error.
	\bp[Cordes Inequality \cite{fujii1993norm}]\label{prop:cordes}
	If $A, B$ are self-adjoint and positive, then
	$$\nor{A^s B^s} \leq \nor{A B}^s \quad \textrm{when } 0 \leq s \leq 1$$ 
	\ep
	
	\section{Auxiliary Results}\label{sect:aux-results}
	The next proposition is used in Lemma~\ref{lm:conc-S} to control the sample error. It is based on the Bernstein inequality for random vectors, Prop.~\ref{prop:tail-vectors}.
	\bp\label{prop:Clahalf-C-control}
	Let  $\hh, \kk$ be two separable Hilbert spaces and $(v_1, z_1),\dots,(v_n,z_n) \in \hh \times \kk$, with $n \geq 1$, be independent and identically distributed random pairs of vectors, such that there exists a constant $\tkappa > 0$ for which $\nor{v}_{\hh} \leq \tkappa$ and $\nor{z}_{\hh} \leq \kappa$ almost everywhere. Let $Q = \mathbb{E} \, v \otimes v$, let $T = \mathbb{E} \, v \otimes z$ and $T_n = \frac{1}{n} \sum_{i=1}^{n} v_i \otimes z_i$. For any $0 < \lambda \leq \nor{Q}{}$ and any $\tau \geq 0$, the following holds
	$$
	\nor{(Q+\lambda I)^{-1/2}(T-T_n)}_{HS} \leq \frac{4\sqrt{\tilde{\cal F}_\infty(\la)}\kappa \log\frac{2}{\tau}}{n} + \sqrt{\frac{4\kappa^2 \tilde{\cal N}(\la)\log\frac{2}{\tau}}{n}}$$
	with probability at least $1-\tau$, where $\esssup$ denotes the essential supremum and
	$$\tilde{\cal F}_\infty(\la) := \esssup_{v \in \hh} ~\nor{(Q+\lambda I)^{-1/2}v}^2, \quad \tilde{\cal N}(\la) := \tr((Q + \la I)^{-1} Q).$$
	In particular, we recall that $\tilde{\cal N}(\la) \leq \tilde{\cal F}_\infty(\la) \leq \frac{\kappa^2}{\lambda}$.
	\ep 
	\bpr
	Define for any $i \in \{1,\dots,n\}$ the random operator $\zeta_i = (Q+\lambda I)^{-1/2} v_i \otimes z_i$. Note that $\mathbb{E} \zeta_i = (Q+\lambda I)^{-1/2} T$. Since $\zeta_i$ is a vector in the Hilbert space of Hilbert-Schmidt operators on $\hh$, we study the moments of $\nor{\zeta_i - \mathbb{E} \zeta_i}_{HS}$ in order to apply Prop.~\ref{prop:tail-vectors}.
	We recall that
	\eqals{
		\esssup \nor{\zeta_i - \mathbb{E} \zeta_i}_{HS} & \leq \esssup \nor{\zeta_i}_{HS} + \mathbb{E} \nor{\zeta_i}_{HS} \leq 2 \esssup \nor{\zeta_i}_{HS} \\
		& = 2 \esssup \nor{(Q+\lambda I)^{-1/2}v_i \otimes z_i} _{HS} \leq \esssup \nor{(Q+\lambda I)^{-1/2}v_i}_\hh \nor{z_i}_\kk\\
		& \leq 2 \esssup  \nor{(Q+\lambda I)^{-1/2}v_i}_\hh ~ \esssup \nor{z_i}_\kk = 2\tilde{F}_\infty(\la)^{1/2} \kappa.
	}
	
	For any $p \geq 2$ we have
	\eqals{
		\mathbb{E}\, \nor{\zeta_i - \mathbb{E} \zeta_i}_{HS}^p & \leq (\esssup_{z} \nor{\zeta_i - \mathbb{E} \zeta_i}^{p-2}) (\mathbb{E}\, \nor{\zeta_i - \mathbb{E} \zeta_i}_{HS}^2) \\
		& \leq (2\tilde{F}_\infty(\la)^{1/2} \kappa)^{p-2} \mathbb{E} \nor{\zeta_i - \mathbb{E} \zeta_i}_{HS}^2.
	}
	Now we study $\mathbb{E} \nor{\zeta_i - \mathbb{E} \zeta_i}_{HS}^2$,
	\eqals{
		\mathbb{E} \nor{\zeta_i - \mathbb{E} \zeta_i}_{HS}^2 & =\tr(\mathbb{E}\zeta_i \otimes \zeta_i - (\mathbb{E} \zeta_i)^2) \leq \tr(\mathbb{E}\zeta_i \otimes \zeta_i)\\
		& = \mathbb{E} ~~ \nor{z_i}^2 \tr((Q + \la I)^{-1/2} (v_i \otimes v_i)(Q + \la I)^{-1/2} ) \\
		& \leq \esssup \nor{z_i}_\kk^2 ~~ \mathbb{E} \tr((Q + \la I)^{-1/2}(v_i \otimes v_i)(Q + \la I)^{-1/2} ) \\
		& \leq \esssup \nor{z_i}_\kk^2 \tr((Q + \la I)^{-1/2} \mathbb{E}(v_i \otimes v_i)(Q + \la I)^{-1/2} ) \\
		& \leq \kappa^2 \tr((Q + \la I)^{-1} Q) = \kappa^2 \tilde{\cal N}(\la),
	}
	for any $1 \leq i \leq n$. Therefore for any $p \geq 2$ we have
	$$\mathbb{E}\, \nor{\zeta_i - \mathbb{E} \zeta_i}_{HS}^p \leq \frac{1}{2}p!\sqrt{2\kappa^2 {\cal N}(\la)}^2 (2\tilde{F}_\infty(\la)^{1/2} \kappa)^{p-2}.$$
	Finally we apply Prop.~\ref{prop:tail-vectors}.
	\epr
	
	The following inequality, together with Prop.~\ref{prop:invAtoinvB}, is used in Prop.~\ref{prop:emp-eff-dim}, Lemmas~\ref{lm:conc-beta},~\ref{lm:conc-C}. A similar technique can be found in \cite{caponnetto2006adaptation}.
	\bp\label{prop:concentration-ClaC}
	Let $v_1,\dots,v_n$ with $n \geq 1$, be independent and identically distributed random vectors on a separable Hilbert spaces $\hh$ such that $Q = \mathbb{E} \, v \otimes v$ is trace class, and for any $\la > 0$ there exists a constant ${\cal F}_\infty(\la) < \infty$ such that $\scal{v}{(Q+\la I)^{-1}v} \leq {\cal F}_\infty(\la)$ almost everywhere. Let $Q_n = \frac{1}{n} \sum_{i=1}^{n} v_i \otimes v_i$ and take $0 < \lambda \leq \nor{Q}{}$.  Then for any $\delta \geq 0$, the following holds with probability at least $1-2\delta$
	$$
	\nor{(Q+\lambda I)^{-1/2}(Q-Q_n)(Q+\lambda I)^{-1/2}}{} \leq \frac{2\beta(1 + {\cal F}_\infty(\la))}{3 n} + \sqrt{\frac{2\beta {\cal F}_\infty(\la)}{n}},$$
	where $\beta = \log \frac{4 \tr Q}{\lambda\delta}$. Moreover, with the same probability
	$$
	\la_{\max}\left((Q+\lambda I)^{-1/2}(Q-Q_n)(Q+\lambda I)^{-1/2}\right) \leq \frac{2\beta}{3 n} + \sqrt{\frac{2\beta {\cal F}_\infty(\la)}{n}}.$$
	\ep
	\bpr
	Let $Q_\la = Q + \la I$.
	Here we apply Prop.~\ref{prop:tail-operators} on the random variables $Z_i = M - Q_\la^{-1/2}v_i \otimes Q_\la^{-1/2}v_i$ with $M = Q_\la^{-1/2}QQ_\la^{-1/2}$ for $1 \leq i \leq n$. Note that the expectation of $Z_i$ is $0$.
	The random vectors are bounded by
	$$\nor{Q_\la^{-1/2}QQ_\la^{-1/2} - Q_\la^{-1/2}v_i \otimes Q_\la^{-1/2}v_i}{} \leq \scal{v_i}{Q_\la^{-1}v_i} + \nor{Q_\la^{-1/2}QQ_\la^{-1/2}}{} \leq {\cal F}_\infty(\la) + 1,$$
	almost everywhere, for any $1 \leq i \leq n$. The second order moment is
	\begin{align*}
	\mathbb{E} (Z_i)^2 &= \mathbb{E} \;\;\scal{v_i}{Q_\la^{-1} v_i} \;Q_\la^{-1/2}v_i \otimes Q_\la^{-1/2}v_i \;\;\;-\;\;\; Q_\la^{-2}Q^2 \\
	& \leq {\cal F}_\infty(\la) \mathbb{E} Q_\la^{-1/2}v_i \otimes Q_\la^{-1/2}v_i = {\cal F}_\infty(\la) Q = S,
	\end{align*}
	for $1 \leq i \leq n$.
	Now we can apply Prop.~\ref{prop:tail-operators}. Now some considerations on $\beta$.
	It is $\beta = \log \frac{2\tr S}{\nor{S}{}\delta} = \frac{2\tr Q_\la^{-1}Q}{\nor{Q_\la^{-1}Q}{}\delta}$, now $\tr Q_\la^{-1}Q \leq \frac{1}{\la} \tr Q$. We need a lower bound for $\nor{Q_\la^{-1}Q}{} = \frac{\sigma_1}{\sigma_1 + \la}$ where $\sigma_1 = \nor{Q}$ is the biggest eigenvalue of $Q$, now $\la \leq \sigma_1$ thus $\frac{\sigma_1}{\sigma_1 + \la} \geq 1/2$ and so $\beta \leq \log\frac{2\tr Q_\la^{-1}Q}{\nor{Q_\la^{-1}Q}{}\delta} \leq \log\frac{4\tr Q}{\la\delta}$.
	
	For the second bound of this proposition we use the second bound of Prop.~\ref{prop:tail-operators}, the analysis remains the same except for uniform bound on $Z_1$, that now is
	$$ \sup_{f \in \hh} \scal{f}{Z_1 f} = \sup_{f \in \hh} \scal{f}{Q_\la^{-1}Qf} - \scal{f}{Q_\la^{-1/2}v_i}^2 \leq \sup_{f \in \hh} \scal{f}{Q_\la^{-1}Qf} \leq 1.$$
	\epr
	In the following remark, we start from the result of the previous proposition, expressing the conditions on $n$ and $\la$ with respect to a given value for the bound. 
	\br\label{rem:ClC}
	With the same notation of Prop.~\ref{prop:concentration-ClaC}, assume that $\nor{v} \leq \tkappa$ almost everywhere\footnote{Or equivalently define $\tkappa^2$ with respect to ${\cal F}_\infty(\la)$ as $\tkappa^2 = \inf_{\la>0} {\cal F}_\infty(\la)(\nor{Q} + \la)$.}, then we have that
	\begin{enumerate}
		\item for any $t \in (0,1]$, when $n \geq \frac{2}{t^2}\left(\frac{2t}{3} + {\cal F}_\infty(\la)\right)\log \frac{4 \tkappa^2}{\lambda\delta}$ and $\la \leq \nor{Q}$, we have
		$$\la_{\max}\left((Q+\lambda I)^{-1/2}(Q-Q_n)(Q+\lambda I)^{-1/2}\right) \leq t,$$
		with probability at least $1-\delta$.
		\item The equation above holds with the same probability with $t = 1/2$, when $\frac{9\tkappa^2}{n}\log\frac{n}{2\delta}\leq \la \leq \nor{Q}{}$ and $n \geq 405\tkappa^2 \vee 67\tkappa^2 \log \frac{\tkappa^2}{2\delta}$.
		\item The equation above holds with the same probability with $t = 1/3$, when $\frac{19\tkappa^2}{n}\log\frac{n}{4\delta}\leq \la \leq \nor{Q}{}$ and $n \geq 405\tkappa^2 \vee 67\tkappa^2 \log \frac{\tkappa^2}{2\delta}$.
	\end{enumerate}
	\er
	
	The next proposition, together with Prop.~\ref{prop:emp-eff-dim} are a restatement of Prop.~1 of \cite{rudi2015less}. In particular the next proposition performs the analytic decomposition of the difference between the empirical and the true effective dimension, while Prop.~\ref{prop:emp-eff-dim}, bounds the decomposition in probability. 
	\bp[Geometry of Empirical Effective Dimension]\label{prop:geo-emp-eff-dim} 
	Let $\cal{L}$ be an Hilbert space. Let $L, \tL: {\cal L} \to {\cal L}$ be two bounded positive linear operators, that are trace class. Given $\la > 0$, let \eqals{
		{\cal N}(\la) = \tr (L\Ll^{-1}) \quad\textrm{and}\quad \tilde{\cal N}(\la) = \tr (\tL\tLl^{-1}),
	}
	with $\Ll = L + \la I$. Then
	$$ 
	|\tilde{\cal N}(\la) - {\cal N}(\la)| \leq \la e(\la) + \frac{d(\la)^2}{1-c(\la)}
	$$
	where $c(\la) = \la_{\max}(\tilde{B})$, $d(\la) = \nor{\tilde{B}}_{\rm HS}$ with $\tilde{B} = \Ll^{-1/2}(L-\tL)\Ll^{-1/2}$ and $e(\la) = \tr({\Ll^{-1/2} \tilde{B} \Ll^{-1/2}})$.
	\ep
	\bpr
	First of all note that $\la_{\max}(\tilde{B})$, the biggest eigenvalue of $\la_{\max}$, is smaller than $1$ since $\tilde{B}$ is the difference of two positive operators and the biggest eigenvalue of the minuend operator is $\nor{\Ll^{-1/2}\L\Ll^{-1/2}} = \frac{\la_{\max}(L)}{\la_{\max}(L) + \la} < 1$. Then we can use the fact that $A(A+\la I)^{-1} = I - \la (A+\la I)^{-1}$ for any bounded linear operator $A$ and that $\tLl^{-1} = \Ll^{-1/2}(I-\tilde{B})^{-1}\Ll^{-1/2}$ (see the proof of Prop.~\ref{prop:invAtoinvB}), since $\la_{\max}(\tilde{B}) < 1$, to obtain
	\begin{align*}
	|\tilde{\cal N}(\la) - {\cal N}(\la)|  & =  |\tr (\tLl^{-1}\tL  - \L\Ll^{-1})| = \la |\tr (\tLl^{-1}  - \Ll^{-1})| = |\la \tr( \tLl^{-1}(\tL - \L)\Ll^{-1})| \\
	{} & = |\la \tr (\Ll^{-1/2}(I - \tilde{B})^{-1}\Ll^{-1/2} (\tL - \L)\Ll^{-1/2}\Ll^{-1/2})|\\
	{} & =  |\la \tr(\Ll^{-1/2}(I - \tilde{B})^{-1}\tilde{B}\Ll^{-1/2})|.
	\end{align*}
	Considering that for any bounded symmetric linear operator $X$ the following identity holds $$(I - X)^{-1}X = X + X(I-X)^{-1}X,$$ when $\la_{\max}(X) < 1$, we have 
	\begin{align*}
	\la |\tr(\Ll^{-1/2}(I - \tilde{B})^{-1}\tilde{B}\Ll^{-1/2})|  & \leq  \underbrace{\la |\tr (\Ll^{-1/2}\tilde{B}\Ll^{-1/2})|}_{A} + \underbrace{\la |\tr (\Ll^{-1/2}\tilde{B}(I - \tilde{B})^{-1}\tilde{B}\Ll^{-1/2})|}_{B}.
	\end{align*} 
	The term $A$ is just equal to $\la e(\la)$.
	Now, by definition of Hilbert-Schmidt norm, the term $B$ can be written as $B = \nor{\la^{1/2}\Ll^{-1/2}\tilde{B}(I - \tilde{B})^{-1/2}}_{\rm HS}^2$, thus we have
	$$B = \nor{\la^{1/2}\Ll^{-1/2}\tilde{B}(I - \tilde{B})^{-1/2}}_{\rm HS}^2 \leq \nor{\la^{1/2}\Ll^{-1/2}}{}^2\nor{\tilde{B}}_{\rm HS}^2\nor{(I - \tilde{B})^{-1/2}}{}^2 \leq (1-c(\la))^{-1}d(\la)^2,$$
	since $\nor{(I - \tilde{B})^{-1/2}}{}^2 = (1 - \la_{\max}(\tilde{B}))^{-1}$ because the spectral function $(1-\sigma)^{-1}$ is increasing and positive on $[-\infty,1)$.
	\epr
	The next result is essentially Prop. 7 of \cite{rudi2015less}, while a similar technique can be found in \cite{caponnetto2006adaptation}. It is used, mainly together with Prop.~\ref{prop:concentration-ClaC}, to give multiplicative bounds to empirical operators. It is used in the analytic decomposition of the excess risk, in Prop.~\ref{prop:geo-emp-eff-dim}, Thm.~\ref{thm:geo-dec}. 
	\bp\label{prop:invAtoinvB}
	Let $\hh$ be a separable Hilbert space, let $A, B$ two bounded self-adjoint positive linear operators on $\hh$ and $\la > 0$. Then
	$$ \nor{(A + \lambda I)^{-1/2}B^{1/2}}{} \leq \nor{(A + \lambda I)^{-1/2}(B+\la I)^{1/2}}{} \leq (1-\beta)^{-1/2}$$
	with
	$$ \beta = \la_{\max}\left[(B+\lambda I)^{-1/2}(B-A)(B+\lambda I)^{-1/2}\right].$$
	Note that $\beta \leq \frac{\la_{\max}(B)}{\la_{\max}(B) + \la} < 1$ by definition.
	\ep
	\bpr 
	Let $B_\la = B + \la I$. First of all note that $\beta < 1$ for any $\la > 0$. Indeed, by exploiting the variational formulation of the biggest eigenvalue, we have
	\eqals{
		\beta &= \la_{\max}(B_\la^{-1/2}(B-A)B_\la^{-1/2}) = \sup_{f \in \hh, \nor{f}_\hh \leq 1} \scal{f}{B_\la^{-1/2}(B-A)B_\la^{-1/2}f}\\
		& = \sup_{f \in \hh, \nor{f}_\hh \leq 1} \scal{f}{B_\la^{-1/2}BB_\la^{-1/2}f} - \scal{f}{B_\la^{-1/2}AB_\la^{-1/2}f}\\
		& \leq \sup_{f \in \hh, \nor{f}_\hh \leq 1} \scal{f}{B_\la^{-1/2}BB_\la^{-1/2}f} = \la_{\max}(B_\la^{-1/2}BB_\la^{-1/2}) \\
		&\leq \frac{\la_{\max}(B)}{\la_{\max}(B) + \la} < 1,
	}
	since $A$ is a positive operator and thus $\scal{f}{B_\la^{-1/2}AB_\la^{-1/2}f} \geq 0$ for any $f \in \hh$. Now note that
	\begin{align*}
	(A+\lambda I)^{-1} & = \left[(B + \lambda I) - (B - A)\right]^{-1} \\
	& = \left[B_\la^{1/2}\left(I - B_\la^{-1/2}(B - A)B_\la^{-1/2}\right)B_\la^{1/2}\right]^{-1}\\
	& = B_\la^{-1/2}\left[I - B_\la^{-1/2}(B - A)B_\la^{-1/2}\right]^{-1}B_\la^{-1/2}.
	\end{align*}
	Now let $X = (I - B_\la^{-1/2}(B - A)B_\la^{-1/2})^{-1}$. We have that, 
	\begin{align*}
	\nor{(A + \lambda I)^{-1/2}B_\la^{1/2}}{} &= \nor{B_\la^{1/2}(A + \lambda I)^{-1}B_\la^{1/2}}{}^{1/2} = \nor{X}{}^{1/2}
	\end{align*}
	because $\nor{Z}{} = \nor{Z^*Z}{}^{1/2}$ for any bounded operator $Z$. Note that 
	\eqals{
		\nor{(A + \lambda I)^{-1/2}B^{1/2}}{} &= \nor{(A + \lambda I)^{-1/2}B_\la^{1/2}B_\la^{-1/2} B^{1/2}}{} \leq \nor{(A + \lambda I)^{-1/2}B_\la^{1/2}}\nor{B_\la^{-1/2} B^{1/2}}{}  \\
		&\leq \nor{X}{}^{1/2}\nor{B_\la^{-1/2}B^{1/2}}{} \leq \nor{X}{}^{1/2}.
	}
	Finally let $Y = B_\la^{-1/2}(B - A)B_\la^{-1/2}$, we have seen that $\beta = \la_{\max}(Y) < 1$, then
	$$\nor{X}{} = \nor{(I - Y)^{-1}}{} = (1 - \lambda_{\max}(Y))^{-1},$$
	since $X = w(Y)$ with $w(\sigma) = (1-\sigma)^{-1}$ for $-\infty \leq \sigma < 1$, and $w$ is positive and monotonically increasing on the domain.  
	\epr
	
	The following proposition is used to give a fine analytical decomposition of the excess risk in Prop.~\ref{prop:geo-emp-eff-dim}, Thm.~\ref{thm:geo-dec}. A similar interpolation inequality for finite dimensional matrices, can be found in \cite{bhatia2013matrix}. Here we prove it for bounded linear operators on separable Hilbert spaces.
	\begin{proposition}
		\label{prop:interp}
		Let $\hh,\kk$ be two separable Hilbert spaces and $X,A$ be bounded linear operators, with $X:\hh\to \kk$ and $A:\hh \to \hh$ be positive semidefinite. 
		\begin{equation}
		\nor{XA^\sigma} \leq \nor{X}^{1-\sigma}\nor{XA}^\sigma, \qquad \forall \sigma \in [0,1].
		\end{equation}
	\end{proposition}
	\begin{proof}
		\begin{equation}
		\nor{XA^\sigma}  =  \nor{A^\sigma (X^*X)^{\frac{1}{\sigma}\sigma} A^\sigma}^{\frac{1}{2}} \leq \nor{A (X^*X)^{\frac{1}{\sigma}}A}^{\frac{\sigma}{2}} 
		\end{equation}
		where the last inequality is due to Cordes (see Proposition~\ref{prop:cordes}). Then we have that $(X^*X)^\frac{1}{\sigma} \preccurlyeq \nor{X}^\frac{2(1-\sigma)}{\sigma} X^*X$ (where $\preccurlyeq$ is the L\"owner partial ordering on positive operators) and so
		$$A (X^*X)^{\frac{1}{\sigma}}A \preccurlyeq \nor{X}^\frac{2(1-\sigma)}{\sigma} A X^*X A$$
		that implies
		\begin{equation}
		\nor{A (X^*X)^{\frac{1}{\sigma}}A}^{\frac{\sigma}{2}} \leq
		\nor{X}^{1-\sigma} \nor{A X^*X A}^\frac{\sigma}{2} = \nor{X}^{1-\sigma}\nor{XA}^\sigma
		\end{equation}
	\end{proof}

	In the next proposition, we bound ${\cal N}_M(\la)$ in terms of the effective dimension ${\cal N}(\la)$ that is the one associated to the kernel $K$. The proof of Prop.~\ref{prop:emp-eff-dim} analogous to the one of Prop.~1 of \cite{rudi2015less}, with simpler proof and slightly improved constants. 
	
	\bp[Bounds on the Effective Dimension]\label{prop:emp-eff-dim} Let $\tilde{\cal N}(\la) = \tr \tL\tLl^{-1}$. Under Assumption~\ref{ass:kernel-bounded}, for any $\delta > 0$, $\la \leq \nor{L}$ and $m \geq (4 + 18{\cal F}_\infty(\la))\log \frac{12\tkappa^2}{\la \delta}$, then the following holds with probability at least $1-\delta$,
	$$
	|\tilde{\cal N}(\la) - {\cal N}(\la)| \leq c(\la,m, \delta){\cal N}(\la) \leq 1.55 {\cal N}(\la),
	$$
	with $c(\la, m, \delta) = \frac{q}{{\cal N}(\la)} + \sqrt{\frac{3 q}{2{\cal N}(\la)}} + \frac{3}{2}\left(\frac{3q}{\sqrt{{\cal N}(\la)}} + \sqrt{3 q}\right)^2$ and $q= \frac{4{\cal F}_\infty(\la) \log \frac{6}{\delta}}{3 m}$.
	\ep
	\bpr
	First of all we recall that $\tilde{\cal N}(\la) = \tr \tL \tLl^{-1}$ and that 
	${\cal N}(\la) = \tr \L \Ll^{-1}$.
	Let $\tau = \delta/3$. By Prop.~\ref{prop:geo-emp-eff-dim} we know that 
	$$ |\tilde{\cal N}(\la) - {\cal N}(\la)| \leq \left(\frac{d(\la)^2}{(1-c(\la)){\cal N}(\la)} + \frac{\la e(\la)}{{\cal N}(\la)}\right) {\cal N}(\la)
	$$
	where $c(\la) = \la_{\max}(\tilde{B})$, $d(\la) = \nor{\tilde{B}}_\textrm{HS}$ with $\tilde{B} = \Ll^{-1/2}(L-\tL)\Ll^{-1/2}$ and $e(\la) = |\tr(\Ll^{-1/2}\tilde{B}\Ll^{-1/2})|$.
	Thus, now we control $c(\la)$, $d(\la)$ and $e(\la)$ in probability.
	Choosing $m$ such that $m \geq (4 + 18{\cal F}_\infty(\la))\log \frac{4\tkappa^2}{\la \tau}$, Prop.~\ref{prop:concentration-ClaC} guarantees that $c(\la) = \la_{\max}(\tilde{B}) \leq 1/3$ with probability at least $1-\tau$.
	
	To find an upper bound for $\la e(\la)$ we define the i.i.d. random variables $\eta_i = \scal{\psi_{\omega_i}}{\la\Ll^{-2} \psi_{\omega_i}}_{\rhox} \in \R$ with $i \in \{1,\dots,m\}$. By linearity of the trace and the expectation, we have $M = \mathbb{E} \eta_1  = \mathbb{E} \scal{\psi_{\omega_1}}{\la\Ll^{-2} \psi_{\omega_1}}_\rhox = \mathbb{E} \tr (\la \Ll^{-2} \psi_{\omega_1}\otimes \psi_{\omega_1}) = \la \tr (\Ll^{-2}\L)$. Therefore,
	\eqals{
		\la e(\la, m) & = \left|\la\tr(\Ll^{-1/2}\tilde{B}\Ll^{-1/2})\right| = \left|\la\tr\left(\Ll^{-1} L \Ll^{-1} - \frac{1}{m}\sum_{i=1}^m (\Ll^{-1}\psi_{\omega_i})\otimes (\Ll^{-1} \psi_{\omega_i})\right)\right|\\
		& = \left|\tr\left(\la\Ll^{-1} L \Ll^{-1}\right) - \frac{1}{m}\sum_{i=1}^m \scal{\psi_{\omega_i}}{\la\Ll^{-2} \psi_{\omega_i}}_{\rhox}\right|
		= \left|M - \frac{1}{m}\sum_{i=1}^m \eta_i\right|.
	}
	By noting that $M$ is upper bounded by $M = \tr (\la \Ll^{-2}L) = \tr((I - \Ll^{-1}L)\Ll^{-1}L) \leq {\cal N}(\la)$, we can apply the Bernstein inequality (Prop.~\ref{prop:tail-bern}) where $T$ and $S$ are
	$$\sup_{\omega \in \Omega}|M - \eta_1| \leq \la\nor{\Ll^{-1/2}}^2\nor{\Ll^{-1/2}\psi_{\omega_1}}^2 + M \leq {\cal F}_\infty(\la) + {\cal N}(\la) \leq  2{\cal F}_\infty(\la) = T,$$
	$$\mathbb{E}(\eta_1 - M)^2 = \mathbb{E}\eta_1^2 - M^2 \leq \mathbb{E}\eta_1^2 \leq (\sup_{\omega \in \Omega}|\eta_1|)(\mathbb{E}\eta_1) \leq  {\cal F}_\infty(\la) {\cal N}(\la) = S.$$
	Thus, we have
	$$ \la |\tr \Ll^{-1/2}\tilde{B}\Ll^{-1/2}| \leq \frac{4{\cal F}_\infty(\la) \log \frac{2}{\tau}}{3 m} + \sqrt{\frac{2 {\cal F}_\infty(\la) \mathcal{N}(\la) \log \frac{2}{\tau}}{m}},$$
	with probability at least $1-\tau$.
	
	To find a bound for $d(\la)$ consider that $\tilde{B} = W - \frac{1}{m}\sum_{i=1}^m \zeta_i$ where $\zeta_i$ are i.i.d. random operators defined as $\zeta_i = \Ll^{-1/2} (\psi_{\omega_i} \otimes  \psi_{\omega_i}) \Ll^{-1/2} \in \cal L$ for all $i \in \{1,\dots, m\}$, and $W = \mathbb{E} \zeta_1 = \Ll^{-1}L \in \cal L$.
	Then, by noting that $\nor{W}_{HS} \leq \mathbb{E}\tr(\zeta_1)  = \mathcal{N}(\la)$, we can apply the Bernstein's inequality for random vectors on a Hilbert space (Prop.~\ref{prop:tail-vectors}) where $T$ and $S$ are: 
	$$\nor{W - \zeta_1}_{HS} \leq \nor{\Ll^{-1/2}\psi_{\omega_1}}_{\rhox}^2 + \nor{W}_{HS} \leq {\cal F}_\infty(\la) + \nor{W}_{HS} \leq 2{\cal F}_\infty(\la) = T,$$
	$$\mathbb{E} \nor{\zeta_1 - W}{}^2 = \mathbb{E} \tr(\zeta_1^2 - W^2) \leq \mathbb{E} \tr(\zeta_1^2) \leq (\sup_{\omega \in \Omega} \nor{\zeta_1}) (\mathbb{E} \tr(\zeta_1)) = {\cal F}_\infty(\la) {\cal N}(\la) = S,$$
	obtaining
	$$d(\la) = \nor{\tilde{B}}_{HS} \leq \frac{4{\cal F}_\infty(\la) \log \frac{2}{\tau}}{m} +  \sqrt{\frac{4{\cal F}_\infty(\la)\mathcal{N}(\la) \log \frac{2}{\tau}}{m}},$$
	with probability at least $1-\tau$.
	Then, by taking a union bound for the three events we have
	$$
	|\tilde{\cal N}(\la) - {\cal N}(\la)| \leq \left(\frac{q}{{\cal N}(\la)} + \sqrt{\frac{3 q}{2{\cal N}(\la)}} + \frac{3}{2}\left(\frac{3q}{\sqrt{{\cal N}(\la)}} + \sqrt{3 q}\right)^2\right){\cal N}(\la),
	$$
	with probability at least $1-\delta$, where $q = \frac{4{\cal F}_\infty(\la) \log \frac{6}{\delta}}{3m}$.
	Finally, if $m \geq (4 + 18 {\cal F}_\infty(\la))\log \frac{12 \tkappa^2}{\la \delta}$, then we have $q \leq 2/27$. Noting that ${\cal N}(\la) \geq \nor{L\Ll^{-1}}{} = \frac{\nor{L}}{\nor{L}+\la} \geq 1/2$, we have that
	$$|\tilde{\cal N}(\la) - {\cal N}(\la)| \leq 1.55 {\cal N}(\la).$$
	\epr
	
	\section{Examples of Random feature maps}\label{sect:rf-examples}
	A lot of works have been devoted to develop random feature maps in the the setting introduced above, or slight variations (see for example \cite{conf/nips/RahimiR07, raginsky2009locality,rahimi2009weighted,cho2009,kar2012random,pham2013fast,conf/icml/LeSS13,yang2014random,hamid2014compact,conf/icml/YangSAM14} and references therein). In the rest of the section, we give several examples. 
	
	\paragraph{Random Features for Translation Invariant Kernels and extensions \cite{conf/nips/RahimiR07, raginsky2009locality,conf/icml/LeSS13,conf/icml/YangSAM14}}
	This approximation method is defined in \cite{conf/nips/RahimiR07} for the translation invariant kernels when $\X = \R^d$. A kernel $k:\X\times\X\to\R$ is {\em translation invariant}, when there exists a function $v:\X\to\R$ such that $k(x,z) = v(x-z)$ for all $x, z \in \X$. Now, let $\hat{v}: \Omega \to \R$ be the Fourier transform of $v$, with $\Omega = \X \times [0,2\pi]$. As shown in \cite{conf/nips/RahimiR07}, by using the Fourier transform, we can express $k$ as
	$$ k(x,z) = v(x-z) = \int_{\Omega} \psi(\omega, x)\psi(\omega,z)\hat{v}(\omega) d\pi, \quad \forall x,z \in \X,$$
	with $\pi$ proportional to $\hat{v}$, $\psi(\omega, x) = \cos(w^\top x + b)$ and $\omega := (w, b) \in \Omega, ~x \in \X$. Note that $X, \Omega, \pi, \psi$ satisfy Assumption~\ref{ass:kernel-bounded}.
	\\
	In \cite{raginsky2009locality}, they further randomize the construction above, by using results from locally sensitive hashing, to obtain a feature map which is a binary code. It can be shown that their methods satisfy Assumption~\ref{ass:kernel-bounded} for an appropriate choice of $\Omega$ and the probability distribution $\pi$.
	\\
	\cite{conf/icml/LeSS13,conf/icml/YangSAM14} consider the setting of \cite{conf/nips/RahimiR07}, but \cite{conf/icml/LeSS13} selects $\omega_1, \dots, \omega_m$ by means of the fast Welsh-Hadamard transform in order to improve the computational complexity for the algorithm computing $\tK$, while \cite{conf/icml/YangSAM14} selects them by using low-discrepancy sequences for quasi-random sampling to improve the statistical accuracy of $\tK$ with respect to $K$.

	\paragraph{Random Features Maps for the Gaussian Kernel, which are functions in $\hh$.}
	This set of random features  is related to the deterministic polynomial features in \cite{cotter2011explicit}. Let $\X = [0,1]^d$ and $\Omega = \N^d$. Let $\sigma > 0$. We have
	$$\psi(\omega, x) = C^{1/2} e^{-\frac{\sigma^2}{2}\|x\|^2} \prod_{j=1}^d x_j^{\omega_j}, \qquad \pi(\omega) =  \frac{\sigma^{-2 \sum_{j} \omega_j} }{C ~ \omega_1! \dots \omega_d!},$$
	where $C$ is the normalization factor $C = e^{\sigma^2 d}$. Indeed by Taylor expansion of the exponential function and of the power of a multinomial, we have
	\eqals{
		e^{-\frac{\sigma^2}{2} \nor{x-x'}^2} &= e^{-\frac{\sigma^2}{2} (\nor{x}^2 + \nor{x'}^2)} e^{\sigma^2 x^\top x}  = e^{-\frac{\sigma^2}{2} (\nor{x}^2 + \nor{x'}^2)}\sum_{t\geq 0} \frac{\sigma^{2t}}{t!} (x^\top x')^t\\
		& = e^{-\frac{\sigma^2}{2} (\nor{x}^2 + \nor{x'}^2)}\sum_{t\geq 0} \frac{\sigma^{2t}}{t!} (\sum_{j=1}^d x_j x_j')^t \\
		& = e^{-\frac{\sigma^2}{2} (\nor{x}^2 + \nor{x'}^2)}\sum_{t\geq 0} \frac{\sigma^{2t}}{t!} \sum_{\omega_1 + \dots + \omega_d = t} \frac{t!}{\omega_1! \dots \omega_d!} \prod_{j=1}^d (x_j x_j')^{\omega_j}\\
		& = \sum_{t\geq 0} \frac{\sigma^{2t}}{C t!} \sum_{\omega_1 + \dots + \omega_d = t} \frac{t!}{\omega_1! \dots \omega_d!} \left(C^{1/2} e^{-\frac{\sigma^2}{2}\nor{x}^2}\prod_{j=1}^d x_j^{\omega_j}\right) \left(C^{1/2} e^{-\frac{\sigma^2}{2} \nor{x'}^2}\prod_{j=1}^d x_j'^{\omega_j}\right) \\
		& = \sum_{\omega \in \N^d} \psi(x, \omega) \psi(x', \omega) \pi(\omega).
	}
	Note that it is possible to sample from $\pi$ in the following way. By the steps above it is clear that a sample from $\pi$ can be obtained by first sampling $t$ from a Poisson distribution with parameter $\sigma^2 d$ and then sampling $\omega_1, \dots, \omega_d$ from a multinomial distribution with probability $p_1 = \dots = p_d = 1/d$ and number of trials $t$.
	
	Finally note that $\psi(\omega, \cdot)$ is in $\hh$, the RKHS induced by the Gaussian kernel (to prove it apply Prop.~3.6 of \cite{steinwart2006explicit} with $b_\nu = \delta_{\nu = (\omega_1,\dots, \omega_d)}$ and note that $e_\nu$ is exactly a multiple of $\psi(\omega,\cdot)$). Additionally note that $\sup_{\omega, x} |\psi(\omega, x) | < \infty$ and moreover $\nor{\psi(\omega,\cdot)}_\hh < \infty$ for any $\omega \in \N^d$. However $\lim_{\|\omega\| \to \infty} \nor{\psi(\omega,\cdot)}_\hh = \infty$.

	\paragraph{Random Features Maps for Dot Product Kernels \cite{kar2012random,pham2013fast,hamid2014compact}}
	This approximation method is defined for the dot product kernels when $\X$ is the ball of $\R^d$ of radius $R$, with $R > 0$. The considered kernels are of the form $k(x,z) = v(x^\top z)$ for a $v:\R \to \R$ such that 
	$$ v(t) = \sum_{p = 0}^\infty c_p t^p, \quad \textrm{with } \quad c_p \geq 0\; \forall p \geq 0.$$ 
	\cite{kar2012random}, start from the consideration that when $w \in \{-1,1\}^d$ is a vector of $d$ independent random variables with probability at least $1/2$, then $\mathbb{E} \; ww^\top = I$. Thus
	$$\mathbb{E}_w \; (x^\top w)(z^\top w) = \mathbb{E} \; x^\top(ww^\top) z = x^\top \mathbb{E}(ww^\top) z = x^\top z,$$
	and $(x^\top z)^p$ can be approximated by $g(W_p, x)^\top g(W_p, z)$ with  $g(W_p,x) = \prod_{i=1}^p x^\top w_i$ and $W_p = (w_1,\dots,w_p) \in \{-1,1\}^{p \times d}$ a matrix of independent random variables with probability at least $1/2$. Indeed it holds,
	\eqals{
		{\mathbb E}_{W_p}\,g(W_p, x) g(W_p, z) &={\mathbb E}_{W_p}\,\prod_{i=1}^p \,(x^\top w_i)\prod_{i=1}^p(z^\top w_i) = {\mathbb E}_{W_p}\,\prod_{i=1}^p \,(x^\top w_i)(z^\top w_i) \\
		& = \prod_{i=1}^p \, {\mathbb E}_{w_i}\,(x^\top w_i)(z^\top w_i) = \prod_{i=1}^p \, x^\top z = (x^\top z )^p.
	}
	Therefore the idea is to define the following sample space $\Omega = \N_0 \times ( \bigcup_{p=1}^\infty T^p)$ with $T^p = \{-1,1\}^{p \times d}$, the probability $\pi$ on $\Omega$ as $\pi(p, W_q) = \pi_\N(p) \pi(W_q|p)$ for any $p,q \in \N_0,\, W_q \in T^q$ with $\pi_\N(p) = \frac{\tau^{-p-1}}{\tau - 1}$ for a $\tau > 1$ and $\pi(W_q|p) = 2^{-pd} \delta_{pq}$ and the function
	$$\psi(\omega, x) = \sqrt{c_p \tau^{p+1}(\tau-1)} g(W_p,x),\quad \forall \omega = (p,W_p) \in \Omega$$ where $w_i$ is the $i$-th row of $W_p$ with $1\leq i \leq p$. 
	Now note that Eq.~\eqref{eq:def-RF-integral} is satisfied, indeed for any $x,z \in \X$
	\eqals{
		\int_{\Omega} \psi(\omega, x)\psi(\omega, z) d\pi &= \int_{\N_0\times (\bigcup_{p=1}^\infty T^p)} c_p g(W_p,x)g(W_p,z) d p d\pi(W_q|p) \\
		& = \sum_{p=0}^\infty c_p \; {\mathbb E}_{W_p} g(W_p, x) g(W_p, z) \\
		& = \sum_{p=0}^\infty c_p \; (x^\top z)^p = v(x^\top z) = K(x,z).
	}
	Now note that Assumption~\ref{ass:kernel-bounded} is satisfied when $v(\tau R^2 d) < \infty$. Indeed, considering that $(x^\top w)^2 \leq \nor{x}^2\nor{w}^2 \leq R^2 d$ for any $x \in \X$ and $w \in \{-1,1\}^d$, we have
	\eqals{
		\sup_{\omega \in \Omega, x \in \X} |\psi(\omega,x)|^2 & = \sup_{p \in \N_0, W_p \in T^p, x \in \X} c_p \tau^{p+1}(\tau-1)  g_p(W_p,x)g_p(W_p,z) \\ 
		& = \sup_{p \in \N_0} c_p \tau^{p+1}(\tau-1)  \sup_{W_p \in T^p, x \in \X} \prod_{i=1}^p (x^\top w_i)^2 \\
		& \leq \sup_{p \in \N_0} c_p \tau^{p+1}(\tau-1)  (R^2 d)^p \leq \frac{\tau}{\tau - 1} \sum_{p = 0}^\infty c_p \tau^{p}  (R^2 d)^p \\
		& = \frac{\tau}{\tau - 1} v(\tau R^2 d).
	}
	\cite{pham2013fast,hamid2014compact} approximate the construction above by using randomized tensor sketching and Johnson-Lindenstrauss random projections. It can be shown that even their methods satisfy Assumption~\ref{ass:kernel-bounded} for an appropriate choice of $\Omega$ and the probability distribution $\pi$.
	
	\paragraph{Random Laplace Feature Maps for Semigroup invariant Kernels \cite{yang2014random}}
	The considered input space is $X = [0,\infty)^d$ and the considered kernels are of the form
	$$K(x,z) = v(x + z), \quad \forall x, z \in \X,$$
	and $v: \X \to \R$ is a function that is positive semidefinite. By Berg's theorem, it is equivalent to the fact that $\breve{v}$, the Laplace transform of $v$ is such that $\breve{v}(\omega) \geq 0$, for all $\omega \in \X$ and that $\int_\X \breve{v}(\omega) d\omega = V < \infty$.
	It means that we can express $K$ by Eq.~\eqref{eq:def-RF-integral}, where $\Omega = \X$, the feature map is $\psi(\omega,x) = \sqrt{V}e^{-\omega^\top x}$, for all $\omega \in \Omega, x \in \X$ and the probability density is $\pi(\omega) = \frac{\breve{v}(\omega)}{V}$. 
	Note that Assumption~\ref{ass:kernel-bounded} is satisfied. 
	
	\paragraph{Homogeneous additive Kernels \cite{vedaldi2012efficient}}
	In this work they focus on $\X = [0,1]^d$ and on additive homogeneous kernels, that are of the form 
	$$ K(x,z) = \sum_{i=1}^d k(x_i,z_i),\quad\forall x,z \in (\R^+)^d,$$
	where $k$ is an $\gamma$-homogeneous kernel, that is a kernel $k$ on $\R^+$ such that \;\; $$k(cs,ct) = c^\gamma k(s,t), \; \forall c,s,t \in \R^+.$$
	As pointed out in \cite{vedaldi2012efficient}, this family of kernels is particularly useful on histograms. Exploiting the homogeneous property, the kernel $k$ is rewritten as follows
	$$k(s,t) = (s t)^\frac{\gamma}{2} v(\log s - \log t),\;\;\forall s,t \in \R^+\quad \textrm{with} \quad v(r) = k(e^{r/2},e^{-r/2}),\;\; \forall r \in \R.$$ 
	Let $\hat{v}$ be the Fourier transform of $v$. In \cite{vedaldi2012efficient}, Lemma~1, they prove that $v$ is a positive definite function, that is equivalent, by the Bochner theorem, to the fact that $\hat{v}(\omega) \geq 0$, for $\omega \in \R$, and $\int \hat{v}(\omega)d\omega \leq V < \infty$. Therefore $k$ satisfies Eq.~\ref{eq:def-RF-integral} with $\Omega = \R \times [0,2\pi]$, a feature map $\psi_0((w,b),s) = \sqrt{V} s^\frac{\gamma}{2} \cos(b + w\omega \log s)$, with $(w, b) \in \Omega$ and a probability density $\pi_0$ defined as $\pi_0((w,b)) = \frac{\hat{v}(w)}{V}U(b)$ for all $\omega \in \Omega$, where $U$ is the uniform distribution on $[0, 2\pi]$. Now note that $K$ is expressed by Eq.~\ref{eq:def-RF-integral}, with the feature map $\psi(\omega, x) = (\psi_0(\omega, x_1),\dots,\psi_0(\omega, x_d))$ and probability density $\pi(\omega) = \prod_{i=1}^d \pi_0(\omega_i)$. In contrast with the previous examples, \cite{vedaldi2012efficient} suggest to select $\omega_1,\dots,\omega_m$ by using a deterministic approach, in order to achieve a better accuracy, with respect to the random sampling with respect to $\pi$.

	\paragraph{Infinite one-layer Neural Nets and Group Invariant Kernels \cite{cho2009}}
	In this work a generalization of the ReLU activation function for one layer neural networks is considered, that is
	$$\psi(\omega, x) = (\omega^\top x)^p {\bf 1}(\omega^\top x),\quad \forall \omega \in \Omega, x \in \X$$
	for a given $p \geq 0$, where ${\bf 1}(\cdot)$ is the step function and $\Omega = \X = \R^d$. In the paper is studied the kernel $K$, given by Eq.~\ref{eq:def-RF-integral} when the distribution $\pi$ is a zero mean, unit variance Gaussian. The kernel has the following form 
	$$K(x,z) := \frac{1}{\pi} \nor{x}^p \nor{z}^p J_p(\theta(x,z)),\quad \forall x,z \in \X.$$
	where $J_p$ is defined in Eq.~4 of \cite{cho2009} and $\theta(x,z)$ is the angle between the two vectors $x$ and $z$. Examples of $J_p$ are the following
	\eqals{
		J_0(\theta) & = \pi - \theta\\
		J_1(\theta) & = \sin \theta + (\pi-\theta) \cos \theta\\
		J_2(\theta) & = 3 \sin\theta \cos \theta + (\pi-\theta)(1+2\cos^2 \theta).
	}
	Note that when $\X$ is a bounded subset of $\R^d$ then Assumption~\ref{ass:kernel-bounded} is satisfied.
	Moreover in \cite{cho2009} it is shown that an infinite one-layer neural network with the ReLU activation function is equivalent to a kernel machine with kernel $K$. 
	The units of a finite one-layer NN are obviously a subset of the infinite one-layer NN. While in the context of Deep Learning the subset is chosen by optimization, in the paper it is proposed to find it by randomization, in particular by sampling the distribution $\pi$.

\end{document}